\documentclass{article}

\usepackage{microtype}
\usepackage{graphicx}
\usepackage{subcaption}
\usepackage{booktabs} 
\usepackage{tabularx}
\usepackage{titletoc}

\usepackage{amsthm}
\usepackage{thmtools}
\usepackage{hyperref}

\usepackage[accepted]{icml2026}

\usepackage{amsmath}
\usepackage{amssymb}
\usepackage{mathtools}
\usepackage{dsfont}

\usepackage{makecell}
\usepackage{multirow}

\newcommand{\QA}{\mathbf{a}}
\newcommand{\Pnull}{{\Pcal_0^n}}
\newcommand{\aug}{\mathrm{aug}}
\newcommand{\tran}{\mathtt{tran}}
\newcommand{\Pdelay}{{P_{\mathrm{delay}}}}
\newcommand{\Ddelay}{{\Delta_{\max}}}
\newcommand{\Dmax}{D_{\mathrm{max}}}
\newcommand{\clip}{\mathrm{clip}}
\newcommand{\cat}[2]{{#1}\!+\!{#2}}
\newcommand{\len}{\texttt{len}}
\newcommand{\Zmap}{{\phi_{\Zcal}}}
\newcommand{\effspace}{{\Zcal_{\mathrm{eff}}}}
\newcommand{\MVP}{\texttt{MVP-Est}}

\newcommand{\Qest}{\texttt{Q-Estimate}}
\newcommand{\Peff}{{P_{\mathrm{eff}}}}
\newcommand{\hatPeff}{{\hat{P}_{\mathrm{eff}}}}
\newcommand{\simplex}{\triangle}
\newcommand{\Policy}{{\Pi_{\mathrm{delay}}}}
\newcommand{\KL}[2]{\mathrm{KL}({#1} \parallel {#2})}

\newcommand{\Sleaf}{\Lcal}
\newcommand{\tDelta}{\widetilde{\Delta}}
\newcommand{\ellfinal}{\ell^*}

\usepackage{tikz}
\usetikzlibrary{matrix,arrows.meta,positioning,calc,fit,backgrounds,decorations,automata}

\usepackage[capitalize,noabbrev]{cleveref}

\theoremstyle{definition}
\newtheorem{remark}{Remark}
\theoremstyle{plain}

\usepackage[textsize=tiny]{todonotes}

\usepackage{abbreviations}
\usepackage{placeins}
\icmltitlerunning{Minimax Optimal Strategy for Delayed Observations in Online Reinforcement Learning}

\begin{document}

\twocolumn[
  \icmltitle{Minimax Optimal Strategy for Delayed Observations in\\Online Reinforcement Learning}

  \icmlsetsymbol{equal}{*}

  \begin{icmlauthorlist}
    \icmlauthor{Harin Lee}{uw}
    \icmlauthor{Kevin Jamieson}{uw}
  \end{icmlauthorlist}

  \icmlaffiliation{uw}{Paul G. Allen School of Computer Science \& Engineering, University of Washington, Seattle, Washington}

  \icmlcorrespondingauthor{Harin Lee}{leeharin@cs.washington.edu}

  \icmlkeywords{Reinforcement learning, Delayed observation, Regret}

  \vskip 0.3in
]

\printAffiliationsAndNotice{}  

\begin{abstract}
We study reinforcement learning with delayed state observation, where the agent observes the current state after some random number of time steps.
We propose an algorithm that combines the augmentation method and the upper confidence bound approach.
For tabular Markov decision processes (MDPs), we derive a regret bound of $\tilde{\mathcal{O}}(H \sqrt{\Dmax SAK})$, where $S$ and $A$ are the cardinalities of the state and action spaces, $H$ is the time horizon, $K$ is the number of episodes, and $\Dmax$ is the maximum length of the delay.
We also provide a matching lower bound up to logarithmic factors, showing the optimality of our approach.
Our analytical framework formulates this problem as a special case of a broader class of MDPs, where their transition dynamics decompose into a known component and an unknown but structured component.
We establish general results for this abstract setting, which may be of independent interest.
\end{abstract}

\section{Introduction}
Reinforcement learning (RL) is a framework where machines learn to make optimal decisions under the guidance of rewards~\citep{sutton1998}.
The learning process involves interaction between the agent and the environment.
The agent observes the state of the environment and makes an action, which affects the reward the agent receives and the transition probability of the next state.
While there have been extensive theoretical advances and empirical successes in RL, one barrier to applying RL to real-world tasks is the delayed accessibility of the current state.
For instance, domains like robotics and autonomous driving encounter inevitable delays in observations due to sensor processing, data transmission, or computational overhead.
In online advertising, a sequence of ads must be planned without observing the users' internal orientation toward the products.
This is problematic because standard methods in RL assume and heavily rely on observing the current state of the environment.
With delays in state observation, the agent is forced to plan a sequence of actions in advance.
As the number of possible action sequences increases exponentially with the length of the delay, even a minor delay may substantially complicate the problem.

While various empirical methods have been proposed to overcome this challenge~\citep{agarwal2021blind,kim2023belief,wu2024boosting,wu2025directly}, there is currently a large gap in the theoretical understanding of this setting.
\citet{walsh2009learning,chen2023efficient} show that obtaining polynomial sample complexity is possible despite the exponential growth of action sequences, but their bounds appear quite loose, and it is not clear what the optimal dependency is.
Moreover, it has not yet been established whether the length of the delay should affect the sample complexity.
In this work, we close these gaps by providing matching upper and lower regret bounds for this setup, providing a better understanding of learning MDPs with delayed observations.

We summarize our main contributions as follows:

\begin{itemize}
    \item We propose an algorithm for learning delayed MDPs that constructs an equivalent augmented MDP without delays and then applies an upper confidence bound (UCB) method.
    Several key distinctions in the augmentation process from previous works enable a better algorithmic design and a clear explanation of the theoretical guarantees.
    
    \item Under the tabular MDP setting, we derive a regret bound of $\tilde{\Ocal}(H \sqrt{\Dmax  SAK})$. This result improves upon the previously best known bound by \citet{chen2023efficient} by a factor of $H^{1/2}\Dmax^2$.
    
    \item We provide a matching lower bound of $\tilde{\Omega}(H \sqrt{\Dmax SAK})$, highlighting the optimality of our algorithm and analysis.
    To the best of our knowledge, this is the first result that rigorously shows that shorter delays (relative to the horizon $H$) decrease the statistical complexity of the problem.

    \item We identify the core properties of delayed MDPs and abstract them to define a more general setting of MDPs whose transition dynamics decompose into a fully known part and an unknown but structured part.
    We provide theoretical results for this setting, which naturally extends to the delayed observation setting as a special case.
    We emphasize that our results apply to any problem domain that fits this abstraction.
\end{itemize}

\subsection{Related Work}

\paragraph{RL under Tabular MDPs.} A variety of theoretically guaranteed methods for learning tabular MDPs have been developed and analyzed~\citep{azar2017minimax,jin2018q,zanette2019tighter,dann2019policy,zhang2021reinforcement,tiapkin2022dirichlet,zhang2024settling,lee2025minimax}.
These works achieve the minimax optimal regret bound of $\tilde{\Ocal}(H \sqrt{SAK})$ for the time-homogeneous setting or $\tilde{\Ocal}(H^{3/2} \sqrt{SAK})$ for the time-inhomogeneous setting, where the lower bound is provided by~\citet{jaksch2010near,domingues2021episodic}.

\paragraph{MDPs with Delayed Observations.}
Attention on MDPs with delayed observations dates back more than twenty years~\citep{bander1999markov,katsikopoulos2003markov,walsh2009learning}, with a heavier focus on characterizing the properties of delayed MDPs, such as defining the equivalent augmented MDP and computational hardness, rather than learning optimal policies.
\citet{walsh2009learning} describe a polynomial $(\varepsilon, \delta)$-PAC bound of $\tilde{\Ocal}( \frac{S^2 A}{\varepsilon^3 (1 - \gamma)^6})$ for $\gamma$-discounted MDPs with constant delay.
\\
A recent work by \citet{chen2023efficient} study the learning problem of delayed MDPs using modern techniques.
They assume stochastic delays with unknown distributions, imposing that the delay distribution must be learned along with the transition distribution.
They achieve a regret bound of $\tilde{\Ocal}(H^{3/2}\Dmax^{5/2}\sqrt{SAK})$ for the time-homogeneous finite-horizon setting.
They also provide a regret lower bound of $\Omega(\sqrt{H\Dmax SAK})$, leaving a gap of $\tilde{\Ocal}(H\Dmax^2)$ between the two bounds.

\paragraph{Delayed feedback in bandits and RL.}
A related yet distinct line of work studies efficient learning under delayed feedback in bandits~\citep{zhou2019learning,vernade2020linear,gael2020stochastic,lancewicki2021stochastic,masoudian2022best,howson2023delayed} and RL~\citep{howson2023optimism,lancewicki2023delay,mondal2023reinforcement,kuang2023posterior,yang2023reduction}.
These works assume that the data generated by a policy reach the learner with a delay, but the executor of the policy has access to the current state.
This fact constitutes a clear distinction from our setting, where the delay occurs in the execution process.

\section{Preliminaries}

\subsection{Notations}

For a measurable set $\Zcal$, we denote the set of all probability distributions over $\Zcal$ by $\simplex(\Zcal)$.
For a function of the form $P : \Xcal \rightarrow \simplex(\Zcal)$, we frequently write $P_{x}$ as shorthand notation for $P(x)$.
For a distribution $P \in \simplex(\Zcal)$ and a function $V : \Zcal \rightarrow \RR$ over the same space, we denote the expectation of $V$ under $P$ by $PV:= \sum_{z \in \Zcal}P(z) V(z)$, and the variance of $V$ under $P$ by $\VV(P, V) := \sum_{z \in \Zcal} P(z) (V(z) - PV)^2$.
\\
For two integers $m \le n$, let $[m:n] := \{m, \ldots, n\}$, and if $n \ge 1$, we define $[n] := [1:n]$.
For real numbers $x, a, b$, let $a \land b := \min\{a, b\}$ and $\clip(x;a , b) := \min\{\max\{x, a\}, b\}$.
\\
We frequently deal with a queue of actions, which we denote by $\QA = (a_{i}, a_{i+1}, \ldots, a_{j}) \in \cup_{n} \Acal^n$.
For an action queue $\QA = (a_1, a_2, \ldots, a_n) \in \Acal^n$, we denote the first element by $\QA_1 := a_1$, the queue starting from the second element by $\QA_{2:} = (a_2, \ldots, a_n)$, its length by $\len(\QA) := n$, and concatenation by $\cat{\QA}{a} := (a_1, \ldots, a_n, a)$ for $a \in \Acal$.

\subsection{Markov Decision Process}

We consider a finite-horizon episodic MDP $\Mcal = (\Scal, \Acal, P, r, H)$ with a time-homogeneous transition kernel.
$\Scal$ is the state space, $\Acal$ is the action space, $P : \Scal \times \Acal \rightarrow \simplex(\Scal)$ is the transition kernel, $r : \Scal \times \Acal \rightarrow \RR$ is the reward function, and $H$ is the time horizon of an episode.
We assume that the state and action spaces have finite cardinalities $S$ and $A$, respectively.
We also assume that the reward function satisfies $r(s, a) \in [0, 1]$ and is known to the agent\footnote{The analysis in this paper remains valid when the assumption $r(s, a) \in [0, 1]$ is relaxed to $0 \le V_h^{\pi}(s) \le H$ for all $h, s$, and $\pi$, which is a generalization noted in~\citet{lee2025minimax}. Also, an unknown reward function can be learned by standard methods~\citep{zanette2019tighter,zhang2024settling}, which are applicable to the delayed observation setting.}.
An agent interacts with the MDP for $K$ episodes.
At the $k$-th episode, through time steps $h = 1, \ldots, H$, the agent observes a state $s_h^k \in \Scal$, takes an action $a_h^k \in \Acal$, receives a reward $r(s_h^k, a_h^k)$, and the next state is sampled as $s_{h+1}^k \sim P(s_h^k, a_h^k)$.
A policy $\pi = \{ \pi_h\}_{h=1}^H$ is a sequence of functions $\pi_h : \Scal \rightarrow \Acal$ that determines the action given the current state.
The value function of a policy is defined as $V_{h}^{\pi}(s) := \EE_{\pi(\cdot \mid s_h = s)} [ \sum_{j=h}^H r(s_j, a_j)]$.
The optimal policy is defined as the policy $\pi^*$ that satisfies $V_h^{\pi^*}(s) = V_h^*(s) := \max_{\pi} V_h^{\pi}(s)$ for all $h$ and $s$.
The goal of the agent is to minimize the cumulative regret over $K$ episodes, which is defined as $\sum_{k=1}^K(V_1^*(s_1^k) - V_1^{\pi^k}(s_1^k))$.

We define the \emph{branching factor} $B$ as an upper bound on the support sizes of $P(s, a)$, that is, the number of candidate states one can reach by taking the state-action pair.
Formally, we have $| \{s' \in \Scal \mid P_{s, a}(s')  > 0\} | \le B \le S$ for all $(s, a)$.
We assume $B$ is known to the agent, which is always possible since the agent may set $B = S$.

\section{Problem Setting}

We consider stochastic delayed MDPs (SDMDPs)~\citep{katsikopoulos2003markov}, which are MDPs with delayed state observation.
An SDMDP consists of an MDP $\Mcal$, a delay distribution $\Pdelay : \Scal \times \Acal \rightarrow \simplex([-1:\Ddelay])$, and an upper bound on the delay length $\Dmax$.
The main difference in this setting is that the current state is revealed to the agent after some time steps.
As a result, the agent must take an action while the current state is still hidden, then observes the effect of the action after the delay.
Specifically, after taking the $h$-th action $a_h$, the following next state $s_{h+1}$ is revealed to the agent after $D_h$ time steps, that is, at the beginning of the $h + 1 + D_h$-th time step.
If $h+1 + D_h \ge H+1$, then the state is revealed after the end of the episode.
We assume that the delay $D_h$ is determined in the following way.

\begin{assumption}
\label{assm:delay}
    For each time step $h$, an inter-arrival time $\Delta_h$ is sampled from the distribution $\Pdelay(s_h, a_h)$ independently of all other randomness including the sampling of $s_{h+1}$.
    Starting from $D_0 = 0$, which implies that the initial state $s_1$ is always immediately observed, the delay $D_h$ is determined incrementally as $D_h = \clip(D_{h-1} + \Delta_h; 0, \Dmax)$, where $\Dmax$ is a known constant that upper-bounds $D_h$.
\end{assumption}

\begin{remark}
    In the previous work~\citep{katsikopoulos2003markov,chen2023efficient}, the inter-arrival time $\Delta_h$ is assumed to be distributed over $[0 : \Ddelay]$, implying that at most one state is revealed per time step.
    Generalizing their assumption, we allow an inter-arrival time of $-1$, in which case multiple consequent states are revealed at the same time step.
\end{remark}

\paragraph{Constant Delayed MDP (CDMDP).} 
The case where the length of the delay $D_h$ is equal to $\Dmax$ for all $h \in [H]$ is called the \emph{constant delayed MDP (CDMDP)}~\citep{walsh2009learning}.
\cref{assm:delay} recovers this setting by adding an auxiliary initial state $s_{\mathrm{start}}$ with a fixed inter-arrival time $\Dmax$ and setting all the other inter-arrival times to $0$, making the total length of the delay $D_h = \sum_{j=1}^h \Delta_j$ equal to $\Dmax$.

\newcolumntype{Y}{>{\centering\arraybackslash}X}

\begin{table}[tb]
    \caption{Example trajectory of a CDMDP with $H = 6$ and constant delay $2$.
    Each column denotes the time step $h$, the last observed state $s_{t_h}$, the action queue $\QA$, the inter-arrival time counter $\tDelta_h$, the inter-arrival time $\Delta_{t_h}$, and the total delay $D_{t_h}$. Note that a new state is observed when $\tDelta_h$ is equal to $\Delta_{t_h}$, or equivalently, when the length of the queue $\len(\QA)$ is equal to $D_{t_h}$.}
    \centering
    \begin{tabularx}{\linewidth}{Y| Y c Y Y Y}
        \toprule
        $h$
        & $s_{t_h}$
        & $\QA$
        & $\tDelta_h$      
        & $\Delta_{t_h}$ 
        & $D_{t_h}$ \\
        \midrule
        1   & $s_1$ & $\emptyset$   & 0 & 2 & 2 \\
        2   & $s_1$ & $(a_1)$       & 1 & - & - \\
        3   & $s_1$ & $(a_1, a_2)$  & 2 & - & - \\
        4   & $s_2$ & $(a_2, a_3)$  & 0 & 0 & 2 \\
        5   & $s_3$ & $(a_3, a_4)$  & 0 & 0 & 2 \\
        6   & $s_4$ & $(a_4, a_5)$  & 0 & 0 & 2 \\
        \midrule
        End & $s_7$ & $\emptyset$   &  - & - & - \\
        \bottomrule
    \end{tabularx}
    \label{tab:constant delay}
\end{table}

\begin{table}[tb]
    \caption{Example trajectory of an SDMDP with $H = 6$ with one possible realization of delays.
    Refer to \cref{tab:constant delay} for the description of each column.
    States $s_3$ and $s_4$ are revealed at the same time step $h = 6$ because $\Delta_3 = -1$.}
    \centering
    \begin{tabularx}{\linewidth}{Y| Y c Y Y Y}
        \toprule
        $h$
        & $s_{t_h}$
        & $\QA$
        & $\tDelta_h$ 
        & $\Delta_{t_h}$ 
        & $D_{t_h}$ \\
        \midrule
        1 & $s_1$ & $\emptyset$       & 0 & 1 & 1 \\
        2 & $s_1$ & $(a_1)$           & 1 & - & - \\
        3 & $s_2$ & $(a_2)$           & 0 & 2 & 3 \\
        4 & $s_2$ & $(a_2, a_3)$      & 1 & - & - \\
        5 & $s_2$ & $(a_2, a_3, a_4)$ & 2 & - & - \\
        \multirow{2}{*}{6} & $s_3$ & $(a_3, a_4, a_5)$ & (-1) & -1  & 2 \\
         & $s_4$ & $(a_4, a_5)$     & 0 & 1 & 3\\
        \midrule
        End & $s_7$ & $\emptyset$ & - &- & - \\
        \bottomrule
    \end{tabularx}
    \label{tab:random_delay}
\end{table}

We consider both cases of SDMDPs where $\Pdelay$ is known and $\Pdelay$ is unknown.
When unknown, following \citet{chen2023efficient}, we allow the agent to learn the distribution by assuming that the original values of $\Delta_1, \ldots, \Delta_H$ are revealed at the end of the episode, which is necessary since their values may be lost when the delay is clipped or the episode terminates.
We emphasize that access to $\Delta_h$ during the episode is not allowed.

With observational delays, the policy can no longer be a mapping from a state to an action since the current state is not observable.
Instead, the policy must consider three elements.
First, we let $t_h$ be the time index of the last state observable at time step $h$.
If a state is revealed at time $h$, it is state $s_{t_h}$, otherwise we have $t_h = t_{h-1}$.
Next, let $\QA = (a_{t_h}, a_{t_h+1}, \ldots, a_{h-1})$ be the queue of unresolved actions.
It is straightforward that the policy must consider $s_{t_h}$ and $\QA$.
In addition, the policy must take the time step that revealed $s_{t_h}$ into account, since it affects when the next state $s_{t_h+1}$ is revealed under \cref{assm:delay}.
We let $\tDelta_h = h - (t_h + D_{t_h-1})$ be the number of time steps elapsed since a state was revealed; recall that $t_h + D_{t_h-1}$ is the time step at which $s_{t_h}$ is revealed.
For example, in the constant $\Dmax$ delay setting, we have $\tDelta_1 = 0, \tDelta_2 = 1, \tDelta_3 = 2,\dots, \tDelta_{\Dmax+1} = \Dmax$, and $\tDelta_{\Dmax+i} = 0$ for all $i > 1$ since we observe a new state at each time step, just delayed by $\Dmax$.
It can be shown that the information of these three elements is sufficient to make optimal actions in an SDMDP~\citep{katsikopoulos2003markov}.
Consider the following policy class:

\begin{align*}
    \Policy = \left\{ \{\pi_h\}_{h=1}^H: \Scal \times \cup_{D=0}^{\Dmax} \Acal^D \times [0:\Ddelay] \rightarrow \Acal  \right\}
    \, .
\end{align*}
A policy $\pi \in \Policy$ chooses an action $a_h = \pi_h(s_{t_h}, \QA, \tDelta)$ at time step $h$.
For a given SDMDP $(\Mcal, \Pdelay, \Dmax)$, we can also define the value function of a policy as $V_1^{\pi}(s) = \EE_{\pi(\cdot \mid s_1 = s)}[\sum_{h=1}^H r(s_h, a_h)]$, and the optimal policy $\pi^*_{\mathrm{delay}}$ that satisfies $V_1^{\pi^*_{\mathrm{delay}}}(s) = \max_{\pi \in \Policy} V_1^{\pi}(s)$ for all $s \in \Scal$.
The cumulative regret that the agent must minimize is redefined as $\sum_{k=1}^K (V_1^{\pi^*_{\mathrm{delay}}}(s_1^k) - V_1^{\pi^k}(s_1^k) )$.

\paragraph{Connection to Partially Observable MDPs (POMDPs).}
POMDPs are a framework in which latent states emit random observations rather than being directly observed~\citep{papadimitriou1987complexity,jin2020sample}. 
Having limited observability, the delayed observation model may be viewed as a special case of this framework.
Specifically, an SDMDP can be modeled as a POMDP by treating a sequence of the last observed state followed by the unobserved states as a single latent state, with only the first state in the sequence being observable.
One benign condition for POMDPs that may have a connection to the delayed MDPs is \emph{multi-step revealing}, which posits that the latent state can be statistically inferred from a number of following observations~\citep{liu2022partially,liu2023optimistic,zhang2025statistical}.
While the delayed MDPs fall into this subclass of POMDPs, techniques for this setting incur regret exponential in the delay length.
This fact indicates that reduction to generic POMDPs fails to exploit the structure induced by the delays.
See \citet{walsh2009learning,chen2023efficient} for further discussion on why methods for POMDPs are not applicable to delayed MDPs.

\section{Algorithm}

In this section, we describe our algorithm for learning SDMDPs.
There are two main steps.
First, we construct an augmented MDP that is equivalent to the given SDMDP.
Then, we learn the augmented MDP by applying standard techniques with certain modifications that exploit the properties of the augmented MDP.
Each step is explained in the following subsections.

\subsection{Augmented Markov Decision Process}
\label{sec:augmented MDP}
Given an SDMDP $(\Mcal, \Pdelay, \Dmax)$, we construct the following equivalent augmented MDP (without delay) $\Mcal_{\aug} = (\Scal_{\aug}, \Acal, P_{\aug}, r_{\aug}, H)$.
The intuition is to treat the tuple of the last observed state $s_{t_h}$, the action queue $\QA_h = (a_{t_h}, \ldots, a_{h-1})$, and the number of time steps without new observation $\tDelta_h$ as an augmented state.
We note that this construction is, in principle, equivalent to the ones in~\citet{katsikopoulos2003markov,chen2023efficient}, but we introduce two types of intermediate states that allow us to better explain the algorithm and the analysis.
\\
We define the augmented state space as
$\Scal_{\aug} := \Scal \times \cup_{D = 0}^{\Dmax} \Acal^D \times \Dcal  \times [H + 1]$, where $\Dcal := [-1: \Ddelay] \cup \{\tran\}$.
We note that $\Scal_{\aug}$ is of size exponential in $\Dmax$, and that any computational implications of this blow-up are apparently unavoidable (see Section~\ref{sec:comp_hardness} for a discussion).
Consider an augmented state $s_{\aug} = (s_{t_h}, \QA, \tDelta_h, h) \in \Scal_{\aug}$.
This state corresponds to the case where the last observed state is $s_{t_h}$, the unresolved action queue is $\QA$, the last action in the queue is $a_{h-1}$ (or possibly empty if $t_h = h$), and $\tDelta_h$ time steps have elapsed since $s_{t_h}$ was revealed, but the cases $\tDelta_h = -1$ and $\tDelta_h = \tran$ convey different meanings.
If the initial state of the SDMDP is $s_1$, then the initial state of the augmented MDP is $(s_1, \emptyset, 0, 1)$.
Depending on $\tDelta$, the properties and the transitions of the state are determined.
We explain them using three categories.

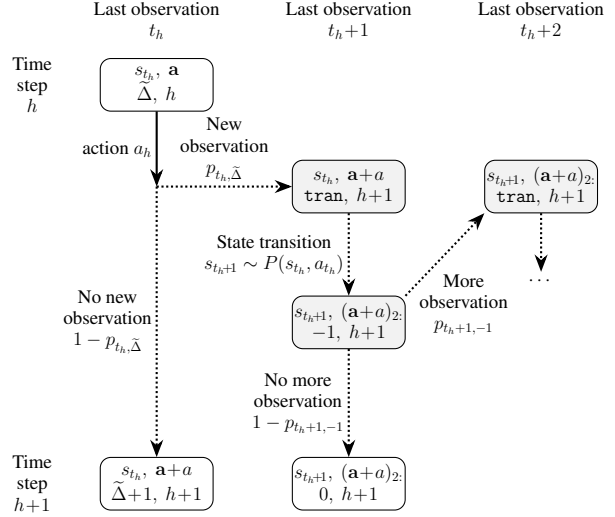
\begin{figure}[tb]
\centering
\begin{tikzpicture}[
  scale=0.5, transform shape,
  >=Stealth,
  state/.style={draw, rounded corners,
    minimum width=30mm, minimum height=14mm,inner sep=0pt},
  interstate/.style={draw, fill=gray!10, rounded corners,
    minimum width=30mm, minimum height=14mm,inner sep=0pt},
  statelabel/.style={align=center,font=\Large,inner sep=0pt},
  edge/.style={->,thick,font=\Large,align=center},
  dottededge/.style={->,thick,densely dotted,font=\Large,align=center}
]

\node[state] (A) {};

\coordinate (J) at ($(A.south)+(0,-20mm)$);
\node[state, below=72mm of J] (B) {};

\node[interstate, right=36mm of J] (C) {};
\node[interstate, below=22mm of C] (D) {};
\node[interstate, right=21mm of C] (E) {};
\node[state] (F) at (D |- B) {};

\node[statelabel] at (A) {$s_{t_h},\ \QA$\\$\tDelta,\ h$};
\node[statelabel] at (B) {$s_{t_h},\ \cat{\QA}{a_h}$\\$\tDelta\!+\!1,\ h\!+\!1$};

\node[statelabel] at (C) {$s_{t_h},\ \cat{\QA}{a_h}$\\$\tran,\ h\!+\!1$};
\node[statelabel] at (D) {$s_{t_h\!+\!1},\ (\cat{\QA}{a_h})_{2:}$\\$-1,\ h\!+\!1$};

\node[statelabel] at (E) {$s_{t_h\!+\!1},\ (\cat{\QA}{a_h})_{2:}$\\$\tran,\ h\!+\!1$};
\node[statelabel] at (F) {$s_{t_h\!+\!1},\ (\cat{\QA}{a_h})_{2:}$\\$0,\ h\!+\!1$};

\coordinate (Tline) at ($(A.north)+(0,10mm)$);
\node[font=\Large,align=center] at (A.center |- Tline) {Last observation\\$t_h$};
\node[font=\Large,align=center] at (C.center |- Tline) {Last observation\\$t_h\!+\!1$};
\node[font=\Large,align=center] at (E.center |- Tline) {Last observation\\$t_h\!+\!2$};

\node[font=\Large,align=center] at ($(A.west)+(-18mm,0)$) {Time\\step\\$h$};
\node[font=\Large,align=center] at ($(B.west)+(-18mm,0)$) {Time\\step\\$h\!+\!1$};

\draw[edge] (A.south) -- node[left] {action $a_h$} (J);

\draw[dottededge] (J) -- node[midway, above] {New\\observation\\$p_{t_h,\tDelta}$} (C.west);
\draw[dottededge] (J) -- node[midway, left] {No new\\observation\\$1-p_{t_h,\tDelta}$} (B.north);

\draw[dottededge] (C.south) -- node[midway, left] {State transition\\$s_{t_h\!+\!1} \sim P(s_{t_h}, a_{t_h})$} (D.north);

\draw[dottededge] (D.north east) -- node[midway, below, yshift=-4mm, xshift=5mm]
  {More\\observation\\$p_{t_h+1, -1}$} (E.south west);

\draw[dottededge] (D.south) -- node[midway, left]
  {No more\\observation\\$1-p_{t_h+1, -1}$} (F.north);

\node[font=\Large] (Cont) at ($(E.south)+(0,-18mm)$) {$\cdots$};
\draw[dottededge] (E.south) -- (Cont);

\end{tikzpicture}
\caption{Illustration of the transition of the augmented MDP after taking one action. Gray states indicates intermediate states that have no actions. 
Straight lines represents the agent's action and dotted lines represents augmented state transitions.
$p_{t_h,\tDelta}$ is shorthand for $P_{\tran}(s_{t_h}, a_{t_h}, \tDelta)$ and $p_{t_h+1, -1}$ is shorthand for $P_{\tran}(s_{t_h+1}, a_{t_h+1}, -1)$, where they show the probability of the transition.}
\label{fig:augmented transition}
\end{figure}

\textbf{Category 1: $\tDelta_h \in [0:\Ddelay]$.}
These augmented states are standard ones where $\tDelta_h$ corresponds to the number of time steps elapsed since $s_{t_h}$ was revealed.
The augmented-state transition at this state determines whether there will be a new observation or not.
When the agent takes an action, the next augmented state is either one of the two:
$(s_{t_h}, \cat{\QA}{a}, \tran, h+1)$, which corresponds to the case where the next state will be revealed, and $(s_{t_h}, \cat{\QA}{a}, \tDelta_h + 1, h+1)$, which corresponds to not having a new observation. 
We note that even if a new observation is determined to be revealed, it is not yet revealed at this stage.
We define $P_{\tran} : \Scal \times \Acal \times [-1 : \Dmax] \rightarrow [0, 1]$ as $P_{\tran}(s, a, \tDelta) := P_{\Delta \sim \Pdelay(s, a)} ( \Delta = \tDelta \mid \Delta \ge \tDelta) = \frac{\Pdelay_{s, a}(\tDelta)}{\sum_{\Delta \ge \tDelta} \Pdelay_{s, a}(\Delta)}$, which represents the probability of transitioning to the first case.
Although $P_{\tran}$ requires additional information of whether the delay is clipped or the episode has terminated, we omit it here for simplicity.
\\
\textbf{Category 2: $\tDelta_h = \tran$.}
An augmented state with $\tDelta_h = \tran$ is an intermediate state that has no actions, but only state transition.
The agent enters this state when the next state is determined to be revealed, and the transition of this augmented state samples the next state.
Denoting $a_{t_h} = \QA_1$, the next augmented state is $(s', \QA_{2:}, -1, h)$ with $s' \sim P_{s_{t_h}, a_{t_h}}$.
A reward of $r(s_{t_h}, a_{t_h})$ is received together with this transition.
\\
\textbf{Category 3: $\tDelta_h = -1$.}
An augmented state with $\tDelta_h = -1$ is another type of intermediate state.
The agent enters this state immediately after $s_{t_h}$ is revealed, and whether there will be further observations within the same time step is determined at this augmented state.
The next augmented state is either $(s_{t_h}, \QA, \tran, h)$ with probability $P_{\tran}(s_{t_h}, a_{t_h}, -1)$, which corresponds to the case where another state will be revealed, or $(s_{t_h}, \QA, 0, h)$ with the remaining probability.

\subsection{Algorithm}
\label{sec:algorithms}

\begin{algorithm*}[tb]
    \caption{MVP-Delayed (Informal)}
    \begin{algorithmic}[1]
        \STATE Construct augmented MDP $\Mcal_{\mathrm{aug}}$ as in \cref{sec:augmented MDP}
        \STATE Define $\ellfinal(D, b) := ( D \log A + \log \frac{64 H (D+1)^2 \Ddelay SA K}{\delta}) \land ( b \log \frac{32 H b \Ddelay SAK}{\delta})$
        \FOR{$k =1, 2, \ldots, K$}
        \STATE Compute value estimates as follows for all $s \in \Scal$, $\QA \in \cup_{D=0}^{\Dmax} \Acal^D$, $\tDelta \in [-1:\Ddelay]$, $h = H, \ldots, 1$:
        \STATE $V^k(s, \QA, \tran, h) \gets \MVP\left(r(s, \QA_1), \hat{P}_{s, \QA_1}, V^k(\cdot, \QA_{2:}, -1, h), N^k(s, \QA_1), \ellfinal(\len(\QA), B) \right)$
        \STATE $Q^k((s, \QA, \tDelta, h), a) \gets \begin{cases}
                (P_{\tran})_{s, \QA'_1, \tDelta} V_\tran + (1 - (P_{\tran})_{s, \QA'_1, \tDelta}) V_{\mathrm{delay}} & (\text{known } \, \Pdelay)
                \\
                \MVP(0, (\hat{P}_\tran^k)_{s, \QA'_1, \tDelta}, (V_\tran, V_{\mathrm{delay}}), N^k(s, \QA'_1, \tDelta), \ellfinal(\len(\QA), 2)) & (\text{unknown } \, \Pdelay)
            \end{cases}$
        \STATE \quad \text{where } $\QA' = \cat{\QA}{a}, (V_\tran, V_\mathrm{delay}) = (V^k(s, \QA', \tran, h+1), V^k(s, \QA', \tDelta + 1, h+1))$
        \STATE $V^k(s, \QA, \tDelta, h) \gets \max_{a \in \Acal} Q^k((s, \QA, \tDelta, h), a)$
        \STATE $\pi_h^k(s, \QA, \tDelta, h) \gets \argmax_{a \in \Acal} Q^k((s, \QA, \tDelta, h), a)$
        \STATE Execute $\pi^k$ and collect data $\{ s_h^k, a_h^k, \Delta_h^k\}_h$
        \ENDFOR
    \end{algorithmic}
    \label{alg:MVP delay informal}
\end{algorithm*}

\begin{algorithm}[tb]
    \caption{\MVP}
    \begin{algorithmic}[1]
        \STATE \textbf{Input:} $r \in \RR$, $\hat{P} \in \simplex(\Scal)$, $V \in \RR^\Scal$, $N \in \NN$, $\ell \in \RR$
        \STATE Set constants $c_1 \gets \frac{20}{3}$, $c_2 \gets \frac{400}{9}$
        \IF{$N \le 1$}
            \STATE \textbf{Return} $H$
        \ELSE
            \STATE \textbf{Return} $\left(r + \hat{P}V + c_1 \sqrt{\frac{\VV(\hat{P}, V) \ell}{N}} + c_2 \frac{H \ell}{N} \right) \land H$
        \ENDIF
    \end{algorithmic}
\label{alg:MVP}
\end{algorithm}

After constructing the augmented MDP, we apply standard techniques for common MDPs with an extension using certain structure of the augmented MDP.
For the tabular case, we choose $\texttt{MVP}$~\citep{zhang2021reinforcement} as a base algorithm, which is a standard UCBVI-based algorithm with Bernstein-type bonuses.
See \cref{alg:MVP} for its optimistic estimation rule.
We note that any other UCBVI-based algorithm also works.

As the augmented state space has exponential size in $\Dmax$, directly applying $\texttt{MVP}$ would lead to exponential regret.
The augmented MDP can be learned much faster by noting that many of the transition probabilities are shared based on $P$ or $\Pdelay$.
For example, an observed transition of $(s, \QA, \tran, h)$ to $(s', \QA_{2:}, -1, h)$ gives information about $P_{s, a}(s')$ that can be used to estimate the transition of $(s, \QA', \tran, h)$ for different $\QA'$ with $\QA_1' = \QA_1$.
The core idea is to estimate $P$ and $\Pdelay$ instead of $P_{\aug}$.
Hence, instead of storing visit counts of augmented state-action pairs, it is sufficient to store the visit counts of the original state-action pairs as $N^k(s, a) = \sum_{i=1}^{k-1} \sum_{h=1}^H \ind\{ (s_h^i, a_h^i) = (s, a)\}$ and store an estimate $\hat{P}_{s, a}^k$ of $P_{s, a}$.
In the case where $\Pdelay$ is unknown, we additionally store $N^k(s, a, \tDelta) = \sum_{i=1}^{k-1} \sum_{h=1}^H \ind\{ (s_h^i, a_h^i) = (s, a), \Delta_h^i \ge \tDelta \}$, which is the number of times that any augmented state of the form $(s, \cat{a}{\QA}, \tDelta, h)$ is visited.
Note that this number is different from the number of $(s, a, \tDelta)$ in the dataset.
We also estimate $P_\tran$ by $\hat{P}_{\tran}^k$, defined as 
\begin{align*}
    \hat{P}_{\tran}^k(s, a, \tDelta) := \frac{N^k(s, a, \tDelta) - N^k(s, a, \tDelta + 1)}{ N^k(s, a, \tDelta)}
    \, ,
\end{align*}
where the numerator is the number of times the $\tDelta$ is sampled from $\Pdelay(s, a)$, and the denominator is the number of samples from $\Pdelay(s, a)$ that are greater than or equal to $\tDelta$.
Using these estimates, we run $\texttt{MVP}$ under a correct update order.
Another difference from the standard $\texttt{MVP}$ is that the log factor is increased to a roughly $D \land b$ factor, where $D$ is the length of the current action queue, and $b$ is an upper bound on the branching factor of the current augmented state-action pair.
An informal description of the algorithm is shown in \cref{alg:MVP delay informal}.
We present the full algorithm that includes the update order and exception handling in \cref{appx:full algo}.

\section{Theoretical Guarantees}
 In this section, we provide regret upper bound results for \cref{alg:MVP delay informal}.
 We define $\iota := \log \frac{H SAK}{\delta}$ for logarithmic factors.

\label{sec:upper bound}
\begin{theorem}
\label{thm:known delay}
    Suppose the delay distribution $\Pdelay$ is known. With probability at least $1 - \delta$, \cref{alg:MVP delay informal} achieves the regret bound of $\Ocal(H\sqrt{(\Dmax \land B) SAK} \iota + H B S A \iota^2)$.
\end{theorem}

\begin{theorem}
\label{thm:unknown delay}
    Suppose the delay distribution $\Pdelay$ is unknown. 
    With probability at least $1 - \delta$, \cref{alg:MVP delay informal} achieves the regret bound of 
    \begin{align*}
        \Ocal\big( & H \sqrt{ ( \Dmax \land B)SAK} \iota + H \sqrt{\Ddelay SAK} \iota 
        \\
        & \quad +  H (B + \Ddelay) SA \iota^2 \big)
        \, .
    \end{align*}
\end{theorem}

\paragraph{Discussion of \cref{thm:known delay,thm:unknown delay}.}
In both known and unknown delay distribution cases, we derive the regret upper bound of $\tilde{\Ocal}(H \sqrt{\Dmax SAK})$.
Compared to the previous bound of $\tilde{\Ocal}(H^{3/2}\Dmax^{5/2}\sqrt{SAK})$ by \citet{chen2023efficient}, we improve the bound by a factor of $H^{1/2}\Dmax^2$, significantly improving the $\Dmax$-dependency.
While the $H^{1/2}$ factor improvement comes from using variance-dependent bonus terms, the $\Dmax^2$ factor improvement comes from our novel analysis.

\cref{thm:known delay,thm:unknown delay} demonstrate that $\Dmax$-dependency can be replaced by a known branching factor $B$ when it is smaller.
This fact implies that the performance degradation caused by lengthening the delay is not indefinite.
This is because the $\Dmax$-dependency arises from taking the union bound over exponentially many states in $\Dmax$, but it can be replaced by taking the union bound over all bounded functions with domain size $B$.

In the unknown $\Pdelay$ case, we note that the second term may be larger than the first term only when $B < \Ddelay \le \Dmax$ holds since $\Ddelay \le \Dmax$.
In this case, we incur a slightly larger regret as there are $\Ddelay SA$ many values of $P_{\tran}(s, a, \tDelta)$ to learn.

\begin{remark}
\label{rmk:first order}
    The $\Dmax$ dependency in \cref{thm:known delay,thm:unknown delay} can be improved when the actual lengths of the delays are shorter than $\Dmax$.
    Defining $\Dmax(s, a)$ as the maximum possible length of delay $D_h$ starting from the state-action pair $(s_h, a_h) = (s, a)$, the $(\Dmax \land B) SA$ factor can be reduced to $\sum_{(s, a) \in \Scal \times \Acal} (\Dmax(s, a) \land B)$, even if the agent is not aware of individual $\Dmax(s, a)$ values.
    In the same way, the $\Ddelay SA$ factor in \cref{thm:unknown delay} may be reduced to $\sum_{(s, a)} \Ddelay(s, a)$, where $\Ddelay(s, a)$ is the unknown maximum inter-arrival time of $(s, a)$.
    As a result, even if the agent only knows crude upper bounds of $\Dmax$ and $\Ddelay$, the regret bound remains the same.
\end{remark}

\section{Regret Lower Bound}
\label{sec:lower bound}

In this section, we provide a regret lower bound result for SDMDPs, showing that our approach is minimax optimal up to logarithmic factors.

\begin{theorem}[Lower bound result]
\label{thm:regret lower bound}
    Let $\tilde{D} := \min\{ \Dmax, \frac{H}{4}, \frac{B}{2}, \frac{S}{4} - 1\} = \Ocal(\Dmax \land B) $.
    Suppose $\tilde{D} \ge c$ for some absolute constant $c$, $H \ge 8 + 4 \log_A S$, and $K \ge \frac{1}{24} \tilde{D}SA \log \tilde{D}$.
    Then, for any algorithm, there exists a CDMDP instance with $S$ states, $A$ actions, time horizon $H$, and delay length $\Dmax$ with branching factor at most $B$ such that the expected regret of the algorithm for $K$ episodes is at least $\Omega \bigl( H\sqrt{\tilde{D} SAK (\log \tilde{D})^{-1}} \bigr)$.
\end{theorem}

As CDMDPs (constant delay) are special cases of SDMDPs (stochastic delay), the theorem applies to SDMDPs as well.

A lower bound of $\Omega(H\sqrt{SAK})$ can be immediately derived using the hard instance and the analysis for the standard MDP setting by \citet{domingues2021episodic}, as they naturally apply to SDMDPs.
From this bound, we show that an additional factor of $\Dmax^{1/2}$ is necessary in the delayed setting, implying that as the delay becomes longer, the problem becomes statistically harder.
Together with \cref{thm:known delay}, these results establish a tight minimax regret bound up to logarithmic factors, characterizing that the optimal dependence on delay is $\Dmax^{1/2}$.
\cref{thm:unknown delay} also meets this lower bound when $\Ddelay \le B$.
We note that the regret lower bound of $\Omega( \sqrt{H\Dmax SAK})$ by \citet{chen2023efficient}\footnote{Proposition 4.3 in \citet{chen2023efficient} states a regret lower bound of $\Omega(H\sqrt{\Dmax SAK})$ for the time-\emph{inhomogeneous} setting, which is translated to $\Omega(\sqrt{H\Dmax SAK})$ for the time-homogeneous setting we consider.} is worse than the standard bound since $\Dmax \le H$.

For the proof of \cref{thm:regret lower bound}, we design a novel structure named \textit{CodeMDP}, which is a CDMDP with delay length $D$ whose learning complexity increases with $D$.
The agent is supposed to find a correct sequence of $D$ actions that maximizes the reward.
Using the structure of the CodeMDP, we reduce an $\ell_1$-norm estimation problem of a $d$-dimensional vector to solving a CodeMDP with unknown transitions. 
We then establish the following lower bound for estimating the $\ell_1$-norm of a vector, which may be of independent interest.

\begin{proposition}
\label{prop:l1 estimation}
    Let $\thetab \in [-1, 1]^d$.
    At each time step $t$, suppose an index $I_t \sim \Unif([d])$ and a Bernoulli random variable $X_t \sim B(\frac{1 + \theta_{I_t}}{2})$ are revealed.
    Then, the sample complexity for estimating $\frac{1}{d}\| \thetab \|_1$ up to some additive error $\varepsilon$ is at least $\Omega\left( \frac{ d}{\varepsilon^2 \log d} \right)$.
\end{proposition}

An important property of \cref{prop:l1 estimation} is that the sample complexity scales with $d$.
As a result, the learning complexity of a CodeMDP also scales with $D$, which plays a crucial role in obtaining the $D_{\max}^{1/2}$-factor in \cref{thm:regret lower bound}.
The detailed proof of \cref{thm:regret lower bound}, including the structure of the Code MDP and its relationship with the problem setting of \cref{prop:l1 estimation}, is discussed in \cref{appx:lower bound}.

\section{Computational Hardness}\label{sec:comp_hardness}

One obstacle in deploying \cref{alg:MVP delay informal} is its exponential time complexity in $\Dmax$.
In this section, we provide a negative result regarding the computational hardness of solving delayed MDPs that implies that a polynomial time algorithm is unlikely to exist.

An unobservable MDP (UMDP) is an MDP that the agent cannot observe intermediate states and has to plan $H$ actions in advance.
It coincides with CDMDPs with delay $D = H$.
More generally, CDMDPs with delay $D$ can be considered as embedding UMDPs with a time horizon $D$.
Assuming that the transition probabilities are given, one method of computing the optimal value of a given UMDP is to compute the values of all possible $A^H$ sequences of actions. 
There are several results regarding the computational hardness of UMDPs showing that the exponential time complexity is unlikely to be avoidable, and we introduce a result in \citet{burago1996complexity}.

\begin{theorem}[Restatement of Theorem 6 in \citet{burago1996complexity}]
\label{thm:np hard}
    Suppose a UMDP $\Mcal = (\Scal, \Acal, P, r, H)$ with $H = |\Scal| = n$ is given.
    The problem of distinguishing whether the optimal value of the given UMDP is $1$ or less than $\exp( - \sqrt{n})$ is NP-hard.    
\end{theorem}

\cref{thm:np hard} states that approximating the optimal value of a UMDP is NP-hard, even if its transition probabilities are fully known.
The theorem is proved by reducing a 3-SAT problem, making the problem of approximating the optimal value of SDMDPs strictly harder than 3-SAT problems.
Therefore, the exponential time complexity of \cref{alg:MVP delay informal} is the best we can hope for.

\section{Proof Sketch: Generalization to MDPs with Partially Known Dynamics}

In this section, we provide a sketch of how the regret bounds in \cref{sec:upper bound} are obtained.
Instead of directly proving the results, we first define a more general model which we call MDPs with partially known dynamics.
This model captures and generalizes the core properties of the augmented MDPs we utilize.
We provide an algorithm and theoretical guarantees for this setting, and then the results for solving SDMDPs follow as specific cases of the theorem.

\subsection{MDPs with Partially Known Dynamics}

We first explain the motivation of defining MDPs with partially known dynamics.
There are two properties of the augmented MDP that lead to a polynomial regret bound despite its exponential state space.
For simpler exposition, assume the constant delay setting, and suppose the augmented state consists only of the last observed state $s_{t_h}$ and the action queue $\QA = (a_{t_h}, \ldots, a_{h-1})$.
The first property of the augmented MDP is that the agent has perfect knowledge about the transition of the action queue --- the first action is popped and the current action is pushed --- and hence there is nothing to learn.
The unknown part of the augmented state transition is the distribution of the next state $s_{t_h+1}$.
The second property of the augmented MDP is that the unknown part of the transition is determined by a relatively small part of the augmented state-action pair.
Specifically, the distribution of the next state $s_{t_h + 1}$ depends only on $(s_{t_h}, a_{t_h})$, and the other part of the augmented state-action pair has no effect on it.
In other words, the function class for the unknown part of the transition has a certain structure.

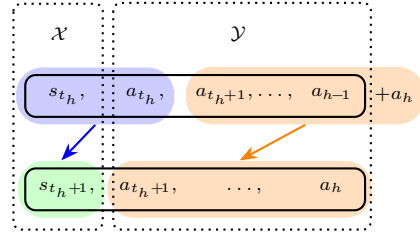
\begin{figure}[tb]
\centering
\begin{tikzpicture}[
  >=Stealth,
  token/.style={font=\scriptsize, inner sep=1.2mm},
  box/.style={draw, rounded corners, thick, inner sep=0.5mm},
  shadeX/.style={draw=none, fill=blue!20,  rounded corners=10pt},
  shadeS/.style={draw=none, fill=green!20, rounded corners=10pt},
  shadeQ/.style={draw=none, fill=orange!25, rounded corners=10pt},
  boundary/.style={dotted, thick, rounded corners=3pt},
  shiftarrow/.style={->, thick, orange},
  transarrow/.style={->, thick, blue}
]

\matrix (M) [matrix of nodes, nodes=token, row sep=8mm, column sep=0.0mm, nodes in empty cells] {
  {$s_{t_h},$} & {$a_{t_h},$}  & {$a_{t_h\!+1},\ldots,$} & {$a_{h\!-\!1}$} \\
  {$s_{t_h+\!1},$}  & {$\;a_{t_h+\!1},$} & {$\ldots,$} & {$a_h$} \\
};

\coordinate (XLeft)  at (M-2-1.west);
\coordinate (XRight) at (M-1-4.east);

\coordinate (TopL) at (XLeft  |- M-1-1.west);
\coordinate (TopR) at (XRight |- M-1-4.east);
\coordinate (BotL) at (XLeft  |- M-2-1.west);
\coordinate (BotR) at (XRight |- M-2-4.east);

\node[box, fit=(M-1-1)(M-1-4)(TopL)(TopR)] (Top) {};
\node[box, fit=(M-2-1)(M-2-4)(BotL)(BotR)] (Bot) {};

\node[token, anchor=west] (Plus) at ($(M-1-4.east)+(0.5mm,0)$) {$+a_h$};

\node[token, inner sep=0pt] (Xlab) at ($(M-1-1.north)+(-1mm,6mm)$) {$\Xcal$};
\coordinate (Yc) at ($(M-1-2.north)!0.5!(M-1-4.north)$);
\node[token, inner sep=0pt] (Ylab) at ($(Yc)+(0,6mm)$) {$\Ycal$};

\coordinate (BlueL)   at ($(M-1-1.west)+(-3mm,0)$);
\coordinate (GreenL)  at ($(M-2-1.west)+(-1.5mm,0)$);
\coordinate (OrangeL) at ($(M-2-2.west)+(0.5mm,0)$);
\coordinate (OrangeR) at ($(M-2-4.east)+(2.5mm,0)$);
\coordinate (OrangeB) at ($(M-2-2.south)$);

\begin{pgfonlayer}{background}
  \node[shadeX, fit=(M-1-1)(M-1-2)(BlueL), inner xsep=0mm, inner ysep=1.5mm] (GTopSA) {};
  \node[shadeQ, fit=(M-1-3)(M-1-4)(Plus), inner xsep=0mm, inner ysep=1.5mm] (GTopQ) {};
  \node[shadeS, fit=(GreenL)(M-2-1), inner xsep=0mm, inner ysep=1.5mm] (GBotS) {};
  \node[shadeQ, fit=(OrangeL)(OrangeB)(M-2-4)(OrangeR), inner xsep=0mm, inner ysep=1.5mm] (GBotQ) {};
\end{pgfonlayer}

\path (M-1-1.east) -- (M-1-2.west) coordinate[midway] (Sep);
\coordinate (SepL) at ($(Sep)+(-0.4mm,0)$);
\coordinate (SepR) at ($(Sep)+( 1mm,0)$);

\node[fit=(Xlab)(M-1-1)(M-2-1), inner sep=3mm] (BXfit) {};
\node[fit=(Ylab)(M-1-2)(M-1-4)(M-2-2)(M-2-4), inner sep=3mm] (BYfit) {};

\draw[boundary] ($(BXfit.north west)+(1mm,0)$) rectangle (SepL |- BXfit.south);
\draw[boundary] (SepR |- BYfit.north) rectangle ($(BYfit.south east)+(-1.8mm,0)$);

\draw[transarrow] (GTopSA.south) -- (GBotS.north) ;
\draw[shiftarrow] (GTopQ.south) -- (GBotQ.north);

\end{tikzpicture}
\caption{Illustration of core properties of the augmented MDP's state transition.
Consider a transition accompanied by an augmented state-action pair of $((s_{t_h}; a_{t_h}, \ldots, a_{h-1}), a_h)$.
The orange-shaded part indicates that the transition dynamics for the action queue are known, which is simply shifting from the previous action queue. The blue-shaded part indicates that the unknown part of the state transition is determined only by $(s_{t_h}, a_{t_h})$ and is irrelevant of the other part of the augmented state-action pair.}
\end{figure}

Motivated by these properties, we define an abstract model that has these properties by definition, which may be of independent interest.
We formalize the condition we impose on the MDP as follows.

\begin{definition}[MDPs with partially known dynamics]
\label{assm:partially known dynamics}
    We call an MDP $\Mcal = (\Scal , \Acal, P, r, H)$ an \emph{MDP with partially known dynamics} if the following conditions are satisfied.
    The state space decomposes as $\Scal = \Xcal \times \Ycal$.
    The marginal transition kernel for $\Ycal$ is known, that is, there exists a known function $P_{\Ycal} : \Scal \times \Acal \rightarrow \simplex(\Ycal)$ such that $\sum_{x \in \Xcal} P_{s, a}(x, y) = (P_{\Ycal})_{s, a}(y)$.
    We denote the conditional transition kernel of $\Xcal$ by $P_{\Xcal} : \Scal \times \Acal \times \Ycal \rightarrow \simplex(\Xcal)$, which satisfies $P_{s, a}(x, y) = (P_{\Xcal})_{s, a, y}(x) (P_{\Ycal})_{s, a}(y)$.
    The conditional transition kernel $P_{\Xcal}$ is unknown to the agent but lies in a known function class $\Pcal_{\Xcal} \subset \{ P : \Scal \times \Acal \times \Ycal \rightarrow \simplex(\Xcal ) \}$.
\end{definition}

We consider the case where the functions in the function class $\Pcal_{\Xcal}$ depend only on a finite number of features.

\begin{assumption}[Tabular structure]
\label{assm:partially known dynamics tabular}
    $\Pcal_{\Xcal}$ in \cref{assm:partially known dynamics} satisfies the following.
    There exist a finite set $\effspace$ and a known feature map $\Zmap : \Scal \times \Acal \times \Ycal \rightarrow \effspace$ such that $P_{\Xcal} \in \Pcal_{\Xcal}$ only if there exists a distribution $\Peff : \effspace \rightarrow \simplex(\Xcal)$ with $P_{\Xcal}(s, a, y) = \Peff(\Zmap(s, a, y))$.
\end{assumption}

\begin{remark}
    It is important to note that \cref{assm:partially known dynamics tabular} is \emph{not} equivalent to MDPs with $\effspace$ as a feature space.
    $\effspace$ only represents partial information about the state that governs the unknown part of the state transition.
    It is possible that multiple state-action pairs are mapped to the same element of $\effspace$, but their transition dynamics for the second argument $\Ycal$ may differ, possibly resulting in drastically different action values for those state-action pairs.
\end{remark}

\subsection{Algorithm and Theoretical Guarantees}

\begin{algorithm*}[tb]
  \caption{Optimistic Algorithm for Partially Known Dynamics}
  \label{alg:partially known dynamics}
  \begin{algorithmic}[1]
    \STATE Define $\ellfinal(z) := \log \frac{32 H | \Ycal | |\effspace | K }{\delta} \land B(z) \log \frac{32 H B(z) |\effspace | K}{\delta}$
    \FOR {$k = 1, 2, \ldots$}
        \STATE $V_{H+1}^k(s) \gets 0$ for all $s \in \Scal$
        \FOR{$h = H, H-1, \ldots, 1$}
            \FORALL{$s \in \Scal$}
                \STATE $\displaystyle Q_h^k(s, a) \gets \Expec_{\substack{y \sim (P_{\Ycal})_{(s, a)},\\z \gets \Zmap(s, a, y)}}[\MVP(r(s, a), (\hatPeff^k)_z, V_{h+1}^k(\cdot, y), N^k(z), \ellfinal(z))]$ for all $a \in \Acal$
                \STATE $V_h^k(s) \gets \max_{a \in \Acal} Q_h^k(s, a)$
                \STATE $\pi_h^k(s) \gets \argmax_{a \in \Acal} Q_h^k(s, a)$
            \ENDFOR
        \ENDFOR
        \STATE Execute $\pi^k$ and collect data $\{s_h^k = (x_h^k, y_h^k), a_h^k\}_h$
        \STATE Store $z_h^k \gets \Zmap(s_h^k, a_h^k, y_{h+1}^k)$ for $h \in [H]$, update $N^{k+1}(z)$ and $(\hatPeff^{k+1})_z$ accordingly
    \ENDFOR
  \end{algorithmic}
\end{algorithm*}

We propose a general algorithm that can run under \cref{assm:partially known dynamics tabular}, using $\texttt{MVP}$ as a base algorithm.
Let $B(z) \le |\Xcal|$ be a known upper bound on the support size of $\Peff(z)$.
\cref{alg:partially known dynamics} shows the specific procedure.
\cref{alg:partially known dynamics} has two main differences from learning standard tabular MDPs.
First, instead of directly learning the transition distribution $P$, the algorithm learns $\Peff$, storing the visit count of $z \in \effspace$ instead of the actual state-action pairs.
Secondly, the logarithmic term is increased by $\log |\Ycal| \land B(z) \log B(z)$, which results from the union bound.

We present the following theorem for \cref{alg:partially known dynamics}.

\begin{theorem}
\label{thm:partially known dynamics}
    Under \cref{assm:partially known dynamics tabular}, 
    \cref{alg:partially known dynamics} achieves the regret bound of
    \begin{align*}
        \tilde{\Ocal} \Bigg(& H \sqrt{ K \sum_{z \in \effspace} \left( \left( \log |\Ycal|   \right) \land B(z)  \right)} + H \sum_{z \in \effspace} B(z)  \Bigg)
        \, .
    \end{align*}
\end{theorem}

When $B(z) = B$ is a constant, the leading term of \cref{thm:partially known dynamics} becomes $\tilde{\Ocal}(H\sqrt{|\effspace| K ((\log|\Ycal|) \land B)})$.
Compared to the $\tilde{\Ocal}(H\sqrt{SAK})$ regret bound of standard MDPs, the cost of learning $SA$ many transition distributions is reduced to $|\effspace|$, but the increased logarithmic factor is reflected in the $(\log | \Ycal|) \land B$ factor.
The proof of \cref{thm:partially known dynamics} is largely based on the standard optimism-based analysis adapted to this setting.
The complete proof is presented in \cref{appx:proof of partially}.

\subsection{Application to SDMDPs}

\label{sec:proof sketch}

In this section, we explain that \cref{thm:known delay,thm:unknown delay} are corollaries of \cref{thm:partially known dynamics}, where rigorous proofs are presented in \cref{appx:proof of delay tabular}.
We mainly focus on the unknown $\Pdelay$ case (\cref{thm:unknown delay}) in this section.

We first explain how the augmented MDP constructed in \cref{sec:augmented MDP} satisfies \cref{assm:partially known dynamics tabular}.
Recall that the augmented state space for a given SDMDP is constructed as $\Scal_{\aug} = \Scal \times \cup_{D=0}^{\Dmax}\Acal^D \times \Dcal \times [H+1]$.
We take $\Xcal = \Scal \times \Dcal$ and $\Ycal = \cup_{D=0}^{\Dmax} \Acal^D \times [H+1]$.
The space $\Ycal$ consists of the action queue and the time step, and the agent is fully aware of their dynamics.
We take $\effspace = \Scal \times \Acal \times \Dcal$ with $\Zmap((s, \QA, \tDelta, h), a, y) = (s, (\cat{\QA}{a})_1, \tDelta)$.
Strictly speaking, we need additional information in the elements of $\effspace$ to handle exceptional cases such as the termination of an episode or the truncation of the delay.
However, as it only increases the size of $\effspace$ by a constant factor, we omit this detail for brevity.
$\Peff$ is defined as $\Peff(s, a, \tran) = P(s, a) \times \ind_{\tDelta = -1}$ and $\Peff(s, a, \tDelta) = P_{\tran}(s, a, \tDelta) \ind_{(s, \tran)} + (1 - P_{\tran}(s, a, \tDelta) ) \ind_{(s, \tDelta +1)}$ for $\tDelta \in [-1:\Ddelay]$, where $\ind_x$ denotes the point mass distribution at $x$.
In this way, the augmented MDP satisfies \cref{assm:partially known dynamics tabular}.
\\
Now, we explain how \cref{thm:unknown delay} is derived from \cref{thm:partially known dynamics}.
There are $SA$ elements of the form $z = (s, a, \tran)$.
Each of them adds $\log |\Ycal| \land B(z) \lesssim \Dmax \land B$ to the regret bound inside the square root.
Therefore, they contribute to the regret bound by $\tilde{\Ocal}(H \sqrt{(\Dmax \land B)SAK} + HBSA)$.
There are $\Ocal(SA\Ddelay)$ remaining elements in $\effspace$ of the form $z = (s, a, \tDelta)$ with $\tDelta \in [-1:\Ddelay]$.
We note that they only have a branching factor of two, as they transition to either $(s, \tran)$ or $(s, \tDelta + 1)$ in $\Xcal$.
Consequently, these elements contribute to the $\tilde{\Ocal}( H \sqrt{\Ddelay SA K} + H \Ddelay SA)$ terms.
Taking the sum over the two bounds yields the bound of \cref{thm:unknown delay}.
\\
\cref{thm:known delay} is proved by decomposing the state space as $\Xcal = \Scal$ and $\Ycal = \cup_{D=0}^{\Dmax} \Acal^D \times \Dcal \times [H+1]$, where $\Dcal$ moved from $\Xcal$ to $\Ycal$ to reflect that the delay distribution is known, and then following the first half of the argument for \cref{thm:unknown delay}.

\section{Conclusion}

We study an efficient method of learning MDPs with delayed state observation.
Under the tabular setting, we propose an augmentation-based algorithm and provide a regret bound of $\tilde{\Ocal}(H\sqrt{\Dmax SAK})$, together with a matching regret lower bound, showing that our result is optimal up to logarithmic factors.
Our techniques for constructing the augmented MDP and decomposing the known and unknown parts of the transition naturally apply beyond the tabular case, and we believe they provide a general framework for addressing observational delays in RL.

\pagebreak
\section*{Impact Statement}

This paper presents work whose goal is to advance the field of Machine
Learning. There are many potential societal consequences of our work, none which we feel must be specifically highlighted here.

\section*{Acknowledgements}
KJ and HL were supported in part by NSF 2141511, 2023239, and a Singapore AI Visiting Professorship award.

\bibliography{references}
\bibliographystyle{icml2026}

\newpage
\appendix
\onecolumn

\crefalias{section}{appendix}
\crefalias{subsection}{appendix}
\crefalias{subsubsection}{appendix}

\newpage

\section{Full Algorithms}

\label{appx:full algo}

In this section, we present the full procedure of \cref{alg:MVP delay informal} introduced in \cref{sec:algorithms}.
\cref{alg:MVP delay formal} is the full version of \cref{alg:MVP delay informal} that specifies the update order and exception handling.
As explained in \cref{sec:algorithms}, instead of storing the visit counts of the augmented state-action pairs, it stores the visit counts of the original MDP's state-action pairs and the tuple of state, action, and inter-arrival time.
Specifically, it maintains the information of
\begin{equation}
\label{eq:update rule delay tabular}
    \begin{split}
    & N^k(s, a) := \sum_{i=1}^{k-1} \sum_{h=1}^H \ind\{ (s_h^i, a_h^i) = (s, a) \}
    \\
    & N^k(s, a, \tDelta) := \sum_{i=1}^{k-1} \sum_{h=1}^H \ind \{(s_h^i, a_h^i) = (s, a), \Delta_h^i \ge \tDelta  \}
    \\
    & (\hat{P}^k)_{s, a}(s') := \frac{1}{N^k(s, a)}\sum_{i=1}^{k-1} \sum_{h=1}^H \ind \{(s_h^i, a_h^i, s_{h+1}^i) = (s, a, s')  \}
    \\
    & (\hat{P}_\tran^k)_{s, a}(\tDelta) := \frac{N^k(s, a, \tDelta) - N^k(s, a, \tDelta + 1)}{N^k(s, a, \tDelta)}
    \, .
    \end{split}
\end{equation}
As a subroutine of \cref{alg:MVP delay formal}, we have \cref{alg:Q-estimate} ($\Qest$) that handles the computation of $Q$-values for augmented states with $\tDelta \in [-1 : \Ddelay]$.
Recall that these augmented states transition to one of the two augmented states depending on whether a new state is revealed or not.
When $\Pdelay$ is known, $\Qest$ takes the expectation of the next augmented states' values.
When $\Pdelay$ is unknown, $\MVP$ is called to compute the UCB value instead.
It also handles the case where the next state is trivially revealed due to clipping or the termination of the episode.

\begin{algorithm*}[htbp!]
  \caption{MVP-Delayed}
  \label{alg:MVP delay formal}
  \begin{algorithmic}[1]
    \STATE {\bfseries Input:} Per time step delay $\Ddelay$, Maximum delay $\Dmax$
    \STATE Define $\ellfinal(D, b) := ( D \log A + \log \frac{64 H (D+1)^2 \Ddelay SA K}{\delta}) \land ( b \log \frac{32 H b \Ddelay SAK}{\delta})$
    \FOR {$k = 1, 2, \ldots$}
        \STATE Initialize $N^k(s, a), N^k(s, a, \tDelta), \hat{P}^k(s, a), \hat{P}_\tran^k(s, a, \tDelta)$ as Eq.~\eqref{eq:update rule delay tabular}
        \STATE Initialize $V^k(s, \emptyset, \tran, H + 1) \gets 0$, $V^k(s, \emptyset, -1, H+1) \gets 0$ for all $s \in \Scal$
        \FORALL{$s \in \Scal, \QA \in \cup_{D=1}^{\Dmax} \Acal^D$}
            \STATE $V^k(s, \QA, \tran, H+1)\gets \MVP(r(s, \QA_1), \hat{P}_{s, \QA_1}^k, V^k( \cdot, \QA_{2:}, -1, H+1), N^k(s, \QA_1), \ellfinal(D, B))$
            \STATE $V^k(s, \QA, -1, H+1) \gets V^k(s, \QA, \tran, H+1)$
        \ENDFOR
        \FOR{$h = H , \ldots, 1$}
            \STATE // \textit{Update Category 1 states ($\tDelta \in [0:\Ddelay]$)}
            \FORALL{$s \in \Scal, \QA \in \cup_{D=0}^{\Dmax} \Acal^D, \tDelta \in [ 0 : \Ddelay]$}
                \STATE $Q^k((s, \QA, \tDelta, h), a) \gets \Qest(s, \cat{\QA}{a}, \tDelta, h+1)$ for all $ a\in \Acal$
                \STATE $V^k(s, \QA, \tDelta, h) \gets \max_{a \in \Acal} Q((s, \QA, \tDelta, h), a)$
                \STATE $\pi_h^k(s, \QA, \tDelta, h) \gets \argmax_{a \in \Acal} Q((s, \QA, \tDelta, h), a)$
            \ENDFOR
            \STATE // \textit{Update Category 2 and 3 states ($\tDelta = \tran$ or $-1$)}
            \STATE $V^k(s, \emptyset, -1, h) \gets V^k(s, \emptyset, 0, h)$ for all $s \in \Scal$
            \FOR{$D = 1, \ldots, \Dmax$}
                \FORALL{$s \in \Scal, \QA \in \Acal^D$}
                    \STATE $V^k(s, \QA, \tran, h)\gets \MVP(r(s, \QA_1), \hat{P}_{s, \QA_1}^k, V^k( \cdot, \QA_{2:}, -1, h), N^k(s, \QA_1), \ellfinal(D, B))$
                    \STATE $V^k(s, \QA, -1, h) \gets \Qest(s, \QA, -1, h)$
                \ENDFOR
            \ENDFOR
        \ENDFOR
        \STATE Execute $\pi^k$ and collect data $\{ s_h^k, a_h^k, \Delta_h^k\}_{h=1}^H \cup \{s_{H+1}^k\}$
    \ENDFOR
  \end{algorithmic}
\end{algorithm*}

\begin{algorithm*}[htbp!]
    \caption{\Qest}
    \label{alg:Q-estimate}
    \begin{algorithmic}[1]
        \STATE \textbf{Input: } $(s, \QA, \tDelta, h)$
        \STATE $V_\tran \gets V^k(s, \QA, \tran, h)$
        \STATE $V_{\mathrm{delay}} \gets V^k(s, \QA, \tDelta + 1, h)$
        \IF{$\len(\QA) = \Dmax + 1$ or $\tDelta = \Ddelay$ or $h = H+1$}
            \STATE \textbf{Return } $V_\tran$
        \ELSIF{$\Pdelay$ is known}
            \STATE Get $p \gets P_{\tran}(s, \QA_1, \tDelta)$
            \STATE \textbf{Return } $pV_\tran  + (1  - p) V_{\mathrm{delay}}$
        \ELSE
            \STATE \textbf{Return } $\MVP(0, \hat{P}^k_\tran(s, \QA_1, \tDelta), (V_\tran, V_{\mathrm{delay}}), N^k(s, \QA_1, \tDelta), \ellfinal(\len(\QA), 2)$
        \ENDIF
    \end{algorithmic}
\end{algorithm*}

\section{Proof of Theorem~\ref{thm:partially known dynamics}}
\label{appx:proof of partially}

In this section, we present the full proof of \cref{thm:partially known dynamics}.

We define several notations for the analysis.
Define $\ell_1 = \log \frac{32 H^2 |\Ycal| |\effspace| K }{\delta}$, $\ell_2(z) = \log \frac{32 H B(z) |\effspace| K }{\delta}$, and $\ellfinal(z) = \min\{ \ell_1, B(z) \ell_2(z) \}$.
For $z \in \effspace$, let $\Xcal(z) := \{ x \in \Xcal \mid (\Peff)_z(x) > 0 \}$ be the support of $(\Peff)_z$.
Let $N_h^k(z)$ be the number of time steps $z \in \effspace$ appeared in the trajectory up to the $h$-th time step of the $k$-th episode, that is, $N_h^k(z) := N^k(z) + \sum_{j=1}^h \ind\{ z_j^k = z\}$.
For $h \in [H+1]$, $s \in \Scal$, and $k \in [K]$, we define $U_h^k(s)$ and $\tilde{U}_h^k(s)$ iteratively starting from $U_{H+1}^k(s) = \tilde{U}_{H+1}^k(s) := 0$ and 
\begin{align*}
    &U_h^k(s) := \Expec_{\substack{s' = (x', y') \sim P_{s, a}\\z = \Zmap(s, a, y')}} \left[ \frac{\ellfinal(z)}{N^k(z)} + U_{h+1}^k(s') \right] \land 1
    \, ,
    \\
    & \tilde{U}_h^k(s) := \Expec_{\substack{s' = (x', y') \sim P_{s, a}\\z = \Zmap(s, a, y')}} \left[ \frac{B(z)\ell_2(z)}{N^k(z)} + \tilde{U}_{h+1}^k(s') \right] \land 1
    \, ,
\end{align*}
for $h \in [H]$, where $a = \pi_h^k(s)$.

\subsection{High-probability Events}

In this section, we define the high-probability events that constitute the event $\Ecal$ under which \cref{thm:partially known dynamics} holds.

\begin{lemma}
\label{lma:optimal value concentration}
    The following inequality holds for all $z \in \effspace$, $y \in \Ycal$, $h \in [H]$, and $k \in [K]$ that satisfies $N^k(z) \ge 1$ with probability at least $1 - \frac{\delta}{8}$:
    \begin{align*}
        \left| ( (\Peff)_z - (\hatPeff)_z)V_{h+1}^*(\cdot, y) \right| \le \sqrt{\frac{2 \VV((\Peff)_z, V_{h+1}^*(\cdot, y)) \ell_1}{N^k(z)}} + \frac{H \ell_1}{3 N^k(z)}
        \, .
    \end{align*}
\end{lemma}

\begin{proof}
    Apply \cref{lma:bernstein} to $Z = V_{h+1}^*(x, y)$ with $x \sim \Peff(z)$ and the probability of failure as $\frac{\delta}{8}$, then take the union bound over $y \in \Ycal$, $z \in \effspace$, $h \in [H]$, and $1 \le N^k(s, a) \le KH$. 
\end{proof}

The following lemma extends to arbitrary functions at the cost of $B(z)$ dependence in the bound.
The proof is deferred to \cref{appx:proof of all functions concentration}.

\begin{lemma}
\label{lma:all functions concentration}
    There exists an event with probability at least $1 - \frac{\delta}{8}$ such that for all $z \in \effspace$ and for any constants $a\le b$, the following inequality holds for any function $V: \Xcal(z) \rightarrow [a, b]$:
    \begin{align*}
        \left| ((\Peff)_z - (\hatPeff^k)_z)V \right| \le \sqrt{\frac{2 \VV((\Peff)_z, V) B(z) \ell_2(z)}{N^k(z)}} + \frac{(b - a) B(z) \ell_2(z)}{3 N^k(z)}
        \, .
    \end{align*}
\end{lemma}

The following \cref{lma:optimal value concentration empirical,lma:all functions concentration empirical} are empirical-variance versions of \cref{lma:optimal value concentration,lma:all functions concentration}, respectively.
The proofs of these lemmas are nearly identical to their counterparts, only that the uses of \cref{lma:bernstein} are replaced by \cref{lma:empirical bernstein}.

\begin{lemma}
\label{lma:optimal value concentration empirical}
    The following inequality holds for all $z \in \effspace$, $y \in \Ycal$, $h \in [H]$, and $k \in [K]$ that satisfies $N^k(s, a) \ge 1$ with probability at least $1 - \frac{\delta}{8}$:
    \begin{align*}
        \left| ( (\Peff)_z - (\hatPeff)_z)V_{h+1}^*(\cdot, y) \right| \le 2 \sqrt{\frac{\VV( (\hatPeff)_z, V_{h+1}^*(\cdot, y)) \ell_1}{N^k(z)}} + \frac{14 H \ell_1}{3 N^k(z)}
        \, .
    \end{align*}
\end{lemma}

\begin{lemma}
\label{lma:all functions concentration empirical}
    There exists an event with probability at least $1 - \frac{\delta}{8}$ such that for all $z \in \effspace$ and for any constants $a\le b$, the following inequality holds for any function $V: \Xcal(z) \rightarrow [a, b]$:
    \begin{align*}
        \left| ((\Peff)_z - (\hatPeff^k)_z)V \right| \le 2 \sqrt{\frac{ \VV((\hatPeff)_z, V) B(z) \ell_2(z)}{N^k(z)}} + \frac{14 (b - a) B(z) \ell_2(z)}{3 N^k(z)}
        \, .
    \end{align*}
\end{lemma}

We additionally require two lemmas regarding the concentration of $U_1^k(s_1^k)$ and $\tilde{U}_1^k(s_1^k)$ between the sum of $\frac{1}{N^k(z)}$ over the actual trajectory.
The following lemma is an application of Lemma 15 in \citet{lee2025minimax}, and the proof is identical to it. 

\begin{lemma}
\label{lma:concentration of U}
    With probability at least $1 - \frac{\delta}{2}$, we have
    \begin{align*}
        \sum_{k=1}^K U_1^k \le 2 \sum_{k=1}^K \sum_{h=1}^H \ind\{ 2 N^k(z_h^k) > N_h^k(z_h^k)\} \cdot \frac{\ellfinal(z_h^k)}{N^k(z_h^k)}   + 3 |\effspace| \log \frac{8 H}{\delta}
    \end{align*}
    and
    \begin{align*}
        \sum_{k=1}^K \tilde{U}_1^k \le 2 \sum_{k=1}^K \sum_{h=1}^H \ind\{ 2 N^k(z_h^k) > N_h^k(z_h^k)\} \cdot \frac{ B(z_h^k) \ell_2(z_h^k)}{N^k(z_h^k)}   + 3 |\effspace| \log \frac{8 H}{\delta}
        \, .
    \end{align*}
\end{lemma}

We let $\Ecal$ be the intersection of the events of Lemmas~\ref{lma:optimal value concentration}-\ref{lma:concentration of U}.
By the union bound, we have $\PP(\Ecal) \ge 1 - \delta$.

\subsection{Optimism}

In this section, we prove the following optimism lemma for \cref{alg:partially known dynamics}.

\begin{lemma}[Optimism]
\label{lma:partially:optimism}
    The estimated value $V_h^k(s)$ in \cref{alg:partially known dynamics} satisfies $V_h^k(s) \ge V_h^*(s)$ for all $s \in \Scal$ under $\Ecal$.
\end{lemma}

To prove the lemma, we require the following lemmas from \citet{zhang2021reinforcement}.

\begin{lemma}[Lemma 14 in \citet{zhang2021reinforcement}]
\label{lma:monotonic lemma}
    Let $f: \simplex([S]) \times \RR^S \times \RR \times \RR \rightarrow \RR$ with $f(p, V, n, \ell) := pV + \max\{ \frac{20}{3} \sqrt{\frac{\VV(p, V) \ell}{n}}, \frac{400}{9} \frac{\ell}{n} \}$, where $\VV(p, V) = \sum_{s \in [S]} p(s) (V(s) - pV)^2$.
    Then, $f$ satisfies
    \begin{enumerate}
        \item $f(p, V, n, \ell)$ is non-decreasing in $V(s)$ for all $p \in \simplex([S])$, $\| V \|_\infty \le 1$, $n, \ell > 0$.
        \item $f(p, V, n, \ell) \ge pV + 2 \sqrt{\frac{\VV(p, V) \ell}{n}} + \frac{14\ell}{3n}$.
    \end{enumerate}
\end{lemma}

The following lemma guarantees that $\MVP$ outputs an optimistic estimate in Line 8 of \cref{alg:partially known dynamics}.
It follows the same reasoning as Lemma 4 in \citet{zhang2021reinforcement}.

\begin{lemma}
\label{lma:partially:pre optimisim}
    Fix $s \in \Scal$, $a \in \Acal$, $y \in \Ycal$, $h \in [H]$, and $k \in [K]$, and let $z = \Zmap(s, a, y)$.
    Define the conditional optimal action value (conditioned on $y$) as $Q_h^*(s, a, y) = r(s, a) + \sum_{x \in \Xcal} (P_{\Xcal})_{s, a, y}(x) V_{h+1}^*(x, y)$ and the conditional optimistic estimate as $Q_h^k(s, a, y) := \MVP(r(s, a), (\hatPeff)_z, V_{h+1}^k(\cdot, y), N^k(z), \ellfinal(z))$.
    Suppose $V_{h+1}^k(x, y) \ge V_{h+1}^*(x, y)$ holds for all $x \in \Xcal$.
    Then, we have $Q_h^k(s, a, y) \ge Q_h^*(s, a, y)$ under the event $\Ecal$.
\end{lemma}
\begin{proof}
    For simplicity, denote $P(\cdot) :=  (\Peff)_z(\cdot)$, $\hat{P}(\cdot) := (\hatPeff^k)_z(\cdot)$, $r := r(s, a)$, $V_{h+1}^*(\cdot) := V_{h+1}^*(\cdot, y)$, $V_{h+1}^k(\cdot) := V_{h+1}^k(\cdot, y)$, $N := N^k(z)$, and $\ellfinal := \ellfinal(z)$.
    If $Q_h^k(s, a, y) = H$, then the lemma becomes trivial.
    Suppose $Q_h^k(s, a, y) < H$, which implies $N \ge 2$.
    Under the event of \cref{lma:optimal value concentration empirical}, we have
    \begin{align*}
        (P - \hat{P})V_{h+1}^* \le 2 \sqrt{\frac{\VV(\hat{P}, V_{h+1}^*) \ell_1}{N}} + \frac{14 H \ell_1}{3 N}
        \, .
    \end{align*}
    Under the event of \cref{lma:all functions concentration empirical}, we have
    \begin{align*}
        (P - \hat{P})V_{h+1}^* \le 2 \sqrt{\frac{\VV(\hat{P}, V_{h+1}^*) B(z) \ell_2(z)}{N}} + \frac{14 H B(z) \ell_2(z)}{3 N}
        \, .
    \end{align*}
    Taking the minimum over the two bounds, we have
    \begin{align*}
        (P - \hat{P})V_{h+1}^* \le 2 \sqrt{\frac{\VV(\hat{P}, V_{h+1}^*) \ellfinal}{N}} + \frac{14 H \ellfinal}{3 N}
        \, ,
    \end{align*}
    where we use that $\ellfinal = \ellfinal(z) = \ell_1 \land B(z) \ell_2(z)$.
    Then, $Q_h^*(s, a, y)$ is upper bounded as
    \begin{align*}
        Q_h^*(s, a, y)
        & = r + P V_{h+1}^*
        \\
        & = r+ \hat{P}V_{h+1}^* + ( P - \hat{P}) V_{h+1}^*
        \\
        & \le r + \hat{P}V_{h+1}^* + 2 \sqrt{ \frac{\VV(\hat{P}, V_{h+1}^*) \ellfinal}{N} } + \frac{14 H \ellfinal}{3 N}
        \, .
    \end{align*}
    By \cref{lma:monotonic lemma}, we have
    \begin{align*}
        \hat{P}V_{h+1}^*+ 2 \sqrt{ \frac{\VV(\hat{P}, V_{h+1}^*) \ellfinal}{N} } + \frac{14 H \ellfinal}{3 N}
        & \le \hat{P}V_{h+1}^* + \max\left\{ \frac{20}{3} \sqrt{\frac{\VV(\hat{P}, V_{h+1}^*)\ellfinal}{N}}, \frac{400 H \ellfinal}{9 N} \right\}
        \\
        & \le \hat{P} V_{h+1}^k + \max\left\{ \frac{20}{3} \sqrt{\frac{\VV(\hat{P}, V_{h+1}^k)\ellfinal}{N}}, \frac{400 H \ellfinal}{9 N} \right\}
        \\
        & \le \hat{P} V_{h+1}^k + \frac{20}{3} \sqrt{\frac{\VV(\hat{P}, V_{h+1}^k)\ellfinal}{N}} + \frac{400 H \ellfinal}{9 N}
        \, ,
    \end{align*}
    where we use the condition $V_{h+1}^k \ge V_{h+1}^*$ for the second inequality.
    Therefore, we conclude that
    \begin{align*}
        Q_h^*(s, a, y)
        & \le r + \hat{P}V_{h+1}^* + 2 \sqrt{ \frac{\VV(\hat{P}, V_{h+1}^*) \ellfinal}{N} } + \frac{14 H \ellfinal}{3 N}
        \\
        & \le r + \hat{P} V_{h+1}^k + \frac{20}{3} \sqrt{\frac{\VV(\hat{P}, V_{h+1}^k)\ellfinal}{N}} + \frac{400 H \ellfinal}{9 N}
        \\
        & = Q_h^k(s, a, y)
        \, .
    \end{align*}
\end{proof}

We now prove \cref{lma:partially:optimism}.

\begin{proof}[Proof of \cref{lma:partially:optimism}]
    We prove the lemma by backward induction on $h$.
    The inequality is trivial for $h = H+1$.
    Suppose the inequality holds for $h + 1$.
    Then, by \cref{lma:partially:pre optimisim}, we have $Q_h^k(s, a, y) \ge Q_h^*(s, a, y)$.
    It implies that
    \begin{align*}
        Q_h^k(s, a) & = \sum_{y \in \Ycal} P_{\Ycal}(y \mid s, a) Q_h^k(s, a, y)
        \\
        & \ge \sum_{y \in \Ycal} P_{\Ycal}(y \mid s, a) Q_h^*(s, a, y)
        \\
        & = Q_h^*(s, a)
        \, .
    \end{align*}
    Since $V_h^k(s)$ is the maximum of $Q_h^k(s, a)$ over $a \in \Acal$, we have
    \begin{align*}
        V_h^k(s)  = \max_{a \in \Acal} Q_h^k(s, a) \ge \max_{a \in \Acal} Q_h^*(s, a) = V_h^*(s)
        \, .
    \end{align*}
    This completes the induction step.
\end{proof}

\subsection{Bounding Cumulative Regret}

For the remainder of the analysis, we combine the techniques developed in \citet{zhang2024settling,lee2025minimax} for simpler analysis.
We complete the proof of \cref{thm:partially known dynamics} by proving the following lemma:
\begin{lemma}
\label{lma:partially:final}
    Under $\Ecal$, the following inequality holds for all $k \in [K]$:
    \begin{align*}
        \sum_{k=1}^K \left( V_1^k(s_1^k) - V_1^{\pi^k}(s_1^k) \right) \le 56 H \sqrt{K ( \log eKH ) \sum_{z \in \effspace} \ellfinal(z)} + 2640 H ( \log eKH ) \sum_{z \in \effspace} B(z) \ell_2(z)
        \, .
    \end{align*}
\end{lemma}

\begin{proof}[Proof of \cref{lma:partially:final}]

Recall that $U_h^k$ and $\tilde{U}_h^k$ are functions defined as
\begin{align*}
    &U_h^k(s) := \Expec_{\substack{s' = (x', y') \sim P_{s, a}\\z = \Zmap(s, a, y')}} \left[ \frac{\ellfinal(z)}{N^k(z)} + U_{h+1}^k(s') \right] \land 1
    \, ,
    \\
    & \tilde{U}_h^k(s) := \Expec_{\substack{s' = (x', y') \sim P_{s, a}\\z = \Zmap(s, a, y')}} \left[ \frac{B(z)\ell_2(z)}{N^k(z)} + \tilde{U}_{h+1}^k(s') \right] \land 1
    \, ,
\end{align*}
for $h \in [H]$, where $a = \pi_h^k(s)$, starting from $U_{H+1}^k(s) = \tilde{U}_{H+1}^k(s) := 0$.
The following lemma shows that the instantaneous regret is upper bounded by $U_1^k$ and $\tilde{U}_1^k$.

\begin{lemma}
\label{lma:partially:per episode 1}
    Under $\Ecal$, we have
    \begin{align*}
        V_1^k(s_1^k) - V_1^{\pi^k}(s_1^k) \le 28 H \sqrt{U_1^k(s_1^k)} + 660 H \tilde{U}_1^k(s_1^k)
    \end{align*}
    for all $k \in [K]$.
\end{lemma}
We provide a high level proof sketch of \cref{lma:partially:per episode 1} here.
We express the instantaneous regret as the sum of the bonus terms and some extra terms as
\begin{align*}
    V_1^k(s_1^k) - V_1^{\pi^k}(s_1^k) \lesssim \Expec_{\pi^k} \left[\sum_{h=1}^H \left( \sqrt{\frac{\VV((\Peff)_{z_h^k}, V_{h+1}^*(\cdot, y_{h+1}^k) \ellfinal(z_h^k)}{N^k(z_h^k)}} + \frac{H B(z_h^k)\ell_2(z_h^k)}{N^k(z_h^k)} \right) \right]
    \, .
\end{align*}
By applying the Cauchy-Schwarz inequality twice on the first term, we obtain that
\begin{align*}
    \Expec_{\pi^k} \left[ \sum_{h=1}^H \sqrt{\frac{\VV((\Peff)_{z_h^k}, V_{h+1}^*(\cdot, y_{h+1}^k) \ellfinal(z_h^k)}{N^k(z_h^k)}} \right] \le \sqrt{\Expec_{\pi^k} \left[ \sum_{h=1}^H \VV((\Peff)_{z_h^k}, V_{h+1}^*(\cdot, y_{h+1}^k) \right] \Expec_{\pi^k} \left[ \sum_{h=1}^H \frac{\ellfinal(z_h^k)}{N^k(z_h^k)}\right]}
    \, .
\end{align*}
Then, we show that the expected sum of the variances are upper bounded by $H^2$, and the sum of the $\frac{\ellfinal(z)}{N(z)}$-type terms are expressed as $U_1^k$ and $\tilde{U}_1^k$, which yields \cref{lma:partially:per episode 1}.
The full proof of \cref{lma:partially:per episode 1} is deferred to \cref{appx:proof of per episode 1}.

Taking the sum over $k \in [K]$ and applying \cref{lma:partially:per episode 1}, we obtain that
\begin{align*}
    \sum_{k=1}^K \left( V_1^k(s_1^k) - V_1^{\pi^k}(s_1^k) \right)
    & \le \sum_{k=1}^K \left( 28 H \sqrt{U_1^k(s_1^k)} + 660 H \tilde{U}_1^k(s_1^k)\right)
    \\
    & \le 28 H \sqrt{K \sum_{k=1}^KU_1^k(s_1^k)} + 660 H \sum_{k=1}^K \tilde{U}_1^k(s_1^k)
    \, ,
\end{align*}
where we use the Cauchy-Schwarz inequality for the second inequality.
The sums of $U_1^k(s_1^k)$ and $\tilde{U}_1^k(s_1^k)$, which are the expected sums of $\frac{1}{N(z)}$-type of terms, are bounded by the following lemma, whose proof is deferred to \cref{appx:proof of sum of U}.

\begin{lemma}
\label{lma:sum of U}
    Under $\Ecal$, we have
    \begin{align*}
        \sum_{k=1}^K U_1^k(s_1^k) \le 4 (\log eKH)  \sum_{z \in \effspace} \ellfinal(z)
    \end{align*}
    and
    \begin{align*}
        \sum_{k=1}^K \tilde{U}_1^k(s_1^k) \le 4 (\log eKH)  \sum_{z \in \effspace} B(z) \ell_2 (z)
        \, .
    \end{align*}
\end{lemma}

By applying \cref{lma:sum of U}, we have
\begin{align*}
    \sum_{k=1}^K \left( V_1^k(s_1^k) - V_1^{\pi^k}(s_1^k) \right)
    & \le 28 H \sqrt{K \sum_{k=1}^KU_1^k(s_1^k)} + 660 H \sum_{k=1}^K \tilde{U}_1^k(s_1^k)
    \\
    & \le 56 H \sqrt{K ( \log e KH) \sum_{z \in \effspace} \ellfinal(z)} + 2640 H (\log e KH) \sum_{z \in \effspace} B(z) \ell_2(z)
    \, ,
\end{align*}
completing the proof.
\end{proof}

\begin{proof}[Proof of \cref{thm:partially known dynamics}]
    By \cref{lma:partially:optimism} and \cref{lma:partially:final}, we have
    \begin{align}
        \sum_{k=1}^K\left( V_1^*(s_1^k) - V_1^{\pi^k}(s_1^k) \right)
        & \le 
        \sum_{k=1}^K\left( V_1^k(s_1^k) - V_1^{\pi^k}(s_1^k) \right)
        \nonumber
        \\
        & = \Ocal \left( H \sqrt{K (\log KH) \sum_{z \in \effspace} \ellfinal(z)} + H (\log KH) \sum_{z \in \effspace} B(z) \ell_2(z) \right)
        \, .
        \label{eq:partially known tabular final bound}
    \end{align}
    Recall that $\ellfinal(z) = \log \frac{32 H^2 |\Ycal| |\effspace| K}{\delta} \land B(z)  \log \frac{32 HB(z) |\effspace| K}{\delta}$ and $\ell_2(z) = \log \frac{32 HB(z) |\effspace| K}{\delta}$.
    Omitting logarithmic factors on $H, K, \delta, B(z)$, and $|\effspace|$, we conclude that the regret bound is in the order of
    \begin{align*}
        \tilde{\Ocal} \left( H \sqrt{K \sum_{z \in \effspace} (\log | \Ycal| ) \land B(z) } + H \sum_{z \in \effspace} B(z) \right)
        \, .
    \end{align*}
\end{proof}

\subsection{Proof of Theorems~\ref{thm:known delay} and~\ref{thm:unknown delay}}
\label{appx:proof of delay tabular}

\begin{proof}[Proof of \cref{thm:known delay}]
    We first show that the augmented MDP constructed in \cref{sec:augmented MDP} satisfies \cref{assm:partially known dynamics tabular}.
    Recall that the augmented state space is defined as $\Scal_{\mathrm{aug}} = \Scal \times \cup_{D=0}^{\Dmax} \Acal \times \Dcal \times [H+1]$.
    We decompose it as $\Xcal = \Scal$ and $\Ycal = \cup_{D=0}^{\Dmax} \Acal \times \Dcal \times [H+1]$.
    Since the agent knows the dynamics of the action queue, the delay distribution, and the current time step, the agent is fully aware of the dynamics of $\Ycal$.
    We take $\effspace = \Scal \times (\Acal \cup \{\emptyset\})$, where the feature map is defined as $\Zmap((s, \QA, \tran, h), a, y) = (s, \QA_1)$ and $\Zmap((s, \QA, \tDelta, h), a, y) = (s, \emptyset)$ for $\tDelta \ne \tran$.
    The effective transition distribution becomes $\Peff(s, a) = P(s, a)$ for $ a\in \Acal$ and $\Peff(s, \emptyset) = \ind_s$.
    In this way, we have shown that the augmented MDP satisfies \cref{assm:partially known dynamics tabular}.

    Now, we show that \cref{thm:known delay} is a corollary of \cref{thm:partially known dynamics}.
    We have $|\Ycal| \le (A+1)^{\Dmax}(\Ddelay+2)(H+1)$, and hence $\log |\Ycal| \lesssim \Dmax$, omitting the logarithmic factors.
    For $z \in \effspace$, we have $B(z) \le B$.
    We also have $| \effspace | \le S(A+1)$.
    Applying \cref{thm:partially known dynamics}, where we may refer to Eq.~\eqref{eq:partially known tabular final bound} for logarithmic dependence, we conclude that the regret bound is at most
    \begin{align*}
        \Ocal \left( H \sqrt{(( \Dmax \log A +\iota) \land (B \iota)) SA K (\log KH) } + H BSA (\log KH) \iota \right)
        \, ,
    \end{align*}
    where $\iota = \log \frac{HSAK}{\delta}$.

    In order to achieve a tighter problem-dependent bound that scales with the actual length of the delay as mentioned in \cref{rmk:first order}, we take the union bound in a specially designed way.
    Instead of taking the uniform union bound over $\Ycal$, we assign probabilities depending on the length of $\QA$.
    Specifically, while we previously took the union bound by assigning the same probability of $\frac{\delta'}{(A+1)^{\Dmax}}$ to all $\QA \in \cup_{D=0}^{\Dmax} \Acal^D$, instead, we assign probability $\frac{\delta'}{2 (D+1)^2 A^D}$ to action queues $\QA$ in $\Acal^D$.
    By doing so, the $\Dmax \log A$ factor reduces to $D \log A$ for such elements.
    In the analysis, only the reachable states affect the regret bound.
    Hence, defining $\Dmax(s, a)$ as the maximum possible delay length of $D_h$ followed by $(s_h, a_h) = (s, a)$ for any $h\in [H]$, we achieve the regret bound whose $(\Dmax \land B) SA$ factor is replaced by $\sum_{(s, a) \in \Scal \times \Acal} (\Dmax(s, a) \land B)$.
\end{proof}

\begin{proof}[Proof of \cref{thm:unknown delay}]
    \cref{sec:proof sketch} largely explains the main processes, and we fill in some missing details.
    To show that the augmented MDP satisfies \cref{assm:partially known dynamics tabular}, we decompose the state space as $\Xcal = \Scal \times \Dcal$ and $\Ycal = \cup_{D=0}^{\Dmax} \Acal^D \times [H+1]$.
    We take $\effspace = \Scal\times \Acal \times \Dcal \times \{\emptyset, 0, 1 \}$, where we need the last component to handle exceptional cases such as truncation of the delay or the termination of the episode.
    For most augmented states, we have $\Zmap((s, \QA, \tDelta, h), a, y) = (s, \QA_1, \tDelta, \emptyset)$.
    The last component is set to $0$ when we have $\len(\QA) = 0$ and $\tDelta = -1$, meaning that there is no state to reveal and the delay must be set to $0$.
    The last component is set to $1$ when we have either $\len(\QA) = \Dmax$ and $\tDelta \in [0 : \Ddelay]$, which corresponds to case that the delay is truncated by $\Dmax$, or $h = H+1$, which corresponds to the termination of the episode.
    \\
    The effective transition distribution is defined as $\Peff(s, a, \tran, \cdot) = P_{s, a} \times \ind_{\tDelta = -1}$ and $\Peff(s, a, \tDelta, \emptyset) = P_{\tran}(s, a, \tDelta) \ind_{s, a, \tran} + (1 - P_{\tran}(s, a, \tDelta)) \ind_{s, a, \tDelta + 1}$ for $\tDelta \ne \tran$.
    If the last variable is 0 or 1 in the latter case, then $P_{\tran}$ is replaced by 0 or 1, respectively.
    In this way, we can correctly classify the augmented MDP as a special case of \cref{assm:partially known dynamics tabular}.

    Showing that \cref{thm:unknown delay} follows from \cref{thm:partially known dynamics} is exactly as explained in \cref{sec:proof sketch}, where we obtain the regret bound of
    \begin{align*}
        \Ocal \Big( & H \sqrt{((\Dmax \log A + \iota) \land (B\iota) + \Ddelay \iota) SAK (\log KH)}
        \\
        & \qquad + H (B + \Ddelay) SA (\log KH) \iota\Big)
        \, .
    \end{align*}
    
    Achieving the improved problem-dependent bound described in \cref{rmk:first order} is done in the same way with \cref{thm:known delay}.
\end{proof}

\section{Proofs of Technical Lemmas in Appendix~\ref{appx:proof of partially}}

\subsection{Proof of Lemma~\ref{lma:all functions concentration}}
\label{appx:proof of all functions concentration}

We first prove the following lemma.

\begin{lemma}
\label{lma:probability concentration}
    The following inequality holds for all $x \in \Xcal$, $z \in \effspace$, and $k \in [K]$ that satisfies $N^k(z) \ge 2$ with probability at least $1 - \frac{\delta}{8}$:
    \begin{align*}
        \left| (\Peff)_z(x) - (\hatPeff^k)_z(x) \right| \le \sqrt{ \frac{ 2 (\Peff)_z(x) \ell_2(z) }{N^k(z)}} + \frac{ \ell_2(z)}{3 N^k(z)}
        \, .
    \end{align*}
\end{lemma}
\begin{proof}
    Fix $x \in \Xcal$, $z \in \effspace$, and $n \ge 2$.
    If $(\Peff)_z(x) = 0$, then the inequality is trivial.
    Suppose $x \in \Xcal(z)$.
    Apply \cref{lma:bernstein} to $Z = \ind\{ x' = x\}$ with $x' \sim \Peff(z)$.
    Note that in this case, we have $\Var(Z) = (\Peff)_z(x) (1 - (\Peff)_z(x)) \le (\Peff)_z(x)$.
    Then, we obtain that the following inequality holds with probability at least $1 - \frac{\delta}{8}$:
    \begin{align*}
        \left| (\Peff)_z(x) - (\hatPeff^k)_z(x) \right| \le  \sqrt{ \frac{ 2 (\Peff)_z(x) \log \frac{32}{\delta} }{N^k(z)}} + \frac{  \log \frac{32}{\delta}}{3 
        N^k(z)}
        \, .
    \end{align*}
    The proof is completed by taking the union bound over $z \in \effspace$, corresponding $x \in \Xcal(z)$, and $2 \le n \le KH$, where we use that $|\Xcal (z) | \le B(z)$ by the definition of $B(z)$.
\end{proof}

\begin{proof}[Proof of \cref{lma:all functions concentration}]
    Fix $z \in \effspace$ and $k \in [K]$ with $N^k(z)\ge 1$.
    Since $\sum_{x \in \Xcal(z)}(\Peff)_z(x)= \sum_{x \in \Xcal(z)} (\hatPeff^k)_z(x) =1$, we have
    \begin{align*}
        \left((\Peff)_z-(\hatPeff^k)_z\right)V
        =
        \left((\Peff)_z-(\hatPeff^k)_z\right)\left(V-(\Peff)_zV\right),
    \end{align*}
    and by the triangle inequality,
    \begin{align*}
        \left|\left((\Peff)_z-(\hatPeff^k)_z\right)\left(V-(\Peff)_zV\right)\right|
        \le
        \sum_{x\in\Xcal(z)}
        \left|
        \left((\Peff)_z(x)-(\hatPeff^k)_z(x)\right)\left(V(x)-(\Peff)_zV\right)
        \right|.
    \end{align*}
    
    Under the event of \cref{lma:probability concentration}, for all $x\in\Xcal$, we have
    \begin{align*}
        \left|(\Peff)_z(x)- (\hatPeff^k)_z(x)\right|
        \le
        \sqrt{\frac{2 (\Peff)_z(x) \ell_2(z)}{N^k(z)}}+\frac{\ell_2(z)}{3 N^k(z)} \, .
    \end{align*}
    Multiplying both sides by $\left|V(x)-(\Peff)_zV\right|$ and summing over $x\in\Xcal(z)$ gives
    \begin{align*}
        &\sum_{x\in\Xcal(z)}
        \left|
        \left((\Peff)_z(x)-(\hatPeff^k)_z(x)\right)\left(V(x)-(\Peff)_zV\right)
        \right|
        \\
        &\le
        \sum_{x\in\Xcal(z)}
        \left(
        \sqrt{\frac{2 (\Peff)_z(x)\ell_2(z)}{N^k(z)}} \left|V(x)-(\Peff)_zV\right|
        +
        \frac{\ell_2(z)}{3 N^k(z)}\left|V(x)-(\Peff)_zV\right|
        \right)
        \\
        &\le
        \sum_{x\in\Xcal(z)}
        \sqrt{\frac{2 (\Peff)_z(x)\left(V(x)-(\Peff)_zV\right)^2\,\ell_2(z)}{N^k(z)}}
        +
        \frac{\ell_2(z)}{3 N^k(z)}
        \sum_{x\in\Xcal(z)}\left|V(x)-(\Peff)_zV\right|
        \\
        &\le
        \sqrt{\frac{2 B(z) \ell_2(z)}{N^k(z)}
        \sum_{x\in\Xcal(z)}(\Peff)_z(x)\left(V(x)-(\Peff)_zV\right)^2}
        +
        \frac{\ell_2(z)}{3 N^k(z)}\cdot B(z) (b-a)
        \\
        &=
        \sqrt{\frac{2 \VV\left((\Peff)_z,V\right) B(z) \ell_2(z)}{N^k(z)}}+\frac{(b - a) B(z) \ell_2(z)}{3 N^k(z)} \, ,
    \end{align*}
    where we use the Cauchy-Schwarz inequality and that $\left|V(x)-(\Peff)_zV\right|\le b-a$ in the third inequality.
\end{proof}

\subsection{Proof of Lemma~\ref{lma:partially:per episode 1}}

\label{appx:proof of per episode 1}

For the proof of \cref{lma:partially:per episode 1}, we require the following lemmas.

\begin{lemma}
\label{lma:surplus bound}
    Fix $s \in \Scal$, $h \in [H]$, and $a = \pi_h^k(s)$.
    Under $\Ecal$, we have
    \begin{align*}
        V_h^k(s) - r(s, a) - P_{s, a}V_{h+1}^k
        & \le \frac{1}{2H} \left( P_{s, a}(V_{h+1}^k - V_{h+1}^*)^2 - (V_{h}^k(s) - V_{h}^*(s))^2 \right)
        \\
        & \qquad + \Expec_{\substack{y' \sim P_{\Ycal}(s, a) \\z = \Zmap(s, a, y')}} \left[28 \sqrt{ \frac{\VV((\Peff)_z, V_{h+1}^*(\cdot, y')) \ellfinal(z) }{N^k(z)}} 
        + \frac{660 H B(z) \ell_2(z)}{N^k(z)} \right] 
        \, .
    \end{align*}
\end{lemma}
The proof of \cref{lma:surplus bound} is quite technical, and it is deferred to \cref{appx:proof of surplus bound}.
Using \cref{lma:surplus bound}, we can prove the following bound for the per-episode regret.

\begin{lemma}
\label{lma:partially:per episode 2}
    Under $\Ecal$, we have
    \begin{align*}
        V_1^k(s_1^k) - V_1^{\pi^k}(s_1^k) \le 28 \sqrt{ \Expec_{\pi^k} \left[\sum_{h=1}^H \VV((\Peff)_{z_h^k}, V_{h+1}^*(\cdot, y_{h+1}^k)) \right] U_1^k(s_1^k)} + 660 H \tilde{U}_1^k(s_1^k) 
    \end{align*}
    for all $k \in [K]$.
\end{lemma}
\begin{proof}
    We define
    \begin{align*}
        W_h^k(s) := \Expec_{\pi^k(\cdot \mid s_h = s)}\left[ \sum_{j=h}^H \VV((\Peff)_{z_j^k}, V_{j+1}^*(\cdot, y_{j+1}^k))  \right]
        \, ,
    \end{align*}
    which is the expectation of the sum of variances.
    Note that we have $W_h^k(s) =  \Expec_{\substack{s' = (x', y') \sim P_{s, a}\\z = \Zmap(s, a, y')}} [ \VV((\Peff)_z, V_{h+1}^*(\cdot, y')) + W_{h+1}^k(s')]$.
    We prove that the following inequality holds for all $h\in[H]$ and $ s\in \Scal$ by backward induction on $h$:
    \begin{align*}
        V_h^k(s) - V_h^{\pi^k}(s) & \le 28 \sqrt{W_h^k(s) U_h^k(s)} + 660 H \tilde{U}_h^k(s) 
        \\
        & \qquad  -  \frac{1}{2H} \left(V_{h}^k(s) - V_{h}^*(s)\right)^2
        \, .
    \end{align*}
    The inequality is trivial for $h = H+1$.
    Suppose the inequality holds for $h+1$.
    We note that the inequality is also trivial when $\tilde{U}_h^k(s) = 1$, as the left-hand side is at most $H$, while the right-hand side becomes larger than $H$.
    Suppose $\tilde{U}_h^k(s) < 1$.
    We note that we have $U_h^k(s) \le \tilde{U}_h^k$, hence we also have $U_h^k(s) < 1$.
    We perform the following decomposition:
    \begin{align*}
        V_h^k(s) - V_h^{\pi^k}(s) = \underbrace{V_h^k(s) - r(s, a) - P_{s, a}V_{h+1}^k}_{I_1} +\underbrace{P_{s, a}(V_{h+1}^k - V_{h+1}^{\pi^k})}_{I_2}
        \, .
    \end{align*}
    $I_1$ is bounded by \cref{lma:surplus bound} as
    \begin{align*}
        I_1 & \le \underbrace{ \frac{1}{2H} \left( P_{s, a}(V_{h+1}^k - V_{h+1}^*)^2 - (V_{h}^k(s) - V_{h}^*(s))^2 \right)}_{I_3}
        \\
        & \qquad + \underbrace{\Expec_{\substack{y' \sim P_{\Ycal}(s, a) \\z = \Zmap(s, a, y')}} \left[28 \sqrt{ \frac{\VV((\Peff)_z, V_{h+1}^*(\cdot, y')) \ellfinal(z) }{N^k(z)}} \right]}_{I_4} + 
        \underbrace{\Expec_{\substack{y' \sim P_{\Ycal}(s, a) \\z = \Zmap(s, a, y')}} \left[\frac{660 H B(z) \ell_2(z)}{N^k(z)} \right] }_{I_5}
        \, .
    \end{align*}
    $I_2$ is bounded by the induction hypothesis as
    \begin{align*}
        I_2 & \le \underbrace{ \Expec_{s' \sim P_{s, a}} \left[28\sqrt{W_{h+1}^k(s') U_{h+1}^k(s') } \right]}_{I_6} + \underbrace{ 660 H P_{s, a} \tilde{U}_{h+1}^k}_{I_7} - \underbrace{\frac{1}{2H} P_{s, a}(V_{h+1}^k - V_{h+1}^*)^2}_{I_8}
        \, .
    \end{align*}
    We have $I_3 - I_8 = -\frac{1}{2H}(V_h^k(s) - V_h^*(s))^2$ and $I_5 + I_7 = 660H \tilde{U}_h^k(s)$.
    It remains to bound $I_4 + I_6$.
    Using the Cauchy-Schwarz inequality, we obtain
    \begin{align*}
        & \sqrt{ \frac{\VV((\Peff)_z, V_{h+1}^*(\cdot, y')) \ellfinal(z) }{N^k(z)}} + \sqrt{W_{h+1}^k(s') U_{h+1}^k(s')}
        \\
        & \le \sqrt{\VV((\Peff)_z, V_{h+1}^*(\cdot, y')) + W_{h+1}^k(s') } \sqrt{\frac{\ellfinal(z)}{N^k(z)} + U_{h+1}^k(s')}
        \, .
    \end{align*}
    By the Cauchy-Schwarz inequality for expectation and the law of total expectation, we obtain that
    \begin{align*}
        I_4 + I_6
        & \le 28 \Expec_{\substack{s' = (x', y') \sim P_{s, a}\\ z = \Zmap(s, a, y')}} \left[ \sqrt{\VV((\Peff)_z, V_{h+1}^*(\cdot, y')) + W_{h+1}^k(s') } \sqrt{\frac{\ellfinal(z)}{N^k(z)} + U_{h+1}^k(s')} \right]
        \\
        &\le 28 \sqrt{\Expec_{\substack{s' = (x', y') \sim P_{s, a}\\ z = \Zmap(s, a, y')}} \left[ \VV((\Peff)_z, V_{h+1}^*(\cdot, y')) + W_{h+1}^k(s') \right]}
        \\
        & \qquad \times \sqrt{\Expec_{\substack{s' = (x', y') \sim P_{s, a}\\ z = \Zmap(s, a, y')}} \left[\frac{\ellfinal(z)}{N^k(z)} + U_{h+1}^k(s')  \right]}
        \\
        & = 28 \sqrt{W_{h}^k(s) {U_h^k(s)}}
        \, .
    \end{align*}
    Combining the bounds for $I_3 - I_8$ and $I_5 + I_7$, we obtain the desired bound for $V_h^k(s ) - V_h^{\pi^k}(s)$, completing the induction step.
\end{proof}

\begin{proof}[Proof of \cref{lma:partially:per episode 1}]
    Given \cref{lma:partially:per episode 2}, it suffices to prove that $\Expec_{\pi^k} \left[\sum_{h=1}^H \VV((\Peff)_{z_h^k}, V_{h+1}^*(\cdot, y_{h+1}^k)) \right]  \le H^2$.
    By \cref{lma:eff variance decomposition}, we have
    \begin{align*}
        & \Expec_{y_{h+1}^k \sim P_{\Ycal}(s_h^k, a_h^k)} \left[ \VV((\Peff)_{z_h^k}, V_{h+1}^*(\cdot, y_{h+1}^k)) \right]
        \\
        & \le P_{s_h^k, a_h^k}(V_{h+1}^*)^2 - (V_h^*(s_h^k))^2 + 2 H \max\{ V_h^*(s_h^k) - P_{s_h^k, a_h^k}V_{h+1}^*, 0 \}
        \\
        & = P_{s_h^k, a_h^k}(V_{h+1}^*)^2 - (V_h^*(s_h^k))^2 + 2 H \left( V_h^*(s_h^k) - P_{s_h^k, a_h^k}V_{h+1}^* \right)
        \, ,
    \end{align*}
    where we use that $V_h^*(s_h^k) - P_{s_h^k, a_h^k}V_{h+1}^* \ge Q_h^*(s_h^k) - P_{s_h^k, a_h^k}V_{h+1}^* = r(s_h^k, a_h^k) \ge 0$ for the last equality.
    Taking the sum over $h \in [H]$ and taking the expectation, the terms are telescoped and we derive
    \begin{align*}
        \Expec_{\pi^k} \left[\sum_{h=1}^H \VV((\Peff)_{z_h^k}, V_{h+1}^*(\cdot, y_{h+1}^k)) \right] 
        & \le - (V_1^*(s_1^k))^2 + 2H V_1^*(s_1^k)
        \\
        & = H^2 - (V_1^*(s_1^k) - H)^2
        \\
        & \le H^2
        \, .
    \end{align*}
    Plugging this bound into \cref{lma:partially:per episode 2}, we obtain \cref{lma:partially:per episode 1}.
\end{proof}

\subsection{Proof of Lemma~\ref{lma:sum of U}}
\label{appx:proof of sum of U}

\begin{proof}[Proof of \cref{lma:sum of U}]
    Under the event of \cref{lma:concentration of U}, we have
    \begin{align*}
        \sum_{k=1}^K U_1^k(s_1^k) & \le 2 \sum_{k=1}^K \sum_{h=1}^H \ind\{2 N^k(z_h^k) > N_h^k(z_h^k) \}\frac{\ellfinal(z_h^k)}{N^k(z_h^k)} + 3 |\effspace | \log \frac{8H}{\delta}
        \, .
    \end{align*}
    Under $2N^k(z_h^k) > N_h^k(z_h^k)$, we have $\frac{1}{N^k(z_h^k)} \le \frac{2}{N_h^k(z_h^k)}$, and we also have $N_h^k(z_h^k) \ge 2$, and hence the first sum is bounded as follows:
    \begin{align*}
        2 \sum_{k=1}^K \sum_{h=1}^H \ind\{2 N^k(z_h^k) > N_h^k(z_h^k) \}\frac{\ellfinal(z_h^k)}{N^k(z_h^k)}
        & \le \sum_{k=1}^K \sum_{h=1}^H \frac{4\ind\{ N_h^k(z_h^k) \ge 2\} \ellfinal(z_h^k)}{N_h^k(z_h^k)}
        \\
        & = 4 \sum_{z \in \effspace} \ellfinal(z) \sum_{n=2}^{N^{K+1}(z)}\frac{1}{n}
        \\
        & \le 4 \sum_{z \in \effspace} \ellfinal(z) (\log KH)
        \, .
    \end{align*}
    We add $3 |\effspace | \log \frac{8H}{\delta}$ and using that $\log \frac{8H}{\delta} \le \ellfinal(z)$, we derive that
    \begin{align*}
        \sum_{k=1}^K U_1^k(s_1^k) & \le 4 \sum_{z \in \effspace} \ellfinal(z) (\log KH) + 3 |\effspace | \log \frac{8H}{\delta}
        \\
        & \le 4 \sum_{z \in \effspace} (\ellfinal(z) (\log KH) + 1)
        \\
        & = 4 ( \log eKH) \sum_{z \in \effspace} \ellfinal(z)
        \, ,
    \end{align*}
    proving the first inequality.
    The second inequality is proved in the exact same way.
\end{proof}

\subsection{Proof of Lemma~\ref{lma:surplus bound}}
\label{appx:proof of surplus bound}

We require two additional lemmas to prove \cref{lma:surplus bound}.

\begin{lemma}
\label{lma:upper bound the bonus term}
    Under $\Ecal$, the following inequality holds for all $z \in \effspace$, $y \in \Ycal$, $h \in [H]$, and $k \in [K]$ with $N^k(z) \ge 2$:
    \begin{align*}
        & \frac{20}{3} \sqrt{ \frac{\VV((\hatPeff)_z V_{h+1}^k(\cdot, y)) \ellfinal(z)}{N^k(z)}} + \frac{400 H \ellfinal(z)}{9 N^k(z)}
        \\
        & \le 
        12 \sqrt{ \frac{\VV((\Peff)_z, V_{h+1}^*(\cdot, y)) \ellfinal(z) }{N^k(z)}} + \frac{1}{8H} \VV((\Peff)_z,(V_{h+1}^k - V_{h+1}^*)(\cdot, y) ) 
        + \frac{325 H B(z) \ell_2(z)}{N^k(z)}
        \, .
    \end{align*}
\end{lemma}

\begin{proof}
    We fix $z \in \effspace$, $y \in \Ycal$, $h \in [H]$, and $k \in [K]$.
    Denote $P := (\Peff)_z$, $\hat{P} := (\hatPeff)_z$, $V_{h+1}^* := V_{h+1}^*(\cdot, y)$, $V_{h+1}^k := V_{h+1}^k(\cdot, y)$, $N := N^k(z)$, and $\ellfinal := \ellfinal(z)$ for simplicity.
    Under the event of \cref{lma:all functions concentration}, by applying the lemma to $V(\cdot) = (V_{h+1}^k(\cdot) - PV_{h+1}^k)^2$ we have
    \begin{align}
        (\hat{P} - P)(V_{h+1}^k - P V_{h+1}^k)^2 & \le \sqrt{\frac{2 \VV(P, (V_{h+1}^k - PV_{h+1}^k)^2) B(z) \ell_2(z)}{N}} + \frac{H^2 B(z) \ell_2(z)}{3 N}
        \nonumber
        \\
        & \le \frac{1}{2 H^2}\VV(P, (V_{h+1}^k - PV_{h+1}^k)^2) + \frac{4  H^2 B(z) \ell_2(z)}{3N}
        \nonumber
        \\
        & \le \frac{1}{2} P (V_{h+1}^k - PV_{h+1}^k)^2 + \frac{4 H^2 B(z) \ell_2(z)}{3N}
        \label{eq:empirical variance to variance 1}
        \, ,
    \end{align}
    where we use the AM-GM inequality for the second inequality and that $\Var(Z) \le \EE[Z^2] \le c \EE[Z]$ for random variables $Z \in [0, c]$ for the third inequality.
    Then, we bound the empirical variance of $V_{h+1}^k$ as follows:
    \begin{align*}
        \VV(\hat{P}, V_{h+1}^k)
        &  = \hat{P} \left( V_{h+1}^k - \hat{P}V_{h+1}^k \right)^2
        \\
        & \le \hat{P} \left( V_{h+1}^k - P V_{h+1}^k \right)^2
        \\
        & \le \frac{3}{2} P \left( V_{h+1}^k - P V_{h+1}^k \right)^2 + \frac{4 H^2 B(z) \ell_2(z)}{3N}
        \\
        & = \frac{3}{2} \VV(P, V_{h+1}^k) + \frac{4 H^2 B(z) \ell_2(z)}{3N}
        \\
        & \le 3 \VV(P, V_{h+1}^*) + 3 \VV(P, V_{h+1}^k - V_{h+1}^*) + \frac{4 H^2 B(z) \ell_2(z)}{3N}
        \, ,
    \end{align*}
    where the second inequality uses inequality~\eqref{eq:empirical variance to variance 1}, and the last inequality uses that $\Var(X+Y) \le 2 \Var(X) + 2 \Var(Y)$ for any random variables $X$ and $Y$.
    Applying this bound, we have
    \begin{align}
         & \sqrt{ \frac{\VV(\hat{P}, V_{h+1}^k) \ellfinal}{N}} 
         \nonumber
         \\
         & \le \sqrt{ \frac{ \left(3 \VV(P, V_{h+1}^*) + 3 \VV(P, V_{h+1}^k - V_{h+1}^*) + \frac{4 H^2 B(z) \ell_2(z)}{3 N} \right) \ellfinal}{N}}
         \nonumber
         \\
         & \le \sqrt{ \frac{3 \VV(P, V_{h+1}^*) \ellfinal }{N}}
         + \sqrt{\frac{3 \VV(P,V_{h+1}^k - V_{h+1}^*) \ellfinal }{N}} 
         + \frac{2 H  B(z) \ell_2(z) }{\sqrt{3} N}
         \label{eq:empirical variance decomposition 1}
         \, ,
    \end{align}
    where the last inequality uses that $\sqrt{a+b+c} \le \sqrt{a} + \sqrt{b} + \sqrt{c}$, and $ \ellfinal(z) \le B(z) \ell_2(z)$ for the last term.
    For the second term, we apply the AM-GM inequality and obtain
    \begin{align}
        \sqrt{\frac{3 \VV(P,V_{h+1}^k - V_{h+1}^*) \ellfinal }{N}}
        & \le \frac{3 \VV(P,V_{h+1}^k - V_{h+1}^*)}{160 H} + \frac{40 H \ellfinal}{N}
        \nonumber
        \\
        & \le \frac{3 \VV(P,V_{h+1}^k - V_{h+1}^*)}{160 H} + \frac{40 H B(z)\ell_2(z)}{N}
         \label{eq:empirical variance decomposition 2}
        \, .
    \end{align}
    Combining inequalities~\eqref{eq:empirical variance decomposition 1} and~\eqref{eq:empirical variance decomposition 2}, we have
    \begin{align*}
         \sqrt{ \frac{\VV(\hat{P}, V_{h+1}^k) \ellfinal}{N}} 
         & \le \sqrt{ \frac{3 \VV(P, V_{h+1}^*) \ellfinal}{N}} +\frac{3 \VV(P, V_{h+1}^k - V_{h+1}^*)}{160 H} + \frac{42 H B(z) \ell_2(z)}{N}
         \, .
    \end{align*}
    Finally, we obtain that
    \begin{align*}
        \frac{20}{3} \sqrt{ \frac{\VV(\hat{P} V_{h+1}^k) \ellfinal}{N}} + \frac{400 H \ellfinal}{9 N}
        & \le 12 \sqrt{ \frac{\VV(P, V_{h+1}^*) \ellfinal }{N}} + \frac{1}{8H} \VV(P, V_{h+1}^k - V_{h+1}^* ) + \frac{280 H B(z) \ell_2(z)}{N} + \frac{400 H\ellfinal }{9 N}
        \\
        & \le 12 \sqrt{ \frac{\VV(P, V_{h+1}^*) \ellfinal }{N}} + \frac{1}{8H} \VV(P, V_{h+1}^k - V_{h+1}^* ) 
        + \frac{325 H B(z) \ell_2(z)}{N}
        \, .
    \end{align*}
\end{proof}

\begin{lemma}
\label{lma:upper bound on est est}
    Under $\Ecal$, we have that for all $z \in \effspace$, $y \in \Ycal$, $h \in [H]$, and $k \in [K]$ with $N^k(z) \ge 1$, we have
    \begin{align*}
        ((\hatPeff^k)_z - (\Peff)_z) V_{h+1}^k(\cdot, y) \le \sqrt{\frac{2 \VV((\Peff)_z, V_{h+1}^*(\cdot, y)) \ellfinal(z) }{N^k(z)}} + \frac{1}{8H}\VV((\Peff)_z, (V_{h+1}^k - V_{h+1}^*)(\cdot, y) + \frac{14 H B(z) \ell_2(z)}{3N^k(z)}
        \, .
    \end{align*}
\end{lemma}
\begin{proof}
    Denote $P := (\Peff)_z$, $\hat{P} := (\hatPeff)_z$, $V_{h+1}^* := V_{h+1}^*(\cdot, y)$, $V_{h+1}^k := V_{h+1}^k(\cdot, y)$, $N := N^k(z)$, and $\ellfinal := \ellfinal(z)$.
    We begin by adding and subtracting $(\hat{P} - P)V_{h+1}^*$.
    \begin{align}
        (\hat{P}  - P) V_{h+1}^k 
        & = (\hat{P} - P)V_{h+1}^* + (\hat{P} - P)(V_{h+1}^k - V_{h+1}^*)
        \label{eq:estimate concentration 1}
        \, .
    \end{align}
    By \cref{lma:optimal value concentration,lma:all functions concentration}, we have
    \begin{align}
        (\hat{P} - P)V_{h+1}^* \le \sqrt{\frac{2 \VV(P, V_{h+1}^*) \ellfinal}{N}} + \frac{H \ellfinal}{3N}
        \label{eq:estimate concentration 2}
        \, ,
    \end{align}
    where we use that $\ellfinal := \ellfinal(z) = \ell_1 \land B(z) \ell_2(z)$.
    By \cref{lma:all functions concentration}, and using that $0 \le V_{h+1}^k(x) - V_{h+1}^*(x) \le H$ for all $x \in \Xcal$ by \cref{lma:partially:optimism}, we have
    \begin{align}
        (\hat{P} - P)(V_{h+1}^k - V_{h+1}^*) & \le \sqrt{\frac{2 \VV(P, V_{h+1}^k - V_{h+1}^*) B(z) \ell_2(z)}{N}} + \frac{H B(z) \ell_2(z)}{3 N}
        \nonumber
        \\
        & \le \frac{1}{8H}\VV(P, V_{h+1}^k - V_{h+1}^*) + \frac{13 H B(z) \ell_2(z)}{3N}
        \label{eq:estimate concentration 3}
        \, ,
    \end{align}
    where we use the AM-GM inequality for the last inequality.
    Plugging inequalities~\eqref{eq:estimate concentration 2} and~\eqref{eq:estimate concentration 3} into inequality~\eqref{eq:estimate concentration 1}, we obtain that
    \begin{align*}
        (\hat{P}  - P) V_{h+1}^k 
        & \le \sqrt{\frac{2 \VV(P, V_{h+1}^*) \ellfinal}{N}} + \frac{H \ellfinal}{3N} + \frac{1}{8H}\VV(P, V_{h+1}^k - V_{h+1}^*) + \frac{13 H B(z) \ell_2(z)}{3N}
        \\
        & \le \sqrt{\frac{2 \VV(P, V_{h+1}^*) \ellfinal}{N}} + \frac{1}{8H}\VV(P, V_{h+1}^k - V_{h+1}^*) + \frac{14 H B(z) \ell_2(z)}{3N}
        \, ,
    \end{align*}
    where we use that $\ellfinal \le B(z)\ell_2(z)$ by definition for the last inequality.
\end{proof}

\begin{proof}[Proof of \cref{lma:surplus bound}]
    We let $b_h^k(s, a, y') := \frac{20}{3} \sqrt{\frac{\VV((\hatPeff)_z, V_{h+1}^k(\cdot, y') \ellfinal(z)}{N^k(z)}} + \frac{400 H \ellfinal(z)}{N^k(z)}$, where $z = \Zmap(s, a, y')$.
    The value of $b_h^k(s, a, y')$ corresponds to the optimistic bonus term added by $\MVP$.
    We first expand $V_h^k(s)$ as follows:
    \begin{align*}
        V_h^k(s)
        & = \Expec_{y' \sim P_{\Ycal}(s, a)} [ Q_h^k(s, a, y') ]
        \\
        & \le \Expec_{y' \sim P_{\Ycal}(s, a)} \left[ r(s, a) + (\hatPeff^k)_z V_{h+1}^k(\cdot, y') + b_h^k(s, a, y') \right]
        \\
        & = r(s, a) + \Expec_{y' \sim P_{\Ycal}(s, a) } \left[ (\hatPeff^k)_z V_{h+1}^k(\cdot, y') + b_h^k(s, a, y') \right]
        \, .
    \end{align*}
    Then, we obtain that
    \begin{align*}
        V_h^k(s) - r(s, a) - P_{s, a}V_{h+1}^k
        & \le \Expec_{y' \sim P_{\Ycal}(s, a)} \left[ b_h^k(s, a, y') + ((\hatPeff^k)_z - (\Peff)_z) V_{h+1}^k(\cdot, y' ) \right]
        \, .
    \end{align*}
    By \cref{lma:upper bound the bonus term,lma:upper bound on est est}, we have
    \begin{align*}
        & b_h^k(s, a, y') + ((\hatPeff^k)_z - (\Peff)_z) V_{h+1}^k(\cdot, y' )
        \\
        & \le 14 \sqrt{ \frac{\VV((\Peff)_z, V_{h+1}^*(\cdot, y')) \ellfinal(z) }{N^k(z)}} + \frac{1}{4H} \VV((\Peff)_z,(V_{h+1}^k - V_{h+1}^*)(\cdot, y') ) 
        + \frac{330 H B(z) \ell_2(z)}{N^k(z)}
        \, .
    \end{align*}
    Let $I_1 := \max\{ V_h^k(s) - r(s, a) - P_{s, a}V_{h+1}^k, 0 \}$.
    Then, we obtain
    \begin{align}
        I_1 \le \Expec_{y' \sim P_{\Ycal}(s, a)} \left[14 \sqrt{ \frac{\VV((\Peff)_z, V_{h+1}^*(\cdot, y')) \ellfinal(z) }{N^k(z)}} + \frac{1}{4H} \VV((\Peff)_z,(V_{h+1}^k - V_{h+1}^*)(\cdot, y') ) 
        + \frac{330 H B(z) \ell_2(z)}{N^k(z)} \right]
        \, .
        \label{eq:surplus 1}
    \end{align}
    Now, we focus on the second term in the expectation.
    Denoting $\tilde{V}_h^k(s) := V_h^k(s) - V_h^*(s)$, we know that $0 \le \tilde{V}_h^k(s) \le H$ by \cref{lma:partially:optimism} and bounds on $V_h^k(s) \le H$ and $V_h^*(s) \ge 0$.
    By \cref{lma:eff variance decomposition}, we have
    \begin{align*}
        \VV((\Peff)_z, \tilde{V}_{h+1}^k(\cdot, y')) \le P_{s, a}(\tilde{V}_{h+1}^k)^2 - (\tilde{V}_h(s))^2 + 2 H \max \left\{ \tilde{V}_h^k(s) - P_{s, a}\tilde{V}_{h+1}^k, 0 \right\}
        \, .
    \end{align*}
    Note that we have
    \begin{align*}
        \tilde{V}_h^k(s) - P_{s, a} \tilde{V}_{h+1}^k
        & = V_h^k(s) - V_h^*(s) - P_{s, a}(V_{h+1}^k - V_{h+1}^*)
        \\
        & \le V_h^k(s) - Q_h^*(s, a) - P_{s, a}V_{h+1}^k + P_{s, a} V_{h+1}^*
        \\
        & = V_h^k(s) - r(s, a) - P_{s, a}V_{h+1}^k
        \, ,
    \end{align*}
    and hence
    \begin{align*}
        \VV((\Peff)_z, \tilde{V}_{h+1}^k(\cdot, y')) \le P_{s, a}(\tilde{V}_{h+1}^k)^2 - (\tilde{V}_h^k(s))^2 + 2 H I_1
        \, .
    \end{align*}
    Plugging this bound into inequality~\eqref{eq:surplus 1}, we obtain
    \begin{align*}
        I_1 \le & \frac{1}{2}I_1 + \frac{1}{4H} \left( P_{s, a}(V_{h+1}^k - V_{h+1}^*)^2 - (V_{h}^k(s) - V_{h}^*(s))^2 \right)
        \\
        & \qquad + \Expec_{y' \sim P_{\Ycal}(s, a)} \left[14 \sqrt{ \frac{\VV((\Peff)_z, V_{h+1}^*(\cdot, y')) \ellfinal }{N^k(z)}} 
        + \frac{330 H B(z) \ell_2(z)}{N^k(z)} \right] 
        \, .
    \end{align*}
    Solving the inequality with respect to $I_1$, we conclude that
    \begin{align*}
        I_1
        & \le  \frac{1}{2H} \left( P_{s, a}(V_{h+1}^k - V_{h+1}^*)^2 - (V_{h}^k(s) - V_{h}^*(s))^2 \right)
        \\
        & \qquad + \Expec_{y' \sim P_{\Ycal}(s, a)} \left[28 \sqrt{ \frac{\VV((\Peff)_z, V_{h+1}^*(\cdot, y')) \ellfinal }{N^k(z)}} 
        + \frac{660 H B(z) \ell_2(z)}{N^k(z)} \right] 
        \, .
    \end{align*}
\end{proof}

\subsection{Additional Technical Lemmas}

\begin{lemma}
\label{lma:eff variance decomposition}
    Let $c \ge 0$ be a constant.
    Let $\{V_h\}_{h=1}^{H+1}$ be a sequence of functions with $V_h: \Scal \rightarrow [0, c]$.
    Then, for any $(s, a) \in \Scal \times \Acal$ and $h \in [H]$, we have
    \begin{align*}
        \Expec_{y' \sim P_{\Ycal}(s, a)} \VV((\Peff)_{\Zmap(s, a, y')}, V_{h+1}(\cdot, y')) \le P_{s, a}(V_{h+1})^2 - (V_h(s))^2 + 2 c \max\{ V_h(s) - P_{s, a}V_{h+1}, 0\}
        \, .
    \end{align*}
\end{lemma}
\begin{proof}
    By the law of total variance, we have
    \begin{align*}
        \Expec_{y' \sim P_{\Ycal}(s, a)} \VV((\Peff)_{\Zmap(s, a, y')}, V_{h+1}(\cdot, y'))
        & = \VV(P_{s, a}, V_{h+1}) - \VV((P_{\Ycal})_{s, a}, (\Peff)_{\Zmap(s, a, y')} V_{h+1}(\cdot, y'))
        \\
        & \le \VV(P_{s, a}, V_{h+1})
        \, .
    \end{align*}
    Applying \cref{lma:variance decomposition} to $\VV(P_{s, a}, V_{h+1})$ completes the proof.
\end{proof}

\begin{lemma}[Lemma 27 in \citet{lee2025minimax}]
\label{lma:variance decomposition}
    Let $c \ge 0$ be a constant.
    Let $\{V_h\}_{h=1}^{H+1}$ sequences of functions with $V_h: \Scal \rightarrow [0, c]$.
    Then, for any $(s, a) \in \Scal \times \Acal$, we have
    \begin{align*}
        \VV(P_{s, a}, V_{h+1}) \le P_{s, a}(V_{h+1})^2 - (V_h(s))^2 + 2 c \max\{ V_h(s) - P_{s, a}V_{h+1}, 0\}
        \, .
    \end{align*}
\end{lemma}

\section{Proof of Theorem~\ref{thm:regret lower bound}}
\label{appx:lower bound}

In this section, we prove the lower bound result.

\subsection{Hard Structure for Delayed Observation: CodeMDP}

\label{appx:codeMDP}

In this section, we describe a specially designed MDP with delayed observation, whose learning complexity increases with the length of the delay.
This structure serves as a core component of the hard instances for the regret lower bound.

Given $D$, a \emph{CodeMDP} is an MDP with a state space $\Scal_{\mathrm{code}} = ([D] \times\{0, 1\}) \cup \{ s_{\mathrm{succ}}, s_{\mathrm{fail}}\}$ and an action space $\Acal = \{0, 1\}$.
The transition of the CodeMDP is deterministic and is as follows:
\begin{itemize}
    \item For states $(i, b)$ with $i = 2, \ldots, D$ and $b = 0, 1$, any action takes the agent to state $(i - 1, b)$.
    \item For states $(1, b)$ with $b = 0, 1$, taking action $b$ leads to $s_{\mathrm{succ}}$, and taking the other action leads to $s_{\mathrm{fail}}$.
    \item For state $s_{\mathrm{succ}}$, taking any action yields a reward of $1$ and returns to $s_{\mathrm{succ}}$.
    \item For state $s_{\mathrm{fail}}$, taking any action yields a reward of $0$ and returns to $s_{\mathrm{fail}}$.
\end{itemize}

Due to the delay, the agent cannot observe the intermediate states.
The result is only revealed after the episode, or at least after entering either $s_{\mathrm{suc}}$ or $s_{\mathrm{fail}}$, so the agent must plan a sequence of actions $a_1, \ldots, a_D$ in advance.
The mechanism of CodeMDP is as follows.
The agent randomly arrives at an initial state, say $(i, b)$, but agent can not observe the arrived state.
The agent takes the planned sequence of actions $a_1, \ldots, a_D$.
After taking $i-1$ actions, the agent arrives at $(1, b)$, which is unobservable, and then takes the action $a_i$.
If $a_i = b$, then the agent transitions to $s_{\mathrm{succ}}$ and receives the reward, otherwise fails.
\\
The randomness of the CodeMDP that the agent must learn is the initial state distribution.
For instance, suppose the initial state is distributed uniformly over states $\{(i, 0)\}_{i=1}^D$, then the agent receives the reward if and only if $a_i = 0$, so the optimal sequence of actions is $(0, \ldots, 0)$.
In general, suppose the initial state is sampled from a fixed distribution $P$ over $\Scal_{\mathrm{code}}$, where $P(s_{\mathrm{succ}}) = P(s_{\mathrm{fail}}) = 0$.
For the sake of brevity, we momentarily assume that $s_{\mathrm{succ}}$ transitions to $s_{\mathrm{fail}}$ after receiving a reward, so that the total reward is always 0 or 1.
Then, the expected reward, or the $Q$-value of an action sequence $\ab = (a_1, \ldots, a_D)$ under $P$ is $ Q_{P}(\ab)  = \sum_{i=1}^D \sum_{b \in \{0, 1\}} P(i, b) \ind\{ a_i = b \} = \sum_{i=1}^D P(i, a_i)$.
It follows that the optimal actions are $a_i^* = \argmax_{b \in \{0, 1\}} P(i, b)$ for $i = 1, \ldots, D$, and the optimal value function is
\begin{align*}
    V^*(P) := \sum_{i=1}^D \max\{ P(i, 0), P(i, 1)\}
    \, .
\end{align*}
In this sense, $P$ encodes a correct ``codeword'' $(a_1^*, \ldots, a_D^*)$ that yields the maximum reward of $V^*(P)$.

Now, we consider a certain parameterization of $P$.
For $\thetab \in [-1, 1]^D$, let $P(\thetab)$ assign probability $\frac{1 - \theta_i}{2D}$ to state $(i, 0)$ and $\frac{1 + \theta_i}{2D}$ to state $(i, 1)$.
The codeword for this distribution is $( \ind\{\theta_1 \ge 0\}, \ldots, \ind\{\theta_D \ge 0\})$ with the optimal value being
\begin{align*}
    V^*(P(\thetab)) & = \sum_{i=1}^D \max \left\{ \frac{1 + \theta_i}{2D}, \frac{1 - \theta_i}{2D}  \right \} 
    \\
    & = \frac{1}{2} + \frac{1}{2D}\sum_{i=1}^D |\theta_i| 
    \\
    & = \frac{1}{2} + \frac{1}{2D} \| \thetab \|_1
    \, .
\end{align*}

Suppose there is an additional initial state $s_0$ with $A$ actions.
Suppose there are $A$ vectors $\{ \thetab_a \}_{a \in \Acal}$ such that $P_{s_0, a} = P(\thetab_a)$.
Then, the optimal action from $s_0$ is the action $a^*$ with the largest $\| \thetab_a \|_1$, and taking any other action incurs instantaneous regret of at least $\frac{1}{2D}( \| \thetab_{a^*} \|_1 - \| \thetab_a \|_1)$, let alone taking the correct codeword actions of $P(\thetab_a)$.
Then, the agent must find the optimal arm by finding the vector $\thetab_a$ with the largest $\ell_1$-norm.
This problem can be reduced to the following problem formulation, which is also the one described in \cref{prop:l1 estimation}.

\paragraph{$\ell_1$-norm estimation problem.}
Let $\thetab \in [-1, 1]^d$.
For $t = 1, \ldots, n$, an index $I_t \sim \mathrm{Unif}([d])$ is sampled, and then $X_t \sim \mathrm{Ber}(\frac{1 + \theta_{I_t}}{2})$ is sampled.
Based on the observations $\{(I_t, X_t)\}_{t=1}^n$, what is the minimum number of $n$ required to estimate $\frac{1}{d} \| \thetab \|_1$ up to an additive error of $\varepsilon$?

The observation $(I_t, X_t)$ in the $\ell_1$-norm estimation problem corresponds to the state $(I_t, X_t) \in \Scal_{\mathrm{code}}$ being sampled as the first state of the CodeMDP.
Although the agent cannot observe that state immediately, it is observed after the end of the episode, which is why we can identify the observation types of the $\ell_1$-norm estimation and the CodeMDP.
Note that the $\ell_1$-norm estimation problem only models the learning of $V^*(P(\thetab_a))$ and disregards the learning of the optimal action sequence.
We later show that the hardness of learning $V^*(P(\thetab))$ is enough to prove \cref{thm:regret lower bound}.

We restate \cref{prop:l1 estimation}, which states a lower bound for the sample complexity of the $\ell_1$-norm estimation problem.

\begin{proposition}[Formal statement of \cref{prop:l1 estimation}]
\label{prop:l1 estimation 2}
    Suppose $d$ is larger than some absolute constant and $0 < \varepsilon < \frac{1}{48 \log d}$ is given.
    Let $n = \lfloor  \frac{d}{1152 \varepsilon^2 \log d} \rfloor$.
    Suppose $\hat{L}^n$ is an estimator for $\frac{1}{d} \| \thetab \|_1$ based on $n$ observations $\{(I_t, X_t)\}_{t=1}^n$.
    Then, we have
    \begin{align*}
        \inf_{\hat{L}^n} \sup_{\thetab \in [-1, 1]^d} \PP_{\thetab} \left( \left| \hat{L}^n - \frac{1}{d} \| \thetab \|_1 \right| \ge \varepsilon \right) \ge \frac{1}{8}
        \, ,
    \end{align*}
    meaning that every estimator based on only $n$ samples incurs the error of at least $\varepsilon$ for some $\thetab$ with probability at least $\frac{1}{8}$.
\end{proposition}

The proposition is closely related to Theorem 3 in \citet{cai2011testing}.
Their theorem also considers the problem of estimating $\| \thetab\|_1$, but the observation considered is a single $d$-dimensional vector sampled from a Gaussian distribution $\Ncal(\thetab, I_d)$, and the measurement of risk is also different.
The proof of \cref{prop:l1 estimation 2} follows the high-level roadmap of Theorem 3 in \citet{cai2011testing}, but the intermediate steps require a significantly different type of computation due to this difference.
The proof of \cref{prop:l1 estimation 2} is presented in \cref{appx:proof of l1 estimation}.

\subsection{Construction of Hard Instances}
\label{appx:construction of hard instance}

Recall the structure where there is an initial state $s_0$ with $A$ actions that transition to a CodeMDP, whose transition probabilities are parameterized by $\thetab_a$.
Suppose exactly one of the $\|\thetab_a\|_1$ is larger than the others by $\varepsilon$.
\cref{prop:l1 estimation 2} implies that the agent must take each action at least $\Omega( \frac{D}{\varepsilon^2 \log D})$ times to specify the optimal action, and failing to do so will incur $\Omega(\varepsilon K)$ regret.
Tuning $\varepsilon$ leads to a (non-rigorous) lower bound of $\Omega (\sqrt{\frac{D A K}{\log D}}) $.
We inflate this bound by a $\sqrt{S}$ factor by using the tree structure considered in \citet{domingues2021episodic}.
Hence, the final structure of the hard instances will be a combination of the tree structure and the CodeMDP, where the initial state of the whole MDP will be the root of the tree.
We explain the structure of the tree in more detail.

The tree part of the MDP is a $A$-ary tree of states with $\Theta(S)$ leaves.
Specifically, we choose a set of leaves $\Sleaf \subset \Scal$ with $| \Sleaf |$ equal to the power of 2 that satisfies $\frac{S}{8}\le | \Sleaf| < \frac{S}{4}$.
Then, we construct the tree bottom-up, adding parent states, one for at most $A$ children.
It can be shown that the height of the tree is at most $H_{\mathrm{tree}} \le 1 + \lceil \log_A \frac{S}{4} \rceil$ and the total number of states in the tree is at most $\frac{S}{2}$.
For a state in the tree that is not a leaf, each action corresponds to one of its children, where the action transitions to that child.
Then, the path from the root to a leaf corresponds to a sequence of $H_{\mathrm{tree}} - 1$ actions that takes the agent deterministically from the root to the leaf state.
\\
Then, we add a CodeMDP structure with $D = \tilde{D} := \min\{ \Dmax, \frac{H}{4}, \frac{B}{2}, \frac{S}{4} -1\}$.
For each leaf-action pair $(l, a) \in \Sleaf, \Acal$, we assign a vector $\thetab_{l, a} \in [-1, 1]^d$ so that $P_{l, a} = P(\thetab_{l, a})$, which is a distribution over $\Scal_{\mathrm{code}}$ defined in \cref{appx:codeMDP}.
One of the goals of the agent becomes to find the pair that has the largest $\frac{1}{\tilde{D}}\| \thetab_{l, a} \|_1$ based on the observation.
\cref{fig:hard instance} illustrates the structure of the whole MDP.

We assume $H \ge 4 H_{\mathrm{tree}} \ge 8 + 4 \log_A S$, so that taking $H_{\mathrm{tree}}$ actions to reach a leaf state and then taking $\tilde{D} \le \frac{H}{4}$ many actions for the CodeMDP takes at most $\frac{H}{2}$ time steps.
We note that the total number of states required to construct this structure is at most $\frac{S}{2} + 2 \tilde{D} + 2 \le S$, and the branching factor is $2\tilde{D} \le B$.

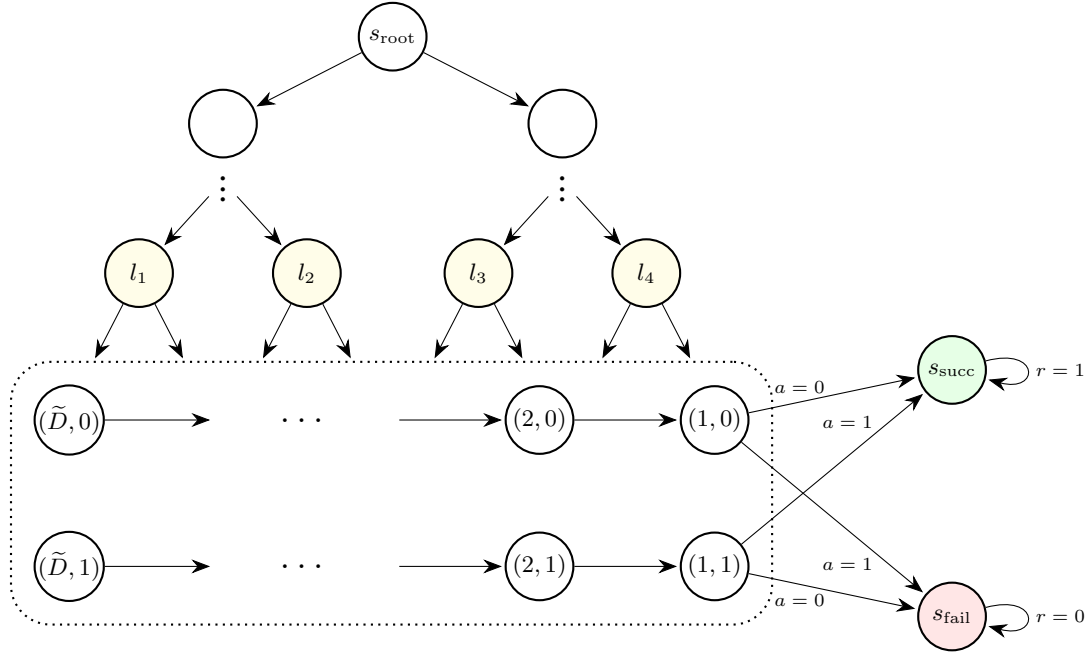
\begin{figure}[tb]
\centering

\begin{tikzpicture}[
    >={Stealth[length=2.5mm]},
    shorten >=1pt,
    state/.style={
        circle,
        draw=black,
        thick,
        minimum size=0.9cm,
        inner sep=0pt,
        font=\small
    },
    leaf/.style={
        state,
        fill=yellow!10
    },
    missing/.style={
        draw=none, 
        font=\Large
    }
]

    \node[state] (root) at (4, 4.5) {$s_{\mathrm{root}}$}; 

    \node[state] (n1) [below left=0.5cm and 1.6cm of root] {};
    \node[state] (n2) [below right=0.5cm and 1.6cm of root] {};

    \node[missing] (dots1) [below=-0.1cm of n1] {$\vdots$};
    \node[missing] (dots2) [below=-0.1cm of n2] {$\vdots$};

    \node[leaf] (leaf1) [below left=0.5cm and 0.6cm of dots1] {$l_1$};
    \node[leaf] (leaf2) [below right=0.5cm and 0.6cm of dots1] {$l_2$};
    \node[leaf] (leaf3) [below left=0.5cm and 0.6cm of dots2] {$l_3$};
    \node[leaf] (leaf4) [below right=0.5cm and 0.6cm of dots2] {$l_4$};

    \draw[->] (root) -- (n1);
    \draw[->] (root) -- (n2);
    \draw[->] (dots1) -- (leaf1);
    \draw[->] (dots1) -- (leaf2);
    \draw[->] (dots2) -- (leaf3);
    \draw[->] (dots2) -- (leaf4);

    \node[state] (c1_start) at (-0.3, -0.6) {$(\tilde{D}, 0)$};
    \node[missing] (c1_dots) [right=2.2cm of c1_start] {$\dots$};
    \node[state] (c1_mid)  [right=2.2cm of c1_dots] {$(2, 0)$};
    \node[state] (c1_end)  [right=1.4cm of c1_mid] {$(1, 0)$};

    \node[state] (c2_start) [below=1cm of c1_start] {$(\tilde{D}, 1)$};
    \node[missing] (c2_dots) [right=2.2cm of c2_start] {$\dots$};
    \node[state] (c2_mid)  [right=2.2cm of c2_dots] {$(2, 1)$};
    \node[state] (c2_end)  [right=1.4cm of c2_mid] {$(1, 1)$};

    \draw[->, shorten >=8mm] (c1_start) -- (c1_dots);
    \draw[->, shorten <=8mm] (c1_dots) -- (c1_mid);
    \draw[->] (c1_mid) -- (c1_end);

    \draw[->, shorten >=8mm] (c2_start) -- (c2_dots);
    \draw[->, shorten <=8mm] (c2_dots) -- (c2_mid);
    \draw[->] (c2_mid) -- (c2_end);

    \node[draw, dotted, thick, inner sep=0.3cm, rounded corners=15pt, fit=(c1_start) (c1_end) (c2_start) (c2_end)] (chainbox) {};

    \node[state, fill=green!10] (succ) [above right=0.01cm and 2.5cm of c1_end] {$s_{\mathrm{succ}}$};
    \node[state, fill=red!10]   (fail) [below right=0.01cm and 2.5cm of c2_end] {$s_{\mathrm{fail}}$};

    \draw[->] (succ) edge [loop right] node[font=\scriptsize] {$r=1$} (succ);
    \draw[->] (fail) edge [loop right] node[font=\scriptsize] {$r=0$} (fail);

    \draw[->] (c1_end) -- node[above, font=\scriptsize, pos=0.3] {$a=0$} (succ);
    \draw[->] (c1_end) -- node[below, font=\scriptsize, pos=0.7, xshift=-0.3cm] {$a=1$} (fail);
    \draw[->] (c2_end) -- node[below, font=\scriptsize, pos=0.3] {$a=0$} (fail);
    \draw[->] (c2_end) -- node[above, font=\scriptsize, pos=0.7, xshift=-0.3cm] {$a=1$} (succ);

    \foreach \leaf in {leaf1, leaf2, leaf3, leaf4} {
        \draw[->, black] (\leaf) -- ($( \leaf |- chainbox.north) + (-0.6, 0)$);
        \draw[->, black] (\leaf) -- ($( \leaf |- chainbox.north) + (0.6, 0)$);
    }

\end{tikzpicture}
\caption{Illustration of hard instances for \cref{thm:regret lower bound}. The structure consists of a tree structure (top) and a CodeMDP (bottom).
The leaf states are labeled $l_1, l_2, l_3, l_4$.
Each leaf state and action pair has its own probability distribution over the states in the CodeMDP.
Once the agent takes an action from the leaf state, it must make $\tilde{D}$ actions without observing which state it has landed in.
The agent enters the success state $s_{\mathrm{succ}}$ if it landed at state $(i, b)$ and the $i$-th out of $\tilde{D}$ actions is $b$, and it receives the reward.
The agent enters the fail state $s_{\mathrm{fail}}$ if it landed at state $(i, b)$ and the $i$-th out of $\tilde{D}$ actions is not $b$, and it cannot receive any reward.
}
\label{fig:hard instance}
\end{figure}

\subsection{Proof of Theorem~\ref{thm:regret lower bound}}

In this section, we prove \cref{thm:regret lower bound} using the tools we developed.

We assumed that $s_{\mathrm{succ}}$ transitions to $s_{\mathrm{fail}}$ in \cref{appx:codeMDP} for brevity.
We revert the assumption and assume that $s_{\mathrm{succ}}$ returns to itself.
By doing so, the total reward the agent receives per episode scales with $H$.
However, we must additionally address that the total reward may vary depending on \emph{when} the agent enters the success state.
For instance, if the initial state of the CodeMDP is $(1, b)$, then the agent enters $s_{\mathrm{succ}}$ after one action, whereas starting from the state $(\tilde{D}, b)$ takes $\tilde{D}$ time steps to enter $s_{\mathrm{succ}}$.
Let $\tilde{H} = H - H_{\mathrm{tree}}$.
Then, if the agent enters the state $(i, b)$ in the CodeMDP through one of the leaf states, the agent receives the reward of $\tilde{H} - i$, assuming it successfully enters $s_{\mathrm{succ}}$.
Then, the optimal value function of a given distribution $P$ over $S_{\mathrm{code}}$ changes to
\begin{align*}
    V^*(P) := \sum_{i=1}^{\tilde{D}} (\tilde{H} - i) \max\{ P(i, 0), P(i, 1) \}
    \, .
\end{align*}
Considering $P(\thetab)$ parameterized by $\thetab \in [-1, 1]^{\tilde{D}}$, we have
\begin{align*}
    V^*(P(\thetab)) 
    & = \sum_{i=1}^{\tilde{D}} (\tilde{H} - i) \frac{1 + | \theta_i |}{2 \tilde{D}}
    \\
    & = \frac{1}{2} \sum_{i=1}^{\tilde{D}} (\tilde{H} - i ) + \frac{1}{2\tilde{D}} \sum_{i=1}^{\tilde{D}} (\tilde{H} - i) | \theta_i |
    \, .
\end{align*}
We denote the part that depends on $\thetab$ by $\tilde{V}^*(\thetab)$, defined as
\begin{align*}
    \tilde{V}^*(\thetab) := \frac{1}{2\tilde{D}}\sum_{i=1}^{\tilde{D}} (\tilde{H} - i) | \theta_i |
    \, .
\end{align*}

Although the optimal value function is no longer a function of $\| \thetab\|_1$, we note that this modification does not change the observation type, so the proof of \cref{prop:l1 estimation 2} is still valid, where a slightly modified concentration result for this weighted norm is required.
By translating \cref{prop:l1 estimation 2} to the CodeMDP setting, we obtain the following lemma.

\begin{lemma}
\label{lma:lower bound existence}
    Suppose $\varepsilon \in (0, \frac{1}{2}]$, $\tilde{D}$, and $\tilde{H}$ are given.
    Assume that $\tilde{D}$ is larger than some absolute constant, $\tilde{H} \ge \frac{H}{2}$, and $\tilde{D} \le \tilde{H} -1$.
    Then, there exist two probability distributions $\Pcal_1$ and $\Pcal_2$ over $\thetab$ that satisfy the following:
    \begin{enumerate}
        \item For all $n \le \frac{\tilde{D} \log \tilde{D}}{2 \varepsilon^2}$, we have $\KL{\Pcal_1^n}{\Pcal_2^n} \le \frac{1}{8}$, where $\Pcal_1^n, \Pcal_2^n$ denotes the distribution over $n$ samples of $\Pcal_1$ and $\Pcal_2$, respectively.
        \item There exists a threshold value $\alpha $ such that $\PP_{\thetab \sim \Pcal_1}( \tilde{V}^*(\thetab) \ge \alpha) \le \frac{1}{4}$ and $\PP_{\thetab \sim \Pcal_2}( \tilde{V}^*(\thetab) \le \alpha + \frac{\varepsilon H}{192 \log \tilde{D}} ) \le \frac{1}{4}$.
    \end{enumerate}
\end{lemma}
The proof of \cref{lma:lower bound existence} uses the same technical tools as the proof of \cref{prop:l1 estimation 2}.
The full proof is deferred to \cref{appx:proof of lower bound existence}

As we work with parameterized MDPs with randomized parameters, we need a lemma that bounds the KL-divergence of two mixture distributions over the trajectories induced by random MDPs.
We define the sample space $\Omega$ to be the space of trajectories $\{ (s_1^k, a_1^k, \ldots, s_H^k, a_H^k, s_{H+1}^k )\}_k$.
Assuming a fixed algorithm is given, we denote the probability distribution over $\Omega$ induced by an MDP $\Mcal$ over $K$ episodes by $\PP_{\Mcal}^K$.
We also define $\PP_{\Mcal}^{\tau}$ for a stopping time $\tau$ in the same way.

\begin{lemma}
\label{lma:kl of MDPs}
    Let $\Mcal(\thetab)$ be an MDP parameterized by $\thetab \in \Theta$, where $\thetab$ only affects the transition distribution of a single state-action pair $(s_0, a_0) \in \Scal \times \Acal$.
    Let $\Pcal_1$ and $\Pcal_2$ be two distributions over $\Theta$ with $\Pcal_1 \ll \Pcal_2$.
    Suppose $n$ and $K$ are given positive integers.
    Let $\tau \in [K+1]$ be the index of the episode $k$ at which $N^k(s_0, a_0)$ reaches $n$, where $\tau = K+1$ if there is no such episode.
    For simplicity, assume that $s_0$ can be visited at most once per episode for all $\thetab$.
    Assuming the algorithm is fixed, define $\PP_{1}^{\tau}$ as the mixture distribution of $\PP_{\Mcal(\thetab)}^{\tau}$ with $\thetab \sim \Pcal_1$, and define $\PP_{2}^{\tau}$ in the same way with $\thetab \sim \Pcal_2$.
    Then, we have
    \begin{align*}
        \KL{\PP_{1}^\tau}{\PP_{2}^{\tau}} \le \KL{\Pcal_1^n}{\Pcal_2^n}
        \, ,
    \end{align*}
    where $\Pcal_i^n$ is the distribution over $n$ samples of $P_{s_0, a_0}(\thetab)$ with $\thetab \sim \Pcal_i$ for $i = 1, 2$.
\end{lemma}

The proof of \cref{lma:kl of MDPs} is deferred to \cref{appx:proof of kl of MDPs}.

Now, we prove \cref{thm:regret lower bound}.

\begin{proof}[Proof of \cref{thm:regret lower bound}]
    Create a combination of a tree-structured MDP and a CodeMDP as described in \cref{appx:construction of hard instance}.
    Fix $\varepsilon \in (0, \frac{1}{2}]$, whose value is assigned later.
    Let $n = \frac{\tilde{D} \log \tilde{D}}{2 \varepsilon^2}$.

    For the distribution $\Pcal_1$ defined in \cref{lma:lower bound existence}, we define $\Pcal_1'$ to be its conditional distribution, conditioned on the event $V^*(\thetab) < \alpha$, which is well-defined by the same lemma.
    We set the transition distributions of the leaf-action pairs as $P(l, a) = P(\thetab_{l, a})$ with $\thetab_{l, a} \iidsim \Pcal_1'$ for all $(l, a) \in \Sleaf \times \Acal$.
    Let $\PP_0$ be the mixture distribution over $K$ trajectories induced by the sampling of an MDP in the described way and then interacting with the given algorithm for $K$ episodes.
    We denote $\EE_0$ as the corresponding expectation.
    \\
    As we have $\EE_0 [ \sum_{(l, a) \in \Sleaf \times \Acal} N^{K+1}(l, a)] = K$, there exists a leaf-action pair $(l_0, a_0)$ such that $\EE_1 [ N^{K+1}(l_0, a_0) ] \le \frac{K}{LA}$, where $L := | \Sleaf |$.
    By the Markov inequality, we have $\PP_0( N^{K+1}(l_0, a_0)  \ge n ) \le \frac{K}{LAn}$.
    \\
    Let $\PP_1$ and $\PP_2$ be mixture distributions that are similar to $\PP_0$ except for one difference:
    the distribution $\PP_1$ samples $P(l_0, a_0)$ from $\Pcal_1$, and $\PP_2$ samples $P(l_0, a_0)$ from $\Pcal_2$.
    In addition, we let $\tau$ be the index of the episode where $N^{\tau}(l_0, a_0) $ reaches $n$ for the first time, where $\tau = K+1$ if there is no such episode.
    Defining $\PP_1^{\tau}$ and $\PP_2^{\tau}$ as distributions over the first $\tau$ episodes under $\PP_1$ and $\PP_2$ respectively, the KL-divergence between the two distributions is bounded as
    \begin{align*}
        \KL{\PP_1^{\tau}}{\PP_2^{\tau}}
        & \le \Expec_{\substack{\thetab_{l, a} \sim \Pcal_1'\\\forall(l, a) \in \Sleaf \times \Acal \setminus\{(l_0, a_0)\} }} \left[ \KL{ \PP_{1}^{\tau} (\cdot \mid \{\thetab_{l, a} \}_{l, a}) }{\PP_{2}^{\tau} ( \cdot \mid \{\thetab_{l, a} \}_{l, a}) } \right]
        \\
        & \le \Expec_{\substack{\thetab_{l, a} \sim \Pcal_1'\\\forall(l, a) \in \Sleaf \times \Acal \setminus\{(l_0, a_0)\} }} \left[ \KL{ \Pcal_1^n }{\Pcal_2^n } \right]
        \\
        & \le \frac{1}{8}
        \, ,
    \end{align*}
    where the first inequality is due to the convexity of KL-divergence, the second inequality comes from \cref{lma:kl of MDPs}, and the last inequality applies \cref{lma:lower bound existence}.
    Then, by Pinsker's inequality (\cref{lma:Pinsker}), we derive that
    \begin{align*}
        \PP_1^{\tau}(N^{\tau}(l_0, a_0) \ge n) + \PP_2^{\tau}(N^{\tau}(l_0, a_0) < n) \ge 1 - \sqrt{\frac{1}{2} \KL{\PP_1^\tau}{\PP_2^\tau}} \ge \frac{3}{4}
        \, .
    \end{align*}
    We note that $\PP_1^{\tau}(N^{\tau}(l_0, a_0) \ge n) = \PP_1(N^{K+1}(l_0, a_0) \ge n)$ holds by the stopping rule for $\tau$.
    $\PP_2^{\tau}(N^{\tau}(l_0, a_0) < n) = \PP_2(N^{K+1}(l_0, a_0) < n)$ holds by the same reason.
    Using that $\PP(A\cap B) \ge \PP(A) - \PP(B^{\mathsf{c}})$ for any events $A$ and $B$, we obtain that
    \begin{align*}
        & \PP_1 \left(N^{K+1}(l_0, a_0) \ge n, \tilde{V}^*(\thetab_{l_0, a_0}) < \alpha \right)
        \nonumber
        \\
        & \qquad + \PP_2 \left(N^{K+1}(l_0, a_0) < n, \tilde{V}^*(\thetab_{l_0, a_0}) > \alpha + \frac{\varepsilon H}{192 \log \tilde{D}} \right)
        \nonumber
        \\
        & \ge \PP_1 \left(N^{K+1}(l_0, a_0) \ge n\right) - \PP_1 \left(\tilde{V}^*(\thetab_{l_0, a_0}) \ge \alpha \right)
        \nonumber
        \\
        & \qquad + \PP_2 \left(N^{K+1}(l_0, a_0) < n\right) - \PP_2\left( \tilde{V}^*(\thetab_{l_0, a_0}) \le \alpha + \frac{\varepsilon H}{192 \log \tilde{D}} \right)
        \nonumber
        \\
        & \ge \PP_1 \left(N^{K+1}(l_0, a_0) \ge n \right)
        + \PP_2 \left(N^{K+1}(l_0, a_0) < n \right) - \frac{1}{2}
        \nonumber
        \\
        & \ge \frac{1}{4}
        \, ,
    \end{align*}
    where we apply \cref{lma:lower bound existence} for the second inequality.
    Rearranging the inequality, we have
    \begin{align}
        \PP_2 \left(N^{K+1}(l_0, a_0) < n, \tilde{V}^*(\thetab_{l_0, a_0}) > \alpha + \frac{\varepsilon H}{192 \log \tilde{D}} \right) 
        &\ge \frac{1}{4} - \PP_1 \left(N^{K+1}(l_0, a_0) \ge n, \tilde{V}^*(\thetab_{l_0, a_0}) < \alpha \right)
        \, .
        \label{eq:probability lower bound 1}
    \end{align}
    We modify the probability on the right-hand side as follows:
    \begin{align*}
        & \PP_1 \left(N^{K+1}(l_0, a_0) \ge n, \tilde{V}^*(\thetab_{l_0, a_0}) < \alpha \right)
        \\
        & = \PP_1 \left( N^{K+1}(l_0, a_0) \ge n \middle |  \tilde{V}^*(\thetab_{l_0, a_0}) < \alpha \right) 
        \PP_1 \left(\tilde{V}^*(\thetab_{l_0, a_0}) < \alpha  \right)
        \\
        & \le \frac{3}{4} \PP_1 \left( N^{K+1}(l_0, a_0) \ge n \middle |  \tilde{V}^*(\thetab_{l_0, a_0}) < \alpha \right) 
        \\
        & = \frac{3}{4} \PP_0 \left( N^{K+1}(l_0, a_0) \ge n \right) 
        \\
        & \le \frac{3 K}{4 LAn}
        \, ,
    \end{align*}
    where we use \cref{lma:lower bound existence} for the first inequality, and the following equality is due to that $\PP_0$ is the conditional distribution of $\PP_1$ conditioned on the event $\tilde{V}^*(\thetab_{l_0, a_0}) < \alpha$.
    Now, we choose $\varepsilon = \sqrt{\frac{ \tilde{D} SA \log \tilde{D}}{96 K}}$ so that $n = \frac{\tilde{D} \log \tilde{D}}{2 \varepsilon^2} = \frac{48 K}{SA} \ge \frac{6 K}{ LA}$, where we use that $L \ge \frac{S}{8}$ by construction.
    We note that $\varepsilon < \frac{1}{2}$ is guaranteed by $K \ge \frac{1}{24} \tilde{D} SA \log \tilde{D}$.
    With this choice of $\varepsilon$, we have $\frac{3 K}{4LAn} \le \frac{1}{8}$.
    Then, from inequality~\eqref{eq:probability lower bound 1}, we obtain that
    \begin{align*}
        \PP_2 \left(N^{K+1}(l_0, a_0) < n, \tilde{V}^*(\thetab_{l_0, a_0}) > \alpha + \frac{\varepsilon H}{192 \log \tilde{D}} \right)
        & \ge \frac{1}{8}
        \, .
    \end{align*}
    It implies that there exists an MDP instance $\Mcal^*$ such that $\tilde{V}^*(\thetab_{l_0, a_0}) > \alpha + \frac{\varepsilon H}{192 \log \tilde{D}}$ holds for $(l_0, a_0)$, and $\tilde{V}^*(\thetab_{l, a}) < \alpha $ holds for all $(l, a) \in \Sleaf \times \Acal \setminus \{l_0, a_0\}$, while $\PP_{\Mcal^*}( N^{K+1}(l_0, a_0) < n ) \ge \frac{1}{8}$ simultaneously holds.
    For this MDP, taking any leaf-action pair $(l, a)$ other than $(l_0, a_0)$ yields an instantaneous regret of at least $\frac{\varepsilon H}{192 \log \tilde{D}}$.
    Then, the cumulative regret is bounded below by $\frac{\varepsilon H}{192 \log \tilde{D}} \Expec_{\Mcal^*}[ K - N^{K+1}(l_0, a_0)]$.
    We lower bound the expected number of sub-optimal leaf-action pair selections as follows:
    \begin{align*}
        \Expec_{\Mcal^*}[ K - N^{K+1}(l_0, a_0)]
        & \ge \frac{K}{2} \PP_{\Mcal^*}\left( K - N^{K+1}(l_0, a_0) \ge \frac{K}{2}\right)
        \\
        & = \frac{K}{2} \PP_{\Mcal^*}\left(\frac{K}{2} \ge N^{K+1}(l_0, a_0) \right)
        \\
        & \ge \frac{K}{2} \PP_{\Mcal^*}\left( n \ge N^{K+1}(l_0, a_0) \right)
        \\
        & \ge \frac{K}{16}
        \, ,
    \end{align*}
    where the first inequality uses Markov's inequality, and the third line uses that $n = \frac{K}{4SA} < \frac{K}{2}$.
    Therefore, the cumulative regret for $\Mcal^*$ is lower bounded by
    \begin{align*}
        \frac{\varepsilon H}{192 \log \tilde{D}} \Expec_{\Mcal^*}[ K - N^{K+1}(l_0, a_0)]
        & \ge \frac{\varepsilon H K }{3072 \log \tilde{D}}
        \\
        & = \frac{H}{3072} \sqrt{\frac{ \tilde{D} SAK}{96 \log \tilde{D}}}
        \, .
    \end{align*}
    The proof is complete.
\end{proof}

\section{Proof of Proposition~\ref{prop:l1 estimation}}
\label{appx:proof of l1 estimation}

In this section, we focus on the $\ell_1$-estimation problem introduced in \cref{appx:codeMDP}.
\cref{prop:l1 estimation 2}, which is the formal version of \cref{prop:l1 estimation}, states that the lower bound for the sample complexity is $\Omega( \frac{d}{\varepsilon^2 \log d})$.
The main ideas presented in this section to prove \cref{prop:l1 estimation 2} also apply to the proof of \cref{lma:lower bound existence}, which is a core lemma in proving \cref{thm:regret lower bound}.

The proof of \cref{prop:l1 estimation 2} is based on the method of two fuzzy hypotheses~\citep{tsybakov2008introduction}.
The two distributions over $\thetab$ are based on the distributions introduced in \citet{cai2011testing}.

\begin{lemma}[Lemma 1 in \citet{cai2011testing}]
\label{lma:absolute approximating measure}
    For given even integer $k > 0$, there exist two probability measures $\nu_1$ and $\nu_2$ on $[-1, 1]$ that satisfy the following conditions:
    \begin{itemize}
        \item $\nu_1$ and $\nu_2$ are symmetric around 0.
        \item $\int t^l \, d\nu_1 = \int t^l d\nu_2$ for $l = 0, \ldots, k$.
        \item $\int | t | \, d \nu_2 - \int | t | \, d \nu_1 = 2 \delta_k$, where $\delta_k := \inf_{p \in \Pcal_k}\sup_{x \in [-1, 1]} | |x| - p(x) | \ge \frac{1}{4k}$.
    \end{itemize}
\end{lemma}

We denote $\mu_1 := \int|t |\, d\nu_1$ and $\mu_2 := \int | t | \, d\nu_2$.
By \cref{lma:absolute approximating measure}, we have $\mu_2 - \mu_1 \ge \frac{1}{2k}$.
Additionally, for a real number $\varepsilon > 0$, we denote the distribution of $\varepsilon Z_i$ with $Z_i \sim \rho$ by $\varepsilon \rho$ for simplicity.

With this distribution, we can prove the following KL-divergence bound on the two fuzzy hypotheses.

\begin{proposition}
\label{prop:KL bound}
    For a given even integer $k > 0$, let $\nu_1$ and $\nu_2$ be defined as in \cref{lma:absolute approximating measure}.
    For a given $\varepsilon \in (0, \frac{1}{2}]$, suppose a prior distribution $\Pcal_1$ on $\{ \theta_i\}_{i=1}^d$ is given as $\theta_i  \stackrel{\iid}{\sim} \varepsilon \nu_1$, and $\Pcal_2$ in the same way with $\nu_2$.
    Let $\Pcal_1^n$ and $\Pcal_2^n$ be the probability measures on $n$ observed samples of $\{(I_t, X_t)\}_{t=1}^n$ with given priors.
    Then,
    \begin{align*}
        \KL{\Pcal_1^n}{\Pcal_2^n} \le d \log \left( 1  + \left( \frac{4 n \varepsilon^2}{d (k+1)} \right)^{k+1} \left( 1 + \frac{2 \varepsilon^2}{d} \right)^n \right)
        \, .
    \end{align*}
    In particular, when $n \le \frac{d \log d}{2 \varepsilon^2}$ and $k \ge 4 \log d +2$, one has $\KL{\Pcal_1^n}{\Pcal_2^n} \le \frac{1}{8}$.
\end{proposition}

Proof of \cref{prop:KL bound} is deferred to \cref{sec:proof of KL bound}.

\begin{proof}[Proof of \cref{prop:l1 estimation 2}]
    We assume that $d$ is larger than some absolute constant which we specify later, and $\varepsilon \le \frac{1}{48 \log d}$.
    Let $n = \frac{d}{1152 \varepsilon^2 \log d}$.
    We show that for any estimator $\hat{L}^n$, there exists an instance $\thetab$ that satisfies $\PP_{\thetab}(| \hat{L}^n - \frac{1}{d}\|\thetab\|_1 | \ge \varepsilon ) \ge \frac{1}{8}$.
    \\
    Let $\PP_1 := \Pcal_1^n$ and $\PP_2 = \Pcal_2^n$ be probability distributions defined in \cref{prop:KL bound} with $k$, $\varepsilon'$, and $n$.
    We plug in $k = 6 \log d \ge 4 \log d + 2$, $\varepsilon' = 4 k \varepsilon$, where we assume that $d \ge 3$.
    Note that $n = \frac{d}{1152 \varepsilon^2\log d} \le \frac{d \log d}{2 (\varepsilon')^2}$, so by \cref{prop:KL bound} we have we have $\KL{\PP_1}{\PP_2} \le \frac{1}{8}$.
    \\
    Let $\psi = \ind\{ \hat{L}^n \le \frac{\varepsilon' ( \mu_1 + \mu_2)}{2 } \}$, where $\mu_1$ and $\mu_2$ are defined after \cref{lma:absolute approximating measure}.
    By Pinsker's inequality (\cref{lma:Pinsker}), we have 
    \begin{align}
        \PP_1(\psi = 0) + \PP_2(\psi = 1) \ge 1 - \sqrt{\frac{1}{2}\KL{\PP_1}{\PP_2}} \ge \frac{3}{4}
        \, .
        \label{eq:Pinsker 1}
    \end{align}
    Using Hoeffding's inequality (\cref{lma:Hoeffding}), we have
    \begin{align*}
        \PP_1\left( \frac{1}{d} \| \thetab \|_1 > \varepsilon' \left( \mu_1 + \frac{1}{48 \log d} \right) \right) 
        & = \Pcal_1 \left( \frac{1}{d} \sum_{i=1}^d ( | \theta_i | - \varepsilon' \mu_1) > \frac{\varepsilon'}{48 \log d} \right)
        \\
        & \le \exp \left( - \frac{d}{1152 (\log d)^2} \right)
        \, .
    \end{align*}
    We assume that the value of $d$ is larger than some absolute constant, so that $\PP_1 \left( \frac{1}{d} \| \thetab \|_1 > \varepsilon' \left( \mu_1 + \frac{1}{48 \log d} \right) \right) \le \frac{1}{4}$.
    We note that \cref{lma:absolute approximating measure} implies  that $\mu_2 - \mu_1 \ge \frac{1}{2k} \ge \frac{1}{12 \log d}$, so
    \begin{align*}
        \varepsilon' \left( \mu_1 + \frac{1}{48\log d}  \right)
        & \le \varepsilon' \left( \mu_1 + \frac{\mu_2 - \mu_1 }{2} - \frac{1}{48 \log d} \right)
        \\
        & = \frac{\varepsilon'(\mu_1 + \mu_2)}{2} - \varepsilon
        \, ,
    \end{align*}
    which implies that $\PP_1( \frac{1}{d}\|\thetab\|_1 > \frac{\varepsilon'(\mu_1 + \mu_2)}{2} - \varepsilon) \le \frac{1}{4}$.
    In the same way, we obtain that $\PP_2( \frac{1}{d}\|\thetab\|_1 < \frac{\varepsilon'(\mu_1 + \mu_2)}{2} + \varepsilon) \le \frac{1}{4}$.
    \\
    Using these bounds and inequality~\eqref{eq:Pinsker 1}, we have
    \begin{align}
        & \PP_1 \left(\psi = 0, \frac{1}{d} \| \thetab \|_1 \le \frac{\varepsilon'(\mu_1 + \mu_2)}{2} - \varepsilon \right) + \PP_2 \left( \psi = 1, \frac{1}{d} \| \thetab \|_1 \ge  \frac{\varepsilon'(\mu_1 + \mu_2)}{2} + \varepsilon \right) 
        \nonumber
        \\
        & \ge \PP_1(\psi = 0)  - \PP_1 \left( \frac{1}{d}\|\thetab\|_1 > \frac{\varepsilon'(\mu_1 + \mu_2)}{2} - \varepsilon \right)
        \nonumber
        \\
        & \qquad + \PP_2( \psi = 1) -  \PP_2 \left( \frac{1}{d}\|\thetab\|_1 < \frac{\varepsilon'(\mu_1 + \mu_2)}{2} + \varepsilon \right) 
        \nonumber
        \\
        & \ge \frac{1}{4}
        \label{eq:Pinsker 2}
        \, .
    \end{align}
    Note that the events in the first line are subsets of the event that the estimator $\hat{L}^n$ is wrong by a difference of at least $\varepsilon$.
    Therefore, inequality~\eqref{eq:Pinsker 2} shows that there must exists an instance $\thetab$ where $\PP_{\thetab}(| \hat{L}^n - \frac{1}{d} \|\thetab\|_1 | \ge \varepsilon) \ge \frac{1}{8}$.
\end{proof}

\subsection{Proof of Proposition~\ref{prop:KL bound}}
\label{sec:proof of KL bound}
\begin{proof}[Proof of \cref{prop:KL bound}]
    Let $\Omega$ be the sample space of all possible sequence of $\{ ( I_t, X_t)\}_{t=1}^n$.
    Let $\omega = \{I_t, X_t\}_{t=1}^n$.
    Then, we have that
    \begin{align*}
        \Pcal_1^n(\omega) & = \Expec_{\theta_i \stackrel{\iid}{\sim} \varepsilon \nu_1} \left[ \prod_{t=1}^n \left( \frac{1 + (-1)^{X_t + 1} \theta_{I_t}}{2d } \right) \right]
        \\
        & = \left( \frac{1}{2d} \right)^n \Expec_{\theta_i \stackrel{\iid}{\sim} \varepsilon \nu_1} \left[ \prod_{i=1}^d \left( 1 - \theta_i \right)^{N(i, 0)} \left( 1 + \theta_i \right)^{N(i, 1)} \right]
        \, ,
    \end{align*}
    where we define $N(i, 0) := N(i, 0, \omega) = \sum_{t=1}^n \ind\{ I_t = i, X_t = 0\}$ and $N(i, 1)$ in the same way for $X_t = 1$.
    For $x \in \RR$ and $a, b \in \NN\cup \{0 \}$, define $f(x, a, b): = (1 - x)^a (1 + x)^b$.
    Then, the probability law is expressed as
    \begin{align*}
        \Pcal_1^n(\omega) & = \left(\frac{1}{2d} \right)^n \prod_{i=1}^d \Expec_{\theta_i \stackrel{\iid}{\sim} \varepsilon \nu_1} \left[ f(\theta_i, N(i, 0), N(i, 1) )\right]
        \, ,
    \end{align*}
    where we use the fact that $\theta_i$ are mutually independent of one another.
    The expression holds for $\Pcal_2^n$ with $\nu_2$ instead of $\nu_1$.
    The KL divergence between $\Pcal_1^n$ and $\Pcal_2^n$ becomes
    \begin{align*}
        \KL{\Pcal_1^n}{\Pcal_2^n}
        & = \EE_{\Pcal_1^n} \left[ \log \frac{d\Pcal_1^n}{d\Pcal_2^n} \right]
        \\
        & = \EE_{\Pcal_1^n} \left[ \log \prod_{i=1}^d \frac{\Expec_{\theta_i' \stackrel{\iid}{\sim} \varepsilon \nu_1} [ f(\theta_i', N(i, 0), N(i, 1)) }{\Expec_{\theta_i' \stackrel{\iid}{\sim} \varepsilon \nu_2} [ f(\theta_i', N(i, 0), N(i, 1) )}\right]
        \\
        & = \EE_{\Pcal_1^n} \left[ \sum_{i=1}^d \log \frac{\Expec_{\theta_i' \stackrel{\iid}{\sim} \varepsilon \nu_1} [ f(\theta_i', N(i, 0), N(i, 1) )}{\Expec_{\theta_i' \stackrel{\iid}{\sim} \varepsilon \nu_2} [ f(\theta_i', N(i, 0), N(i, 1)) } \right]
        \\
        & = \EE_{\Pcal_1^n} \left[ \sum_{i=1}^d \log \frac{\Expec_{\theta' \sim \varepsilon \nu_1} [ f(\theta', N(i, 0), N(i, 1) )}{\Expec_{\theta' \sim \varepsilon \nu_2} [ f(\theta', N(i, 0), N(i, 1)) } \right]
        \, .
    \end{align*}
    Note that $(N(i, 0), N(i, 1))$ have the same distribution across indices $i = 1, \ldots, d$.
    We define random variables $N_0$ and $N_1$ that have the same distribution as $N(i, 0)$ and $N(i, 1)$, respectively.
    Under $\Pcal_1^n$, the distributions of $N_0$ and $N_1$ follow $B(n, \frac{1 - \theta}{2d})$ and $B(n, \frac{1 + \theta}{2d})$, respectively, where $\theta \sim \varepsilon \nu_1$.    
    Then, we can express the likelihood ratio for each index in the same form and obtain
    \begin{align*}
        \KL{\Pcal_1^n}{\Pcal_2^n}
        & = d \EE_{\Pcal_1^n} \left[ \log \frac{\Expec_{\theta' \sim \varepsilon \nu_1}[ f(\theta', N_0, N_1) \mid N_0, N_1]}{\Expec_{\theta' \sim \varepsilon \nu_2} [ f(\theta', N_0, N_1) \mid N_0, N_1 ]} \right]
        \, .
    \end{align*}
    This can be understood as that the KL divergence between $\Pcal_1^n$ and $\Pcal_2^n$ is $d$ times the KL divergence between the probability measures restricted on one index.
    Using that $\mathrm{KL}(P || Q) \le \log (1 + \chi^2(P; Q))$, where $\chi^2(P; Q)$ is the chi-squared divergence, we have
    \begin{align*}
         & \Expec_{(N_0, N_1) \sim \Pcal_1^n} \left[ \log \frac{\Expec_{\theta' \sim \varepsilon \nu_1}[ f(\theta', N_0, N_1)]}{\Expec_{\theta' \sim \varepsilon \nu_2} [ f(\theta', N_0, N_1) ]} \right]
         \\
         & \le \log \left(1 + \Expec_{(N_0, N_1) \sim \Pcal_2^n} \left[  \left( \frac{\Expec_{\theta' \sim \varepsilon \nu_1}[ f(\theta', N_0, N_1) ] -\Expec_{\theta' \sim \varepsilon \nu_2} [ f(\theta', N_0, N_1) ]}{\Expec_{\theta' \sim \varepsilon \nu_2} [ f(\theta', N_0, N_1) ]}\right)^2 \right] \right)
         \, ,
    \end{align*}
    where we omit conditioning on $N_0, N_1$ for the inner expectations for simplicity.
    For the ease of presentation, let $\Pnull$ be the probability distribution over $N_0$ and $N_1$ when $\theta = 0$ deterministically.
    Then, we have $\frac{d\Pcal_2^n}{d\Pnull} = \frac{\Expec_{\theta' \sim \varepsilon \nu_2} [ f(\theta', N_0, N_1)]}{f(0, N_0, N_1)} =  \Expec_{\theta' \sim \varepsilon \nu_2} [ f(\theta', N_0, N_1)]$.
    Then, the chi-squared divergence can be expressed as
    \begin{align*}
        \chi^2 := & \Expec_{(N_0, N_1)\sim \Pcal_2^n} \left[  \left( \frac{\Expec_{\theta' \sim \varepsilon \nu_1}[ f(\theta', N_0, N_1)] -\Expec_{\theta' \sim \varepsilon \nu_2} [ f(\theta', N_0, N_1) ]}{\Expec_{\theta' \sim \varepsilon \nu_2} [ f(\theta', N_0, N_1) ]}\right)^2 \right]
        \\
        & = \Expec_{(N_0, N_1)\sim\Pnull} \left[\frac{d\Pcal_2^n}{d\Pnull}  \left( \frac{\Expec_{\theta' \sim \varepsilon \nu_1}[ f(\theta', N_0, N_1)] -\Expec_{\theta' \sim \varepsilon \nu_2} [ f(\theta', N_0, N_1) ]}{\Expec_{\theta' \sim \varepsilon \nu_2} [ f(\theta', N_0, N_1) ]}\right)^2 \right]
        \\
        & = \Expec_{(N_0, N_1)\sim\Pnull} \left[ \frac{\left( \Expec_{\theta' \sim \varepsilon \nu_1}[ f(\theta', N_0, N_1)] -\Expec_{\theta' \sim \varepsilon \nu_2} [ f(\theta', N_0, N_1) ] \right)^2}{\Expec_{\theta' \sim \varepsilon \nu_2} [ f(\theta', N_0, N_1) ]} \right]
    \end{align*}
    We first lower bound the denominator.
    Since $\nu_2$ is symmetric, we have 
    \begin{align*}
        \Expec_{\theta' \sim \varepsilon \nu_2}[ f(\theta', N_0, N_1)] = \Expec_{\theta' \sim \varepsilon \nu_2}[ f( - \theta', N_0, N_1)] = \frac{1}{2}  \Expec_{\theta' \sim \varepsilon \nu_2}[f(\theta', N_0, N_1)+f(-\theta', N_0, N_1)]
        \, .
    \end{align*}
    Using the AM-GM inequality, we obtain that
    \begin{align*}
        \frac{1}{2} \left(f(\theta', N_0, N_1) + f( - \theta', N_0, N_1)\right) 
        & = \frac{1}{2} \left( (1 - \theta')^{N_0}(1 + \theta')^{N_1} + (1 + \theta')^{N_0}(1 - \theta')^{N_1} \right)
        \\
        & \ge \left( (1 + \theta')^{N_0 + N_1} (1 -\theta')^{N_0 + N_1} \right)^{\frac{1}{2}} 
        \\
        & = (1 - \theta'^2) ^{\frac{N_0 + N_1}{2}}
        \\
        & \ge (1 - \varepsilon^2)^{\frac{N_0+N_1}{2}}
        \, ,
    \end{align*}
    where the last inequality uses that $\theta' \in [ - \varepsilon, \varepsilon]$.
    Therefore, denoting $N := N_0 + N_1$, the chi-squared divergence is upper bounded by
    \begin{align*}
        \chi^2 \le \Expec_{(N_0, N_1) \sim \Pnull} \left[ (1 - \varepsilon^2)^{-\frac{N}{2}} \left( \Expec_{\theta' \sim \varepsilon \nu_1}[ f(\theta', N_0, N_1)] -\Expec_{\theta' \sim \varepsilon \nu_2} [ f(\theta', N_0, N_1) ] \right)^2  \right]
        \, .
    \end{align*}
    Conditioned on $N$, the conditional distribution of $N_0$ under $\Pnull$ follows $B(N, \frac{1}{2})$.
    Hence, using the law of total expectation, we can rewrite the bound as
    \begin{align}
        \chi^2 \le \Expec_{N \sim \Pnull} \left[ (1 - \varepsilon^2)^{-\frac{N}{2}} \Expec_{\substack{N_0 \sim B(N, \frac{1}{2})\\N_1 = N - N_0}} \left[\left( \Expec_{\theta' \sim \varepsilon \nu_1}[ f(\theta', N_0, N_1)] -\Expec_{\theta' \sim \varepsilon \nu_2} [ f(\theta', N_0, N_1) ] \right)^2 \middle\vert N \right]  \right]
        \label{eq:chi bound 1}
        \, .
    \end{align}
    We transform the inner expectation using the following lemma:
    \begin{lemma}
    \label{lma:KL bound 2}
        For $N \in \NN$, let $N_0 \sim B(N, \frac{1}{2})$ and $N_1 = N - N_0$.
        Recall that $f(x; a, b) := (1 - x)^a (1 + x)^b$.
        For any two distributions $\rho_1$ and $\rho_2$ independent of $N_0$, we have
        \begin{align*}
            \Expec_{N_0, N_1} \left[ \left( \Expec_{Y \sim \rho_1}[ f(Y, N_0, N_1)] - \Expec_{Y \sim \rho_2}[ f(Y, N_0, N_1)]  \right)^2 \right]
            = \sum_{l=0}^N \binom{N}{l} \left( \Expec_{Y \sim \rho_1} [ Y^l] - \Expec_{Y \sim \rho_2} [ Y^l] \right)^2
            \, .
        \end{align*}
    \end{lemma}
    The proof of \cref{lma:KL bound 2} is provided in \cref{appx:proof of KL bound 2}
    By applying \cref{lma:KL bound 2}, we have
    \begin{align*}
        & \Expec_{\substack{N_0 \sim B(N, \frac{1}{2})\\N_1 = N - N_0}} \left[\left( \Expec_{\theta' \sim \varepsilon \nu_1}[ f(\theta', N_0, N_1)] -\Expec_{\theta' \sim \varepsilon \nu_2} [ f(\theta', N_0, N_1) ] \right)^2 \middle\vert N \right] 
        \\
        & = \sum_{l=0}^N \binom{N}{l} \left( \Expec_{\theta' \sim \varepsilon \nu_1}[ (\theta')^l] - \Expec_{\theta' \sim \varepsilon \nu_2}[ (\theta')^l] \right)^2
        \\
        & = \sum_{l=k+1}^N \binom{N}{l} \left( \Expec_{\theta' \sim \varepsilon \nu_1}[ (\theta')^l] - \Expec_{\theta' \sim \varepsilon \nu_2}[ (\theta')^l] \right)^2
        \\
        & \le \sum_{l=k+1}^N \binom{N}{l} \varepsilon^{2l}
        \, ,
    \end{align*}
    where we use the property $\Expec_{\theta' \sim \varepsilon \nu_1}[ (\theta')^l] = \Expec_{\theta' \sim \varepsilon \nu_2}[ (\theta')^l]$ for $l = 0, \ldots, k$ by \cref{lma:absolute approximating measure} for the second equality, and the last inequality uses that $|\theta'| \le \varepsilon$.
    Plugging this result into Eq.~\eqref{eq:chi bound 1}, we obtain that
    \begin{align*}
        \chi^2 \le \Expec_{N\sim \Pnull} \left[ (1 - \varepsilon^2)^{-\frac{N}{2}} \sum_{l=k+1}^N \binom{N}{l} \varepsilon^{2l} \right]
        \, .
    \end{align*}
    Then, we apply Taylor's theorem, whose result is encapsulated by the following lemma:
    \begin{lemma}
    For $n \in \NN$, $p \in [0, 1]$, and $a, x \in \RR$, there exists $\tau \in [0, 1]$ such that
    \label{lma:KL bound 3}
        \begin{align*}
            \Expec_{N \sim B(n, p)} \left[ a^{N} \sum_{l=k+1}^N \binom{N}{l} x^l \right]
            & = \binom{n}{k+1} (a px)^{k+1}(ap(1 + \tau x) + 1 - p)^{n - k - 1}
            \, .
        \end{align*}
    \end{lemma}
    The proof of \cref{lma:KL bound 3} is provided in \cref{appx:proof of KL bound 3}.
    Recall that $N \sim B(n, \frac{1}{d})$ under $\Pnull$.
    Plugging in $x = \varepsilon^2$, $a = \frac{1}{\sqrt{1 - \varepsilon^2}}$, and $p = \frac{1}{d}$, we obtain that
    \begin{align*}
        \chi^2 & \le \binom{n}{k+1} \left( \frac{\varepsilon^2}{d \sqrt{1 - \varepsilon^2}} \right)^{k+1} \left( \frac{1}{d \sqrt{1 - \varepsilon^2}}\left(1 + \tau \varepsilon^2 \right) + 1 - \frac{1}{d} \right)^{n - k - 1}
        \\
        & \le \left( \frac{e n \varepsilon^2}{d (k + 1) \sqrt{1 - \varepsilon^2}} \right)^{k+1} \left( \frac{1}{d \sqrt{1 - \varepsilon^2}}\left(1 + \tau \varepsilon^2 \right) + 1 - \frac{1}{d} \right)^{n - k - 1}
        \\
        & \le \left( \frac{e n \varepsilon^2}{d (k + 1) \sqrt{1 - \varepsilon^2}} \right)^{k+1} \left( 1 + \frac{1}{d} \left( \frac{1 + \varepsilon^2}{\sqrt{1 - \varepsilon^2}} - 1 \right) \right)^{n - k - 1}
        \, ,
    \end{align*}
    where we use $\binom{n}{k+1} \le (\frac{en}{k+1})^{k+1}$ for the second inequality.
    From $0 \le \varepsilon \le \frac{1}{2}$, we bound $\frac{e}{\sqrt{1 - \varepsilon^2}}$ in the first term by 4, and $\frac{1 + \varepsilon^2}{\sqrt{1 - \varepsilon^2}} = \sqrt{ \frac{1 + 2 \varepsilon^2 + \varepsilon^4}{1 - \varepsilon^2}} = \sqrt{ 1 + \frac{3 \varepsilon^2 + \varepsilon^4}{1 - \varepsilon^2}} \le \sqrt{1 + 4 \varepsilon^2 + \frac{4}{3} \varepsilon^4} \le 1 + 2 \varepsilon^2$.
    Then, we have 
    \begin{align*}
        \chi^2 \le  \left( \frac{4 n\varepsilon^2 }{d(k+1)} \right)^{k+1} \left( 1 + \frac{2\varepsilon^2}{d} \right)^n
        \, ,
    \end{align*}
    and the KL divergence is bounded by
    \begin{align*}
        \KL{\Pcal_1^n}{\Pcal_2^n} & \le d \log (1 + \chi^2)
        \\
        & \le d \log \left( 1 + \left( \frac{4 n\varepsilon^2 }{d(k+1)} \right)^{k+1} \left( 1 + \frac{2\varepsilon^2}{d} \right)^n \right)
        \, .
    \end{align*}
    The first part of the lemma is proved.
    
    Suppose $n \le \frac{d \log d}{2 \varepsilon^2}$.
    Then, we have
    \begin{align*}
        \chi^2 & \le \left(\frac{2 \log d}{k+1}\right)^{k+1} \left(1 + \frac{2\varepsilon^2}{d} \right)^{\frac{d \log d}{2\varepsilon^2}}
        \\
        & \le \left(\frac{2 \log d}{k+1}\right)^{k+1} e^{\log d}
        \\
        & = \left(\frac{2 \log d}{k+1}\right)^{k+1} d
        \, .
    \end{align*}
    By setting $k \ge 4 \log d + 2$, we have
    \begin{align*}
        \left( \frac{2 \log d}{k+1} \right)^{k+1}
        & \le \left( \frac{1}{2} \right)^{4 \log d + 3}
        \le \frac{1}{8} e^{(\log \frac{1}{2}) \cdot (4 \log d)}
        \le \frac{1}{8} e^{ - 2 \log d}
        \le \frac{1}{8d^2}
        \, .
    \end{align*}
    Therefore, we have $\chi^2  \le \frac{1}{8d}$, and $\KL{\Pcal_1^n}{\Pcal_2^n} \le \log (1 + \chi^2)^d \le \log ( 1 + \frac{1}{8d})^d \le \log e^{1/8} = \frac{1}{8}$. 
\end{proof}

\section{Proofs of Technical Lemmas in Appendices~\ref{appx:lower bound} and~\ref{appx:proof of l1 estimation}}
\subsection{Proof of Lemma~\ref{lma:lower bound existence}}
\label{appx:proof of lower bound existence}

\begin{proof}[Proof of \cref{lma:lower bound existence}]
    We set $\Pcal_1$ and $\Pcal_2$ to be distributions defined in \cref{prop:KL bound} with $d = \tilde{D}$ and $k = 6 \log \tilde{D}$.
    Then, \cref{prop:KL bound} immediately yields that $\KL{\Pcal_1^n}{\Pcal_2^n} \le \frac{1}{8}$ for $n \le \frac{\tilde{D}\log \tilde{D}}{2 \varepsilon^2}$.

    Now, we prove the second part of the lemma.
    We define $\mu_1$ and $\mu_2$ as defined after \cref{lma:absolute approximating measure}.
    Specifically, they satisfy $\varepsilon\mu_1 = \Expec_{\thetab \sim \Pcal_1} [ |\theta_i |]$ and $\varepsilon \mu_2 = \Expec_{\thetab \sim \Pcal_2} [ | \theta_i |] $ for all $i \in [\tilde{D}]$.
    In addition, \cref{lma:absolute approximating measure} states that $\mu_2 - \mu_1 \ge \frac{1}{2k} = \frac{1}{12 \log \tilde{D}}$.
    We set $\alpha := \varepsilon(\mu_1 +  \frac{1}{48 \log \tilde{D}}) \frac{2 \tilde{H} - 1 - \tilde{D}}{4}$.
    Applying Hoeffding's inequality (\cref{lma:Hoeffding}), we have
    \begin{align*}
        & \PP_{\thetab \sim \Pcal_1} \left( \tilde{V}^*(\thetab) \ge \alpha \right)
        \\
        & = 
        \PP_{\thetab \sim \Pcal_1} \left( \frac{1}{2\tilde{D}}\sum_{i=1}^{\tilde{D}} (\tilde{H} - i) | \theta_i | \ge  \varepsilon \left(\mu_1 +  \frac{1}{48 \log \tilde{D}} \right) \frac{2 \tilde{H} - 1 - \tilde{D}}{4} \right)
        \\
        & = 
        \PP_{\thetab \sim \Pcal_1} \left( \frac{1}{2\tilde{D}}\sum_{i=1}^{\tilde{D}} (\tilde{H} - i) (| \theta_i | - \varepsilon \mu_1) \ge  \frac{\varepsilon}{48 \log \tilde{D}} \cdot \frac{2 \tilde{H} - 1 - \tilde{D}}{4} \right)
        \\
        & \le \exp \left( - \left(  \frac{\varepsilon}{48 \log \tilde{D}} \cdot \frac{2 \tilde{H} - 1 - \tilde{D}}{4}\right)^2 \cdot \left( \sum_{i=1}^{\tilde{D}} \left( \frac{\varepsilon(\tilde{H} - i)}{2 \tilde{D}} \right)^2 \right)^{-1}\right)
        \\
        & \le \exp \left( - \left(  \frac{1}{48 \log \tilde{D}} \cdot \frac{\tilde{H}}{4}\right)^2 \cdot \left( \frac{\tilde{H}^2}{4 \tilde{D}} \right)^{-1}\right)
        \\
        & = \exp \left(  - \frac{\tilde{D}}{(96 \log \tilde{D})^2}\right)
        \, ,
    \end{align*}
    where we use that $\tilde{H} \ge \tilde{D} + 1$ in the fifth line.
    We note that for large enough $\tilde{D}$, we have $\exp( - \tilde{D} / (96 \log \tilde{D})^2) \le \frac{1}{4}$.
    Then we obtain that $\PP_{\thetab \sim \Pcal_1} ( \tilde{V}^*(\thetab) \ge \alpha) \le \frac{1}{4}$.
    \\
    It remains to show that $\PP_{\thetab \sim \Pcal_2}( \tilde{V}^*(\thetab) \le \alpha + \frac{\varepsilon H}{128 \log \tilde{D}} ) \le \frac{1}{4}$.
    We define $\alpha_2 := \varepsilon(\mu_2 - \frac{1}{48 \log \tilde{D}}) \frac{2 \tilde{H} - 1 - \tilde{D}}{4}$, and in the exact same way, we can prove $\PP_{\thetab\sim \Pcal_2}( \tilde{V}^*(\thetab) \le \alpha_2) \le \frac{1}{4}$.
    We have
    \begin{align*}
        \alpha_2 - \alpha
        & = \varepsilon \left( \mu_2 - \mu_1 - \frac{1}{24 \log \tilde{D}} \right) \frac{2 \tilde{H} - 1 - \tilde{D}}{4}
        \\
        & \ge \frac{\varepsilon}{24 \log \tilde{D}} \cdot  \frac{2 \tilde{H} - 1 - \tilde{D}}{4}
        \\
        & \ge \frac{\varepsilon}{24 \log \tilde{D}} \cdot  \frac{\tilde{H} }{4}
        \\
        & \ge \frac{\varepsilon H}{192 \log \tilde{D}}
        \, ,
    \end{align*}
    where we use $\mu_2 - \mu_1 \ge \frac{1}{12 \log \tilde{D}}$ for the first inequality, the $\tilde{D} + 1 \le \tilde{H}$ for the second inequality, and $\tilde{H} \ge \frac{H}{2}$ for the last inequality.
    Therefore, we deduce that $\PP_{\thetab \sim \Pcal_2}( \tilde{V}^*(\thetab) \le \alpha + \frac{\varepsilon H}{192 \log \tilde{D}} ) \le \PP_{\thetab\sim \Pcal_2}( \tilde{V}^*(\thetab) \le \alpha_2) \le \frac{1}{4} $.
\end{proof}

\subsection{Proof of Lemma~\ref{lma:kl of MDPs}}
\label{appx:proof of kl of MDPs}

\begin{proof}[Proof of \cref{lma:kl of MDPs}]
    We denote $\PP_1 := \PP_1^\tau$ and $\PP_2 := \PP_2^\tau$ for simplicity.
    
    For $\omega = \{ (s_1^k, a_1^k, \ldots, s_{H+1}^k ) \}_{k=1}^{\tau} \in \Omega$, we have
    \begin{align*}
        \PP_i (\omega) =\Expec_{\substack{\Mcal(\thetab)\\ \thetab \sim \Pcal_i}} \left[ \prod_{k=1}^\tau \prod_{h=1}^H P_{s_h^k, a_h^k}(s_{h+1}^k) \right]
    \end{align*}
    for $i=1$ and $i =2$.
    Denote the state sampled from the $j$-th selection of $(s_0, a_0)$ in the trajectory $\omega$ by $s^j(\omega)$. 
    Then, we have that
    \begin{align*}
        \frac{\PP_1(\omega)}{\PP_2(\omega)} 
        & = \frac{\Expec_{\Pcal_1}  \left[\prod_{\substack{(k, h) \in [\tau] \times [H] \\ (s_h^k, a_h^k) = (s_0, a_0)}} P_{s_0, a_0}(s_{h+1}^k) \right] }{\Expec_{\Pcal_2}  \left[ \prod_{\substack{(k, h) \in [\tau] \times [H] \\ (s_h^k, a_h^k) = (s_0, a_0)}} P_{s_0, a_0}(s_{h+1}^k) \right]}
        \\
        & = \frac{\Expec_{\Pcal_1}  \left[\prod_{j=1}^{N^{\tau+1}(s_0, a_0)} P_{s_0, a_0}(s^j(\omega)) \right] }{\Expec_{\Pcal_2}  \left[ \prod_{j=1}^{N^{\tau+1}(s_0, a_0)} P_{s_0, a_0}(s^j(\omega)) \right]}
        \\
        & = \frac{\Pcal_1^{N^{\tau+1}}(\{s^j(\omega)\}_{j=1}^{N^{\tau+1}(s_0, a_0)})}{\Pcal_2^{N^{\tau+1}}(\{s^j(\omega)\}_{j=1}^{N^{\tau+1}(s_0, a_0)})}
        \, .
    \end{align*}
    The likelihood ratio only depends on the sequence of states that the state-action pair $(s_0, a_0)$ transitioned into, and it is the ratio between $\Pcal_1^{N^{\tau+1}(s_0, a_0)}$ and $\Pcal_2^{N^{\tau+1}(s_0, a_0)}$.
    Then, the KL-divergence is expressed as the following:
    \begin{align*}
        \KL{\PP_1}{\PP_2}
        & = 
        \Expec_{\omega \sim \PP_1} \left[  \log \frac{\PP_1(\omega)}{\PP_2(\omega)} \right]
        \\
        & = \Expec_{\omega \sim \PP_1} \left[ \log  \frac{\Pcal_1^{N^{\tau+1}}(\{s^j(\omega)\}_{j=1}^{N^{\tau+1}(s_0, a_0)})}{\Pcal_2^{N^{\tau+1}}(\{s^j(\omega)\}_{j=1}^{N^{\tau+1}(s_0, a_0)})} \right]
        \\
        & = \Expec_{N^{\tau+1}(s_0, a_0) \sim \PP_1} \left[ \Expec_{\{s^j(\omega)\}_{j=1}^{N^{\tau+1}(s_0, a_0)}} \left[ \log \frac{\Pcal_1^{N^{\tau+1}}(\{s^j(\omega)\}_{j=1}^{N^{\tau+1}(s_0, a_0)})}{\Pcal_2^{N^{\tau+1}}(\{s^j(\omega)\}_{j=1}^{N^{\tau+1}(s_0, a_0)})} \middle| N^{\tau+1}(s_0, a_0) \right] \right]
        \\
        & = \Expec_{N^{\tau+1}(s_0, a_0) \sim \PP_1} \left[ \KL{\Pcal_1^{N^{\tau+1}(s_0, a_0)}(\cdot \mid N^{\tau+1}(s_0, a_0))}{\Pcal_2^{N^{\tau+1}(s_0, a_0)}(\cdot \mid N^{\tau+1}(s_0, a_0))} \right]
        \, .
    \end{align*}
    By the data processing inequality, we have
    \begin{align*}
        \KL{\Pcal_1^{N^{\tau+1}(s_0, a_0)}(\cdot \mid N^{\tau+1}(s_0, a_0))}{\Pcal_2^{N^{\tau+1}(s_0, a_0)}(\cdot \mid N^{\tau+1}(s_0, a_0))} \le  \KL{\Pcal_1^{n}}{\Pcal_2^{n}}
        \, ,
    \end{align*}
    where we use that $N^{\tau+1}(s_0, a_0) \le n$ by the stopping rule of $\tau$.
    The proof is complete.    
\end{proof}

\subsection{Proof of Lemma~\ref{lma:KL bound 2}}

\label{appx:proof of KL bound 2}

\begin{proof}[Proof of \cref{lma:KL bound 2}]
    Let $c(a, b, l)$ be the coefficients of $f(x, a, b)$ so that $f(x, a, b) = (1 - x)^a (1 + x)^b = \sum_{l=0}^{a+b} c(a, b, l) x^l$.
    Then, for given $N_0, N_1$, the squared term becomes
    \begin{align*}
        & \left( \Expec_{Y \sim \rho_1}[ f(Y, N_0, N_1)] -\Expec_{Y \sim \rho_2} [ f(Y, N_0, N_1) ] \right)^2 
        \\
        & = \left( \sum_{l=0}^N c(N_0, N_1, l) \Expec_{Y \sim \rho_1} [ Y^l] - \sum_{l=0}^N c(N_0, N_1, l) \Expec_{Y \sim \rho_2} [ Y^l]  \right)^2  
        \\
        & = \left( \sum_{l=0}^N c(N_0, N_1, l) \left(\Expec_{Y \sim \rho_1} [ Y^l] - \Expec_{Y \sim \rho_2} [ Y^l] \right)  \right)^2 
        \\
        & = \sum_{l=0}^N \sum_{l'=0}^N c(N_0, N_1, l) c(N_0, N_1, l') \left(\Expec_{Y \sim \rho_1} [ Y^l] - \Expec_{Y \sim \rho_2} [ Y^l] \right) \left(\Expec_{Y \sim \rho_1} [ Y^{l'}] - \Expec_{Y \sim \rho_2} [ Y^{l'}]\right)
        \, .
    \end{align*}
    For the ease of notation, let $m_l := \Expec_{Y \sim \rho_1} [ Y^l] - \EE_{Y\sim \rho_2} [ Y^{l}]$.
    Then, we obtain that
    \begin{align*}
        \left( \Expec_{Y \sim \rho_1}[ f(Y, N_0, N_1)] -\Expec_{Y \sim \rho_2} [ f(Y, N_0, N_1) ] \right)^2 
        & = \sum_{l=0}^N \sum_{l' = 0}^N c(N_0, N_1, l)c(N_0, N_1, l') m_l m_{l'}
        \, .
    \end{align*}
    Then, it holds that
    \begin{align*}
        & \Expec_{N_0 \sim B(N, \frac{1}{2})} \left[ \sum_{l=0}^N \sum_{l' = 0}^N c(N_0, N_1, l) c(N_0, N_1, l') m_l m_{l'} \right] 
        \\
        & = \sum_{N_0=0}^N \binom{N}{N_0} \left( \frac{1}{2} \right)^{N} \sum_{l=0}^N \sum_{l' = 0}^N c(N_0, N_1, l) c(N_0, N_1, l') m_l m_{l'} 
        \\
        & = \left( \frac{1}{2} \right)^{N} \sum_{l=0}^N \sum_{l' = 0}^N  m_l m_{l'} \sum_{N_0=0}^N \binom{N}{N_0} c(N_0, N_1, l) c(N_0, N_1, l') 
        \, .
    \end{align*}
    We need the following lemma, whose proof is provided at the end of this subsection.
    \begin{lemma}
    \label{lma:orthogonality of c}
        For nonnegative integers $l, l'$, and $N$, we have the following equality:
        \begin{align}
            \sum_{a = 0}^N \binom{N}{a} c(a, N-a, l) c(a, N-a, l') = \begin{cases}
                2^N \binom{N}{l} & (l = l')
                \\
                0 & (l \ne l')
            \end{cases}
            \, .
            \label{eq:orthgonality of c}
        \end{align}
    \end{lemma}
    By this lemma, we conclude that
    \begin{align*}
        \Expec_{N_0, N_1} \left[ \sum_{l=0}^N \sum_{l' = 0}^N c(N_0, N_1, l) c(N_0, N_1, l') m_l m_{l'} \right] 
        & = \sum_{l=0}^N \binom{N}{l} m_l^2 
        \, .
    \end{align*}
\end{proof}

\begin{proof}[Proof of \cref{lma:orthogonality of c}]
    We prove the lemma using generating functions.
    Consider the polynomial $((1-x)(1-y) + (1 +x)(1+y))^N$.
    First, we have
    \begin{align*}
        ((1-x)(1-y) + (1 +x)(1+y))^N
        & = \sum_{a=0}^N \binom{N}{a} \left((1 - x)(1-y)\right)^a \left((1+x)(1+y)\right)^{N-a}
        \\
        & = \sum_{a=0}^N \binom{N}{a} \left( (1 - x)^a(1 + x)^{N-a} \right) \left( (1 - y)^a (1 + y)^{N-a} \right)
        \\
        & = \sum_{a=0}^N \binom{N}{a} \left( \sum_{l=0}^N c(a, N-a, l) x^l \right) \left( \sum_{l'=0}^N c(a, N-a, l') y^{l'} \right)
        \\
        & = \sum_{l=0}^N \sum_{l'=0}^N \sum_{a=0}^N \binom{N}{a} c(a, N-a, l) c(a, N-a, l') x^l y^{l'}
        \, .
    \end{align*}
    Hence, the left-hand side of Eq.~\eqref{eq:orthgonality of c} is the coefficient of $x^l y^{l'}$ in $((1 - x)(1 - y) + (1 + x)(1 + y))^N$.
    On the other hand, noting that $(1 - x)(1 - y) + (1 + x)(1 + y) = 2 + 2xy$, we have
    \begin{align*}
        ((1-x)(1-y) + (1+x)(1+y))^N
        & = (2 (1 + xy))^N
        \\
        & = 2^N \sum_{l=0}^N \binom{N}{l} (xy)^l
        \, .
    \end{align*}
    In this expression, the coefficient of $x^l y^{l'}$ is the right-hand side of Eq.~\eqref{eq:orthgonality of c}, proving the equality between the two.
\end{proof}

\subsection{Proof of Lemma~\ref{lma:KL bound 3}}
\label{appx:proof of KL bound 3}
\begin{proof}[Proof of \cref{lma:KL bound 3}]
    Expanding the expectation using the pmf of binomial distribution, we derive that
    \begin{align}
        \Expec_{N \sim B(n, p)} \left[ a^N \sum_{l=k+1}^N \binom{N}{l} x^l \right]
        & = \sum_{N=0}^n \binom{n}{N} p^{N} (1 - p)^{n - N} a^N \sum_{l=k+1}^N \binom{N}{l} x^l
        \nonumber
        \\
        & = \sum_{l=k+1}^n \left( \sum_{N=l}^n \binom{n}{N} \binom{N}{l} (ap)^N (1 - p)^{n - N}  \right) x^l
        \label{eq:KL bound 3-1}
        \, .
    \end{align}
    Consider a polynomial $g(x) := (ap(1 + x) + 1 - p)^n$.
    We have
    \begin{align*}
        (ap(1 + x) + 1 - p)^n
        & = \sum_{N=0}^n \binom{n}{N} (ap)^N (1 + x)^N (1 - p)^{n - N}
        \\
        & = \sum_{N=0}^n \binom{n}{N} (ap)^N (1 - p)^{n - N} \sum_{l=0}^N \binom{N}{l} x^l
        \\
        & = \sum_{l=0}^n \left( \sum_{N=l}^n \binom{N}{n} \binom{N}{l} (ap)^N (1 - p)^{n - N} \right) x^l 
        \, .
    \end{align*}
    Observe that Eq.~\eqref{eq:KL bound 3-1} is a remainder of $g(x)$ when approximated by a degree $k$ polynomial.
    By Taylor's theorem, for fixed $x$, there exists $\tau \in [0, 1]$ such that
    \begin{align*}
        \sum_{l=k+1}^n \left( \sum_{N=l}^n \binom{n}{N} \binom{N}{l} (ap)^N (1 - p)^{n - N}  \right) x^l
        & = 
        g(x) - \sum_{l=0}^k \left( \sum_{N=l}^n \binom{N}{n} \binom{N}{l} (ap)^N (1 - p)^{n - N} \right) x^l
        \\
        & = \frac{x^{k+1}}{(k+1)!} g^{(x+1)}(\tau x)
        \, .
    \end{align*}
    The proof is completed by noting that $\frac{1}{(k+1)!} g^{(k+1)}(x) = \binom{n}{k+1}(ap)^{k+1}(ap(1 + x) + 1 - p)^{n - k - 1}$.
\end{proof}

\section{Auxiliary Lemmas}

\begin{lemma}[Bennett's inequality, Theorem 3 in \citet{maurer2009empirical}]
\label{lma:bernstein}
    For $n \ge 1$, let $Z_1, \ldots, Z_n$ be i.i.d. copies of a random variable $Z$ that lies in $[0, 1]$.
    Let $\bar{Z}_n := \frac{1}{n} \sum_{i=1}^n Z_i$ and $V := \EE[ (Z - \EE[Z] )^2]$.
    Then, for $\delta \in (0, 1]$, it holds that with probability at least $1 - \delta$,
    \begin{align*}
        \EE[Z] - \bar{Z}_n \le \sqrt{\frac{2 V_n \log \frac{1}{\delta}}{n}} + \frac{\log \frac{1}{\delta}}{3n}
        \, .
    \end{align*}
\end{lemma}

\begin{lemma}[Theorem 4 in \citet{maurer2009empirical}]
\label{lma:empirical bernstein}
    For $n \ge 2$, let $Z_1, \ldots, Z_n$ be i.i.d. copies of a random variable $Z$ that lies in $[0, 1]$.
    Let $\bar{Z}_n := \frac{1}{n} \sum_{i=1}^n Z_i$ and $\hat{V}_n := \frac{1}{n} \sum_{i=1}^n (Z_i - \bar{Z}_n )^2$.
    Then, for $\delta \in (0, 1]$, it holds that with probability at least $1 - \delta$,
    \begin{align*}
        \EE[Z] - \bar{Z}_n \le 2 \sqrt{\frac{ \hat{V}_n \log \frac{2}{\delta}}{n}} + \frac{14 \log \frac{2}{\delta}}{3n}
        \, .
    \end{align*}
\end{lemma}

\begin{lemma}[Hoeffding's inequality, Theorem 2 in~\citet{hoeffding1963probability}]
\label{lma:Hoeffding}
    Let $\left\{ \xi_i \right\}_{i=1}^n$ be a sequence of real-valued random variables adapted to a filtration $\left\{ \Fcal_i \right\}_{i=0}^n$.
    Suppose that there exist $a_i < b_i$ such that $\xi_i \in [a_i, b_i]$ holds almost surely for all $i \in [n]$.
    Then, for $\delta \in (0, 1]$, the following inequality holds with probability at least $1 - \delta$:
    \begin{equation*}
        \sum_{i=1}^n \left( \xi_i - \EE \left[ \xi_i \mid \Fcal_{t-1} \right] \right) \le \sqrt{ \frac{1}{2}\left(\sum_{i=1}^n (b_i - a_i)^2 \right) \log \frac{1}{\delta}}
        \, .
    \end{equation*}
\end{lemma}

\begin{lemma}[Pinsker's inequality, Lemma 2.5 in \citet{tsybakov2008introduction}]
\label{lma:Pinsker}
    For two probability measures $P$ and $Q$, we have $V(P, Q) \le \sqrt{\frac{1}{2}\KL{P}{Q}}$, where $V(P, Q)$ is the total variation of the two measures.
\end{lemma}

\end{document}